\pdfoutput=1
\PassOptionsToPackage{table}{xcolor}
\documentclass{article} %
\usepackage{iclr2027_conference,times}

\usepackage{amsmath,amsfonts,bm}

\def\eqref#1{equation~\ref{#1}}

\def\plaineqref#1{\ref{#1}}

\def\1{\bm{1}}

\def\vdelta{{\bm{\delta}}}

\def\va{{\bm{a}}}

\def\ve{{\bm{e}}}

\def\vm{{\bm{m}}}

\def\vu{{\bm{u}}}

\def\vx{{\bm{x}}}
\def\vy{{\bm{y}}}
\def\vz{{\bm{z}}}

\def\mA{{\bm{A}}}

\def\mD{{\bm{D}}}

\def\mG{{\bm{G}}}

\def\mI{{\bm{I}}}

\def\mL{{\bm{L}}}
\def\mM{{\bm{M}}}

\def\mQ{{\bm{Q}}}
\def\mR{{\bm{R}}}

\def\mU{{\bm{U}}}

\def\mW{{\bm{W}}}
\def\mX{{\bm{X}}}

\DeclareMathAlphabet{\mathsfit}{\encodingdefault}{\sfdefault}{m}{sl}
\SetMathAlphabet{\mathsfit}{bold}{\encodingdefault}{\sfdefault}{bx}{n}

\newcommand{\R}{\mathbb{R}}

\usepackage{amsmath}
\usepackage{amssymb}
\usepackage{amsthm}
\usepackage{enumitem}
\usepackage{array}
\usepackage{booktabs}
\usepackage{bm}
\usepackage{adjustbox}
\usepackage{float}
\usepackage{graphicx}
\usepackage{tikz}
\usetikzlibrary{arrows.meta,calc,positioning}
\usepackage{longtable}
\usepackage{makecell}
\usepackage{multirow}
\usepackage{placeins}
\usepackage{hyperref}
\usepackage{url}
\usepackage{xurl}

\newcommand{\cL}{\mathcal{L}}

\newcommand{\cS}{\mathcal{S}}

\newcommand{\TopK}{\operatorname{TopK}}

\newcommand{\cG}{\mathcal{G}}

\newcommand{\norm}[1]{\left\lVert#1\right\rVert_2}
\newcommand{\normF}[1]{\left\lVert#1\right\rVert_F}
\newcommand{\inner}[2]{\left\langle#1,#2\right\rangle}

\newcommand{\normalize}{\operatorname{normalize}}
\newcommand{\ptitle}[1]{\smallskip\noindent\textbf{#1}\quad}

\theoremstyle{plain}
\newtheorem{theorem}{Theorem}
\newtheorem{proposition}{Proposition}
\newtheorem{lemma}{Lemma}
\newtheorem{corollary}{Corollary}
\theoremstyle{definition}
\newtheorem{definition}{Definition}

\theoremstyle{remark}

\title{Do Words Matter?\\
What Text Contributes to Few and Full-Shot\\
Vision--Language Anomaly Detection}

\author{Arda Baris Basaran$^{1}$ \quad Mesut Ceylan$^{1}$ \quad Elgar Fleisch$^{1,2}$ \quad Filipe Barata$^{1}$\\
$^{1}$ETH Zurich \quad $^{2}$University of St.\ Gallen\\
\texttt{\{abasaran,ceylanm,efleisch,fbarata\}@ethz.ch}}

\iclrfinalcopy
\begin{document}

\maketitle
\lhead{Preprint. Under review.}

\begin{abstract}
Vision--language anomaly detectors describe normal and abnormal appearance with text, and their gains
are commonly credited to the semantic content of these descriptions. This attribution is plausible in the zero-shot setting,
but few and full-shot detectors interpose a trained visual adapter, and in that setting it has not been
isolated. Changing a description changes both its meaning and its numerical geometry, so accuracy
comparisons cannot separate the two. We ask what a fixed pair of text embeddings, or \emph{anchors},
contributes to such a detector: semantic content, effective scale, or joint anchor geometry. For a
single anchor pair, the answer is effective scale alone. The score reads the anchors only through their
difference, and a trained adapter can rotate to follow any replacement direction, leaving
that direction, and with it the linguistic content, unidentifiable. We prove that replacement anchors with the same difference-vector length permit exactly the same probability maps and attainable losses. Across seven medical
datasets, three vision backbones, both supervision regimes, and three adapter families, scale-matched
random and orthogonal anchors track correct prompts; breaking the scale match opens a clear gap that
restoring the scale closes without text. The same pattern holds empirically in MVFA's residual,
multi-level detector and when MadCLIP's learned prompts are replaced by a text-free learned scale.
When one adapter serves many categories, the answer expands to
their joint arrangement. Rotating all category axes together does not change anything, as predicted. For the tested
CLIP prompt banks, these axes are nearly collinear, and meaning-free rotations that spread
them apart improve performance on both industrial benchmarks. Even under frozen zero-shot transfer, a random anchor frame
transfers as well as the prompt frame: coordinate consistency, not meaning. What these detectors read from
their prompts is scale and arrangement, not semantic content.
\end{abstract}

\section{Introduction}
\label{sec:intro}

Anomaly detection asks a model to identify appearances it has rarely or never seen. Because anomalies
are rare, visually diverse, and seldom labeled, the field has largely trained detectors on normal
data alone, scoring a test image by how far it departs from the resulting model of normality
\citep{roth2022patchcore,deng2022rd4ad,gudovskiy2022cflow,salehi2021mkd}; supervising anomalies
directly only shifts the scarcity, since labeled anomalous examples are rarer still
\citep{ding2022dra,yao2023bgad}. Such detectors carry no explicit representation of what
\emph{constitutes} abnormality, and must define it negatively, as departure from a reference
distribution that has to be estimated exhaustively from images. The cost is a per-category fit: every
new object or organ needs its own corpus of normal data, and accuracy degrades when only a handful of
examples are available \citep{huang2022regad}. Vision--language models supply the missing half.
Contrastive pretraining on web-scale image--text pairs places visual and textual concepts in a shared
space \citep{radford2021clip}, so the reference for \emph{abnormal} can be written in language
instead of estimated from target-domain images. This both enables low-sample operation and lets a
single category-agnostic detector be re-targeted by changing the noun in a prompt.

\begin{figure}[t]
    \centering
    \resizebox{\linewidth}{!}{\input{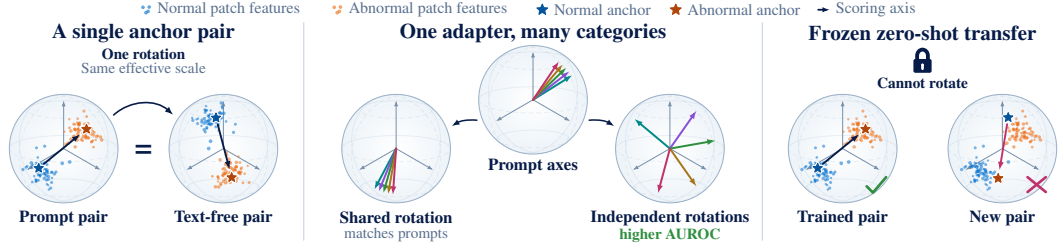}}
    \caption{\textbf{What the text anchors contribute to a trained fixed-anchor detector.} The
    score reads the two anchors only through their difference, the scoring axis. \textbf{Left:} a
    separately trained adapter rotates to follow any replacement axis of the same effective scale, so
    its direction, and with it the semantic content, is unidentifiable. \textbf{Middle:} with one
    adapter serving many categories, a shared rotation preserves the arrangement of their axes;
    independent rotations keep each scale but change the arrangement. \textbf{Right:} a frozen adapter
    cannot rotate, so it must keep the anchor frame it was trained with, which need not be
    linguistic. Embeddings are schematic.}
    \label{fig:geometry_overview}
\end{figure}
\pagebreak

These detectors perform well in practice. Handcrafted prompts give competitive zero-shot
results on industrial benchmarks \citep{jeong2023winclip}; trained visual projections or learned
prompts improve zero- and few-shot performance \citep{chen2023aprilgan,zhou2024anomalyclip};
and the approach extends to medical imaging, where labeled anomalies are scarcest
\citep{huang2024mvfa,shiri2025madclip}. Language-conditioned detection is thus common in the
low-label regime. The trained fixed-anchor subclass studied here uses two prescribed embeddings,
or \emph{anchors}, for its normal--abnormal readout.

The field typically attributes these gains to semantic factors, including anomaly semantics,
language-compatible features, or vision--language alignment
\citep{li2024promptad,zuo2024clipfsac,huang2024mvfa,zhang2024mediclip,shiri2025madclip}.
In a training-free detector, the anchors directly define the scoring rule: no learned visual map
can compensate when their coordinates change. The fixed-anchor few- and full-shot setting differs:
real or synthetic anomaly targets train a visual adapter against the anchors, letting it absorb
changes in their direction. Because changing a description alters both its meaning and its numerical
geometry, ordinary prompt comparisons cannot determine which anchor properties remain observable
after adaptation. Curated prompt banks cost clinical expertise for each new dataset and finding;
if anchors contribute only geometry, pursuing better descriptions and text encoders targets the
wrong variable. This motivates our central question: \emph{For trained fixed-anchor vision--language anomaly detectors, what information from
the text anchors actually affects representational capacity and observed performance: semantic
content, effective scale, or joint anchor geometry?}
We answer this by identifying what the score can read, then testing whether training respects it. The
binary score depends on the anchors only through their difference: a direction, which a trained
adapter can rotate its output to follow (Figure~\ref{fig:geometry_overview}, left), and a magnitude,
which it cannot. We prove the direction is therefore unidentifiable, and probe the claim with five
anchor conditions (correct medical prompts, a mismatched-dataset bank,
and three text-free pairs that preserve or break the surviving magnitude) across seven medical
datasets, two supervision regimes, three vision backbones and three adapter families. We then find
the same scale-controlled pattern in MVFA's residual multi-level detector and in MadCLIP after
replacing its learned prompt path with a text-free learned scale. Two further settings
test the boundaries: one adapter serving many categories, where the invariant becomes the joint
geometry of their difference vectors (Figure~\ref{fig:geometry_overview}, middle), and frozen zero-shot transfer,
where the adapter can no longer rotate.

Our contributions are:

\noindent\textbf{Single-pair identification.}\quad Only the effective scale is identifiable.
Scale-matched random and orthogonal anchors track correct medical prompts with mean gaps of
$0.16\%$--$0.26\%$ image AUROC; breaking the scale opens a clear gap that restoring it
closes with no text.

\noindent\textbf{Shared multi-category geometry.}\quad The invariant expands to the Gram matrix of the
category difference vectors. The tested CLIP prompt axes are nearly collinear, and spreading
them apart, while preserving every category's scale and adding no meaning, improves both industrial
benchmarks, by $5.6\%$ image and $4.4\%$ pixel AUROC on VisA and up to $8.8\%$ and $8.3\%$ across
vision backbones.

\noindent\textbf{Frozen zero-shot transfer boundary.}\quad With the adapter frozen, its anchor frame is
load-bearing but need not be linguistic: a retained random frame averages within $0.5\%$ AUROC
of prompts.

\section{Related Work}
\label{sec:related}

Vision--language anomaly detectors differ in where adaptation occurs. WinCLIP scores handcrafted
prompts with frozen encoders \citep{jeong2023winclip}; APRIL-GAN learns visual projections into a
fixed text space \citep{chen2023aprilgan}; PALADIN aligns DINOv3 patch features to the CLIP text
space for prompt-guided localization \citep{basaran2026paladin}; and AnomalyCLIP learns
object-agnostic prompt tokens \citep{zhou2024anomalyclip}. Medical systems likewise combine prescribed prompt ensembles with
residual visual adapters \citep{huang2024mvfa}, or jointly learn prompt tokens and visual adaptation
\citep{zhang2024mediclip,shiri2025madclip}, while AnomalyDINO removes text entirely
\citep{damm2024anomalydino}. These distinctions determine whether a changed text readout can be
compensated by a trainable visual branch.
Prompt-based detectors commonly credit anomaly semantics or vision--language alignment
\citep{li2024promptad,zuo2024clipfsac,huang2024mvfa,zhang2024mediclip,shiri2025madclip}, yet strong
visual-only systems show that accuracy alone cannot identify a contribution from language
\citep{roth2022patchcore,deng2022rd4ad,gudovskiy2022cflow,salehi2021mkd,
liu2023simplenet,li2021cutpaste,schluter2022natural,huang2022regad}. Learned-prompt methods further
blur attribution because their text representations are fitted rather than prescribed
\citep{zhou2022coop,zhou2022cocoop,khattak2023maple}. Evidence on textual content is mixed even with
frozen visual features: semantic descriptors can improve classification and make its criteria
editable \citep{menon2023visual}, whereas random words can match curated descriptors
\citep{roth2023waffle}. Neither setting studies fixed text anchors with a separately trainable visual
branch. Identifiability analyses more generally characterize representations only up to
transformations that leave the observable model unchanged \citep{roeder2021linear}. The remaining
question is which properties of fixed text embeddings stay identifiable after visual adaptation.

\section{Preliminaries}
\label{sec:prelim}

This section fixes notation and describes the detector family the paper analyzes: how frozen image
patches, a trainable adapter, and a prescribed pair of text embeddings produce an anomaly map.

\ptitle{Notation and the studied detector.}
A frozen vision backbone $b$ exposes $P$ patch features $\vz^{(\ell)}_{1:P}$, one per token of the
patch grid, at each of $L$ retained layers indexed by $\ell$. Each layer has its own trainable
adapter $g^{(\ell)}$ into the $d$-dimensional anchor space, and its normalized adapted patch is
\begin{equation}
\vz^{(\ell)}_{1:P}=b^{(\ell)}(x),
\qquad
\hat\vu^{(\ell)}_p
=\normalize\!\left(g^{(\ell)}(\vz^{(\ell)}_p)\right)\in\R^d.
\label{eq:patch_pipeline}
\end{equation}
The scoring head contains two prescribed unit vectors, $\va_N$ for normal and $\va_A$ for
abnormal, which we call \emph{anchors}. For text conditions, each is the normalized mean of a fixed
prompt ensemble encoded by
the frozen text encoder; random-coordinate controls construct the two vectors directly. In either
case the vectors remain fixed throughout a run. Patch $p$ at layer $\ell$ receives the two cosine
logits
\begin{equation}
s^{(\ell)}_{p,N}=\tau\inner{\hat\vu^{(\ell)}_p}{\va_N},
\qquad
s^{(\ell)}_{p,A}=\tau\inner{\hat\vu^{(\ell)}_p}{\va_A}.
\label{eq:score}
\end{equation}
Here $\tau>0$ is the multiplicative similarity, or logit, scale; under the conventional softmax
parameterization it is an inverse temperature rather than a temperature.

\ptitle{From logits to anomaly maps.}
The patch-grid logits are reshaped and bilinearly upsampled to image resolution before the
two-class softmax. Its abnormal channel is
\begin{equation}
p^{(\ell)}_{A}
=\frac{\exp s^{(\ell)}_A}{\exp s^{(\ell)}_N+\exp s^{(\ell)}_A}
=\sigma\!\left(s^{(\ell)}_A-s^{(\ell)}_N\right).
\label{eq:binary_probability}
\end{equation}

\ptitle{Training objective and evaluation.}
Focal and Dice losses \citep{lin2017focal,milletari2016vnet} train every layer map. At inference,
the $L$ abnormal maps are averaged;
pixel AUROC reads this averaged map, while the image score is the mean of its top $1\%$ pixels.
Appendix~\ref{app:objective} verifies the interpolation, loss, and evaluation steps exactly.

\section{Theoretical Analysis}
\label{sec:theory}

The goal of this section is to identify which properties of the prompt embeddings actually determine the detector's capabilities, and thereby separate the contribution of linguistic meaning from that of embedding geometry. We ask when text-derived anchors can be replaced by text-free vectors without changing the representable abnormal-probability maps. We first answer this for one normal--abnormal pair, then for several pairs sharing an adapter.

\subsection{One anchor pair: separating direction from scale}
\label{sec:gram}

What does the detector read from its anchors? By~\eqref{eq:binary_probability}, only the abnormal
logit minus the normal logit determines the abnormal probability. The anchors therefore enter through their
difference, $\vdelta=\va_A-\va_N$, the \emph{scoring axis}. With anchor cosine
$\rho=\inner{\va_N}{\va_A}$, we separate its length from its direction:
\begin{equation}
s^{(\ell)}_{p,A}-s^{(\ell)}_{p,N}
=\tau\inner{\hat\vu^{(\ell)}_p}{\vdelta}
=\underbrace{\tau\sqrt{2-2\rho}}_{\lambda_{\mathrm{eff}}}
\inner{\hat\vu^{(\ell)}_p}{\hat\vdelta},
\qquad \hat\vdelta=\frac{\vdelta}{\norm{\vdelta}}.
\label{eq:effective_logit_scale}
\end{equation}
The inner product measures an adapted patch's alignment with the scoring axis; positive values
favor abnormality. The factor $\lambda_{\mathrm{eff}}$, the \emph{effective logit scale} (hereafter effective scale), sets how strongly this alignment changes
the probability. Closer anchors give a smaller effective scale.

Suppose we rotate the anchors and then retrain the adapter. Training now encourages higher cosine
alignment with the new scoring axis for abnormal patches and lower alignment for normal patches.
An adapter that can rotate its outputs with that axis can preserve the original cosine values
(Figure~\ref{fig:geometry_overview}, left).
Normalization fixes the output length, however, so rescaling the adapter cannot absorb a change in
axis length. This motivates replacing linguistic coordinates while preserving effective scale.

An adapter family $\cG$ is \emph{rotation-closed} if every rotation of an adapter's outputs can be
achieved by changing its parameters. One matrix $\mQ\in O(d)$ rotates the outputs of every layer; $O(d)$ is
the set of $d$-dimensional orthogonal matrices, which preserve lengths and angles.
A freely trained final linear map absorbs $\mQ$, so our linear, MLP, and Transformer adapters
all qualify (Appendix~\ref{app:closure}).

This gives a direct test of prompt meaning: if replacement anchors preserve the effective scale,
can a retrained adapter reproduce every original abnormal-probability map?
Theorem~\ref{thm:gram} below answers this question for representable probability maps.

\begin{theorem}[Single-pair equivalence and effective scale]
\label{thm:gram}
Consider unit anchor pairs $(\va_N,\va_A)$ and $(\va'_N,\va'_A)$, the original and replacement, used with
positive similarity scales $\tau$ and $\tau'$. Hold each pair fixed within its run and train a separate adapter from a
rotation-closed family. If
\begin{equation}
\tau\norm{\va_A-\va_N}
=\tau'\norm{\va'_A-\va'_N},
\label{eq:matched_effective_logit_scale}
\end{equation}
then the two choices permit exactly the same category-agnostic abnormal-probability maps. Any
training objective or evaluation statistic computed only from those maps has the same attainable
values under both choices. Its infimum and supremum coincide, as does its optimum whenever attained. At fixed $\tau$,
matching the anchor cosine $\rho$ is sufficient; more generally the surviving scalar is the effective
scale $\lambda_{\mathrm{eff}}=\tau\sqrt{2-2\rho}$.
\end{theorem}

Thus, scale-matched text-free anchors can reproduce every map achievable with prompts, and vice
versa. Equal representational capabilities need not yield identical training outcomes: stochasticity and finite optimization can lead to different solutions. Our experiments test whether the two anchor choices achieve comparable performance. Appendices~\ref{app:gram} and~\ref{app:residual} give the
proof and explain possible residuals.

\subsection{Several anchor pairs: why their arrangement matters}
\label{sec:multicategory}

On industrial anomaly detection datasets such as MVTec-AD, one adapter serves object categories
such as bottles and cables, each with its own anchor pair.
Rotating all scoring axes together still allows one compensating adapter rotation. Rotating them
independently can change their relative angles, which no common rotation can undo
(Figure~\ref{fig:geometry_overview}, middle). We therefore need to track their joint arrangement,
as well as each scale.
Category $t$ selects only its own unit anchor pair $(\va_t^N,\va_t^A)$, with cosine $\rho_t$;
all $T$ categories share the adapter. Stack their scoring axes $\vdelta_t=\va_t^A-\va_t^N$ as rows
of $\mD$. Their pairwise inner products form a \emph{Gram matrix}, recording lengths and alignments:
\begin{equation}
\mM=\mD\mD^\top,
\qquad
M_{ts}=\inner{\vdelta_t}{\vdelta_s},
\qquad
M_{tt}=2-2\rho_t.
\label{eq:multicat_gram}
\end{equation}
The diagonal fixes each category's scale at a given $\tau$; the off-diagonals record how axes align.
If two banks share $\mM$, one rotation maps every original scoring axis to its replacement. A shared
adapter can follow that one rotation. The corollary states what this implies for the probability
maps of all categories.

\begin{corollary}[Shared multi-category extension]
\label{cor:multicategory}
Let two category-conditioned anchor banks have difference matrices $\mD$ and $\mD'$. If
$\mD\mD^\top=\mD'{\mD'}^\top$, then separately trained rotation-closed adapters permit exactly
the same category-conditioned abnormal-probability maps under the two banks at the same similarity scale.
\end{corollary}

Our experiments therefore use shared rotations to test equivalence at fixed $\mM$, and independent
rotations to test the effect of changing the arrangement at fixed category scales. The corollary
predicts no direction for that effect.
Here $\mM$ characterizes the axes up to a common orthogonal transformation; matching it is sufficient,
but not universally necessary, for equal capabilities (Appendix~\ref{app:multicategory}).

\ptitle{When the argument stops applying.} Both results require fixed anchors, losses and metrics
computed from cosine-based binary probability maps, and an adapter able to follow changed anchors.
The next proposition collects cases where the compensating rotation is unavailable. Learned prompts
fall outside the fixed-anchor premise; these conclusions concern the studied detector family.

\begin{proposition}[Proof in Appendix~\ref{app:breaking}]
\label{prop:breaking}
The equivalence guarantees above no longer follow if the adapter family cannot absorb the required
output rotation or if the adapter is frozen. In shared multi-category scoring, they also do not
follow when the replacement changes $\mD\mD^\top$, because no single rotation then maps all scoring
axes to their replacements. None of these conditions guarantees different performance or a benefit
from correct text; each only removes the equivalence guarantee.
\end{proposition}

\section{Method: Anchor Controls}
\label{sec:controls}

Suppose we train an adapter with prompts and then replace their anchors with scale-matched,
text-free vectors. Theorem~\ref{thm:gram} says we \emph{could} rotate the trained adapter's outputs
to keep every abnormal-probability map unchanged. In our experiments, however, we do not apply
that rotation: we train a new adapter from random weights for a fixed schedule. Its initialization
is tied to coordinate axes, and Adam rescales updates coordinate by coordinate, so the new run need
not follow the corresponding rotated path. Stochastic batches and limited numerical precision can add further
variation (Appendix~\ref{app:residual}). We therefore measure whether ordinary training reaches
similar held-out performance without words, then change scale, joint geometry, and the ability to
retrain the adapter to test the limits of this result.

\phantomsection\label{sec:instruments}
\ptitle{Few-shot medical controls.} We first test the single-pair result with a newly trained adapter for every anchor choice. Correct
medical prompts provide the reference: for each dataset, complete normal and abnormal description
banks specify its anatomy, imaging modality, state definitions, and templates. We encode each bank,
average and normalize its embeddings, and obtain the reference cosine
$\rho_{\mathrm{correct}}$ and effective scale $\lambda_{\mathrm{eff}}$. All anchor pairs stay fixed
throughout training; learning them would also change $\rho$. We repeat the comparisons across
supervision regimes, vision backbones, and adapter families to check finite-training behavior.
The decisive test removes language while keeping exactly what the theorem says matters. A seeded
random unit-vector pair with cosine $\rho_{\mathrm{correct}}$ bypasses the text encoder but preserves
$\lambda_{\mathrm{eff}}$. These \emph{Gram-matched random anchors} permit the same probability maps
as correct prompts. Any performance gap after separately training the adapters is therefore a
finite-training residual, not a loss of representational capability.

Next we ask whether scale, rather than words, explains a gap when anchors are replaced. \emph{Unmatched
orthogonal anchors} are seeded, perpendicular unit vectors with zero cosine. They also contain no
text, but change $\lambda_{\mathrm{eff}}$. For \emph{matched orthogonal anchors}, we reuse
those same vectors and multiply the similarity scale during training and evaluation by
$\sqrt{1-\rho_{\mathrm{correct}}}$. This restores the original $\lambda_{\mathrm{eff}}$ without
restoring text. The orthogonal pair is thus a scale-changing control, and its matched version is a
second exact probability-map-level control: recovery after scale matching implicates scale rather
than linguistic coordinates.

Finally, a \emph{mismatched-dataset bank} tests a more recognizable semantic error. We replace the
target's complete description banks with those of a fixed donor, changing anatomy, modality, and
state descriptions together; for example, Kaggle BrainMRI is scored with LiverCT banks. This is a
stress test, not an exact test of meaning, because the donor need not match
$\rho_{\mathrm{correct}}$. A difference cannot isolate semantics, and a null result cannot show
that every word is ignored. We fix donor mappings and full-bank construction before training.
The two scale-matched text-free conditions carry the semantic-necessity test; the mismatched-dataset bank
provides complementary evidence. The audit follows controlled identifiability and capacity-matched
controls \citep{locatello2019challenging,hewitt2019control} to test a possible geometric shortcut
\citep{geirhos2020shortcut}. Section~\ref{sec:experimental_setup} gives the protocol;
Appendices~\ref{app:mismatched_dataset_prompts} and~\ref{app:anchor_geometry_values} document
the donor banks and geometry.

\phantomsection\label{sec:architectural_extensions}
\ptitle{Architectural and model extensions.} The exact single-pair result assumes fixed anchors
and a visual adapter that can absorb output rotations. We test how far its observed scale dependence
carries as the detector becomes more complex. First, approximately parameter-matched linear, MLP,
and contextual Transformer adapters retain rotation closure but differ in nonlinear capacity and
patch interaction. This tests whether the result depends on our linear head. We then retain the
published visual and scoring pipelines of MVFA~\citep{huang2024mvfa} and
MadCLIP~\citep{shiri2025madclip}. MVFA inserts residual adapters at multiple levels of one CLIP
image encoder and uses separate detection and segmentation readouts; its prescribed anchors permit
the same scale controls. MadCLIP instead has separate normal and abnormal adaptation branches within
one CLIP image encoder and learns its prompts, so its anchor cosine is fitted. There we replace the
prompt path with a text-free learned scale on a random frame. These studies test empirical reach
where the fixed-anchor theorem alone does not guarantee equivalence. Appendix~\ref{app:external_models}
specifies the arms, optimization, scoring, and checkpoint rules.

\phantomsection\label{sec:industrial_instruments}
\ptitle{Unified industrial controls.} When one adapter serves many industrial categories, matching each pair's scale is not enough to
guarantee equivalence: Corollary~\ref{cor:multicategory} also requires the arrangement of their
scoring axes to stay fixed. We therefore compare two geometric interventions. A \emph{shared
rotation} applies one orthogonal transformation to the whole bank, preserving $\mM$. A single
adapter rotation can compensate, so this is the exact control for finite-training invariance.
\emph{Independent rotations} reorient categories separately. They preserve every diagonal
$M_{tt}$, hence every category's effective scale, but generally change the off-diagonal alignments.
Their performance tests whether joint geometry matters when individual scales are held fixed.

\ptitle{Frozen zero-shot medical transfer.} Theorem~\ref{thm:gram} requires the adapter to follow a new anchor frame. To test this boundary, we
train an adapter on a medical source dataset, freeze it, and score a target without target-side
training. We either retain the source pair or substitute the target pair. Comparing those swaps
for correct and random anchors asks whether transfer depends on linguistic content or consistency
with the adapter's trained coordinates (Section~\ref{sec:frozen_transfer}).

\section{Experiments}
\label{sec:experiments}

\subsection{Setup and reading rule}
\label{sec:experimental_setup}

We report results in the order of the controls: text-free anchors after retraining, changes in
anchor geometry, and adapters that cannot retrain.
We separated datasets by their source of anomaly supervision. The real-anomaly-support regime used
$K$ normal and $K$ observed abnormal samples, whereas the normal-only regime used $K$ normal samples
and generated synthetic anomalies online during training. We used the former for
BraTS2021~\citep{baid2021brats}, LiverCT~\citep{bao2024bmad}, RESC~\citep{hu2019resc}, and
OCT2017~\citep{kermany2018oct}, with $K\in\{2,4,8,16\}$; we used the latter for
BUSI~\citep{aldhabyani2020busi}, CheXpert~\citep{irvin2019chexpert}, and
Kaggle BrainMRI~\citep{zhang2024mediclip,kagglebrainmri}, with $K\in\{4,8,16,32\}$. Because the
supervision mechanisms and support grids differ, we did not pool effect magnitudes, and only
$K\in\{4,8,16\}$ supports a statement across both regimes.
We used image AUROC (AC) as the primary metric and reported pixel AUROC (AS) only where valid masks
exist. Within each table block, conditions shared the dataset, support size, vision backbone, adapter, and
scoring protocol. We compared their three-seed means descriptively without treating closeness as a
formal equivalence test. For each control, we report the mean and largest absolute
correct-minus-control image-AUROC difference within each supervision regime; signed differences
only describe direction (positive favors correct prompts). The main tables show
CLIP ViT-L/14~\citep{radford2021clip}, the one configuration whose visual features and text anchors
share a pretrained joint space and in which semantic anchors are therefore most likely to help;
Appendix~\ref{app:full_results} provides the complete mean$\pm$sd factorials for all three vision
backbones while
Appendix~\ref{app:training_protocol} gives the training, checkpoint, and evaluation protocol.

For each industrial benchmark, we trained one full-shot CLIP ViT-L/14 detector using all normal
training images and synthetic anomalies. The 15 MVTec-AD~\citep{Bergmann2019MVTecAD} or 12
VisA~\citep{Zou2022VisA} category-specific normal--abnormal pairs shared its
category-agnostic linear adapters. We scored each sample only against its own pair and
macro-averaged pixel and image AUROC across categories. Appendices~\ref{app:industrial_prompts}
and~\ref{app:synthetic_anomalies} specify the exact industrial prompt banks and synthetic-anomaly
generator.

\subsection{Few-shot medical anchor controls: effective scale, not prompt content}
\label{sec:anchor_substitution}

If prompt meaning helps after retraining, correct prompts should outperform the two exact text-free
controls despite their matched effective scale. Pooled over three vision backbones, their mean
absolute image-AUROC gaps were only $0.16\%$ and $0.18\%$ with real-anomaly support
(maxima $1.2\%$ and $1.0\%$), and $0.22\%$ and $0.26\%$ with normal-only training (maxima $0.9\%$ and $0.7\%$),
for Gram-matched random and matched orthogonal anchors respectively. These small absolute gaps,
without a consistent advantage for correct prompts, show that finite training closely tracked
the theorem's attainable equivalence in the studied settings.
Table~\ref{tab:semantic_necessity_audit_clip} reports CLIP ViT-L/14;
Appendix Tables~\ref{tab:semantic_necessity_audit_real} and~\ref{tab:semantic_necessity_audit_syn} report all
three vision backbones.

The next comparison asks what happens when we change scale while keeping the anchors text-free.
Without scale matching, the mean
absolute gap increased to $1.46\%$ with real-anomaly support and $0.93\%$ with normal-only training
(maxima $11.8\%$ and $3.4\%$); the $11.8\%$ maximum occurred on DINOv3-L/16~\citep{simeoni2025dinov3} OCT2017 at $K=2$
($93.6\%$ versus $81.8\%$).
Restoring scale with the same orthogonal vectors restored the near-zero gaps above, identifying
scale rather than text removal as the source of this deterioration. The mismatched-dataset bank
changed the complete descriptions yet differed by at most $1.7\%$ and $1.3\%$.

\begin{table}[!tbp]
\centering
\caption{CLIP ViT-L/14 anchor substitution at $K{=}16$ in both supervision regimes.
\textbf{Left:} real-anomaly support ($16$ normal $+$ $16$ observed abnormal samples).
\textbf{Right:} normal-only ($16$ normal samples with online synthetic anomalies). Cells are
image-AUROC (AC) mean$\pm$sd over three seeds (\%). Section~\ref{sec:instruments} defines the five
anchor conditions. Pixel AUROC, the remaining support sizes, and the complete factorials for all
three vision backbones are in Appendix Tables~\ref{tab:semantic_necessity_audit_real}
and~\ref{tab:semantic_necessity_audit_syn}.}
\label{tab:semantic_necessity_audit_clip}
\begingroup
\small
\setlength{\tabcolsep}{4pt}
\renewcommand{\arraystretch}{0.95}
\begin{adjustbox}{max width=\textwidth,center}
\begin{tabular}{lccccccccc}
\toprule
& \multicolumn{5}{c}{\textbf{Real-anomaly support}} & \multicolumn{4}{c}{\textbf{Normal-only}} \\
\cmidrule(lr){2-6}\cmidrule(lr){7-10}
\textbf{Anchor condition} & \textbf{OCT2017} & \textbf{BraTS2021} & \textbf{LiverCT} & \textbf{RESC} &
\textbf{Avg} & \textbf{Kaggle BrainMRI} & \textbf{CheXpert} & \textbf{BUSI} & \textbf{Avg} \\
\midrule
Correct prompts & $99.6{\pm}0.1$ & $96.9{\pm}0.2$ & $91.6{\pm}1.1$ & $97.2{\pm}0.9$ & $96.3$ & $92.9{\pm}0.8$ & $69.3{\pm}1.6$ & $88.2{\pm}0.6$ & $83.4$ \\
Mismatched-dataset bank & $99.7{\pm}0.1$ & $96.9{\pm}0.2$ & $90.8{\pm}1.3$ & $97.2{\pm}0.9$ & $96.1$ & $92.9{\pm}0.7$ & $69.0{\pm}2.3$ & $87.7{\pm}0.9$ & $83.2$ \\
Gram-matched random & $99.6{\pm}0.1$ & $96.9{\pm}0.2$ & $91.5{\pm}1.2$ & $97.2{\pm}0.9$ & $96.3$ & $92.9{\pm}0.2$ & $69.5{\pm}1.3$ & $88.4{\pm}0.6$ & $83.6$ \\
Unmatched orthogonal & $99.2{\pm}0.0$ & $97.0{\pm}0.1$ & $89.0{\pm}0.3$ & $97.1{\pm}1.0$ & $95.6$ & $92.5{\pm}0.4$ & $69.8{\pm}1.3$ & $87.8{\pm}0.6$ & $83.4$ \\
Matched orthogonal & $99.6{\pm}0.1$ & $96.9{\pm}0.2$ & $91.6{\pm}1.4$ & $97.2{\pm}0.9$ & $96.3$ & $92.9{\pm}0.7$ & $69.4{\pm}1.7$ & $87.8{\pm}0.8$ & $83.4$ \\
\bottomrule
\end{tabular}
\end{adjustbox}
\endgroup
\end{table}

\subsection{Extensions across adapter families, residual detectors, and learned prompts}
\label{sec:adapter_family_results}

\ptitle{Rotation-closed adapter families.} We next ask whether the result depends on adapter
architecture. The approximately parameter-matched linear, MLP, and contextual Transformer adapters
differ in nonlinear capacity and patch interaction but remain rotation-closed
(Appendix~\ref{app:closure}). In regime averages across two vision backbones,
correct-prompt versus Gram-matched random gaps stayed within $0.5\%$ image AUROC, with no
consistent sign. Unmatched orthogonal anchors instead reduced contextual-Transformer performance by
$5.4\%$ on CLIP ViT-L/14 and $18.7\%$ on DINOv3-L/16 with real-anomaly support; matching their
scale brought both losses back to within $0.2\%$. Thus the exact controls agree across three
implementations that still satisfy the theorem's architectural premise. Appendix Tables~\ref{tab:adapter_family_dataset_ac_clip} and~\ref{tab:adapter_family_dataset_ac}
give the per-dataset results and parameter counts.

MVFA's prescribed prompts support three text-free controls: Gram-matched random anchors preserve
their geometry, unmatched orthogonal anchors change effective scale, and scale matching restores it. Because
MadCLIP learns its prompt pair, we compare its original implementation (learned prompts) with two text-free
arms on a frozen random orthonormal frame: one learns the cosine (learned scale), the other fixes it
at the prompt-like value $+0.986$ (prompt-scale).

\begin{table}[!tb]
\centering
\caption{External-detector anchor controls at $K{=}16$ (AUROC \%, mean$\pm$sd over three seeds).
AC is image and AS pixel AUROC; OCT2017 pixel AUROC is not reported, so it has no AS column and Avg AS
averages BraTS2021, LiverCT, and RESC. Checkpoints follow each released method's test-selection rule;
Appendix~\ref{app:external_models} gives the complete grids and provenance.}
\label{tab:external_k16}
\begingroup
\scriptsize
\setlength{\tabcolsep}{3.2pt}
\renewcommand{\arraystretch}{0.94}
\begin{adjustbox}{max width=\textwidth,center}
\begin{tabular}{llc cc cc cc cc}
\toprule
\multirow{2}{*}{\textbf{Method}} & \multirow{2}{*}{\textbf{Anchor condition}} & \textbf{OCT2017} &
\multicolumn{2}{c}{\textbf{BraTS2021}} & \multicolumn{2}{c}{\textbf{LiverCT}} &
\multicolumn{2}{c}{\textbf{RESC}} & \multicolumn{2}{c}{\textbf{Avg}} \\
\cmidrule(lr){3-3} \cmidrule(lr){4-5} \cmidrule(lr){6-7} \cmidrule(lr){8-9} \cmidrule(lr){10-11}
& & AC & AC & AS & AC & AS & AC & AS & AC & AS \\
\midrule
\multirow{4}{*}{MVFA} & Correct prompts & $99.9{\pm}0.0$ & $95.0{\pm}0.6$ & $97.5{\pm}0.9$ & $86.1{\pm}1.5$ & $98.8{\pm}1.4$ & $98.2{\pm}0.4$ & $99.4{\pm}0.0$ & $94.8$ & $98.6$ \\
 & Gram-matched random & $99.9{\pm}0.0$ & $94.1{\pm}1.5$ & $97.6{\pm}0.1$ & $83.7{\pm}2.5$ & $99.1{\pm}0.9$ & $98.0{\pm}0.2$ & $99.4{\pm}0.1$ & $93.9$ & $98.7$ \\
 & Unmatched orthogonal & $92.2{\pm}2.1$ & $66.8{\pm}4.0$ & $93.1{\pm}1.1$ & $63.0{\pm}5.2$ & $96.5{\pm}0.3$ & $72.9{\pm}9.1$ & $89.9{\pm}2.4$ & $73.7$ & $93.2$ \\
 & Matched orthogonal & $99.9{\pm}0.1$ & $94.5{\pm}1.1$ & $97.5{\pm}0.6$ & $85.5{\pm}1.8$ & $99.0{\pm}1.0$ & $98.0{\pm}0.3$ & $99.3{\pm}0.1$ & $94.5$ & $98.6$ \\
\midrule
\multirow{3}{*}{MadCLIP} & Learned prompts & $99.6{\pm}0.2$ & $96.0{\pm}0.5$ & $97.6{\pm}0.4$ & $90.5{\pm}0.9$ & $99.5{\pm}0.5$ & $98.6{\pm}0.2$ & $99.4{\pm}0.1$ & $96.2$ & $98.8$ \\
 & Learned scale (no text) & $99.7{\pm}0.1$ & $95.8{\pm}0.6$ & $97.7{\pm}0.6$ & $91.1{\pm}0.8$ & $99.1{\pm}0.7$ & $98.6{\pm}0.2$ & $99.4{\pm}0.1$ & $96.3$ & $98.7$ \\
 & Prompt-scale (no text) & $99.6{\pm}0.1$ & $96.1{\pm}0.5$ & $98.1{\pm}0.1$ & $87.7{\pm}3.4$ & $99.1{\pm}1.0$ & $97.7{\pm}0.2$ & $99.2{\pm}0.1$ & $95.3$ & $98.8$ \\
\bottomrule
\end{tabular}
\end{adjustbox}
\endgroup
\end{table}

\ptitle{MVFA: residual, multi-level visual adaptation.} MVFA~\citep{huang2024mvfa} inserts
multiple residual adapter branches into one CLIP image encoder and uses separate readouts for
detection and segmentation. Its prescribed anchors admit the same controls, although the residual
architecture itself is outside the exact closure result. At $K{=}16$, mean image AUROC was $94.8\%$
with prompts and $93.9\%$ with Gram-matched random anchors. Unmatched orthogonal anchors reduced it to $73.7\%$;
restoring only their scale recovered $94.5\%$ (Table~\ref{tab:external_k16}). This provides the clearest
external drop-and-recovery example, as empirical evidence beyond the theorem's guaranteed scope.

\ptitle{MadCLIP: learned prompts.} MadCLIP~\citep{shiri2025madclip} uses separate normal and
abnormal adaptation branches in one CLIP image encoder and learns the prompt pair, making its cosine
an outcome of training. A fixed-anchor replacement therefore cannot reproduce its internal protocol.
At $K{=}16$, learned prompts and the text-free learned scale obtained $96.2\%$ and $96.3\%$ mean image
AUROC. Across the full grid, their mean learned cosines were $-0.984$ and $-0.994$, respectively.
Both paths approach antipodal anchors, the maximum effective scale: given freedom to learn its
prompts, MadCLIP finds what a single scalar finds.
The prompt-scale arm reached $95.3\%$ at $K{=}16$ but was less stable at small $K$; Appendix
Tables~\ref{tab:mvfa_full} and~\ref{tab:madclip_full} report all image and pixel results.

\subsection{Unified multi-category detection: joint anchor geometry}
\label{sec:industrial_extension}

The industrial comparison asks whether a shared adapter needs more than each category's scale.
The corollary points to the relative directions of the category scoring axes, but a table of
pairwise angles would be hard to read. We summarize their spread with the participation-ratio
effective rank. It is near $1$ when the axes mostly lie along one direction and rises as they
occupy more directions. For the nonnegative eigenvalues
$\mu_1,\ldots,\mu_T$ of $\mM$, it is
\begin{equation}
    r_{\mathrm{eff}}(\mM)
    = \frac{\left(\sum_{i=1}^{T}\mu_i\right)^2}
           {\sum_{i=1}^{T}\mu_i^2},
    \qquad 1 \le r_{\mathrm{eff}}(\mM) \le \operatorname{rank}(\mM).
\end{equation}
The prompt axes were concentrated: $r_{\mathrm{eff}}=1.79$ on MVTec-AD and $1.78$ on VisA (Appendix Table~\ref{tab:industrial_anchor_scales}). Independent
rotations raised these values to $14.28$--$14.29$ of 15 and about $11.10$ of 12, while preserving
every category-specific scale. The shared rotation stayed within $0.1\%$ of correct prompts
on all four AUROC endpoints, consistent with its exact equivalence. Independent rotations instead
gained $5.6\%$ image and $4.4\%$ pixel AUROC on VisA and $1.2\%$ image and $3.0\%$ pixel AUROC on MVTec-AD
(Table~\ref{tab:industrial_anchor_conditions}). Thus, individual scales did not settle performance:
in these benchmarks the spread-out bank worked better for the shared adapter.
DINOv2-L/14~\citep{oquab2024dinov2} and
DINOv3-L/16 reproduced the result on both benchmarks, with gains up to $8.8\%$ image and $8.3\%$ pixel
AUROC on VisA; Appendix
Table~\ref{tab:industrial_anchor_conditions_backbones} reports all three vision backbones.

One plausible mechanism is reduced cross-category interference. Because the detector scores category
$t$ through $\inner{\hat\vu}{\vdelta_t}$, concentrating difference vectors in shared directions may make
category updates compete, whereas spreading them provides distinct readouts. We did not measure
gradient alignment, so this mechanism remains unverified.

\begin{table}[!tb]
\centering
\caption{Full-shot industrial results (\%), mean$\pm$sd over three training seeds
(and geometry seeds where applicable). AC and AS are image and pixel AUROC, category-macro means
over the $T$ categories; Appendix Table~\ref{tab:industrial_anchor_conditions_backbones} adds PRO,
AUPR, and F1max.}
\label{tab:industrial_anchor_conditions}
\begingroup
\scriptsize
\setlength{\tabcolsep}{6pt}
\renewcommand{\arraystretch}{1.0}
\centering
\begin{tabular}{lcccc}
\toprule
\multirow{2}{*}{\textbf{Anchor condition}} & \multicolumn{2}{c}{\textbf{VisA}} & \multicolumn{2}{c}{\textbf{MVTec-AD}} \\
\cmidrule(lr){2-3}\cmidrule(lr){4-5}
& AC & AS & AC & AS \\
\midrule
Correct prompts & $86.2{\pm}0.1$ & $91.7{\pm}0.1$ & $96.2{\pm}0.1$ & $93.6{\pm}0.1$ \\
Shared rotation & $86.3{\pm}0.1$ & $91.6{\pm}0.0$ & $96.2{\pm}0.1$ & $93.6{\pm}0.1$ \\
Independent rotations & $91.8{\pm}0.1$ & $96.1{\pm}0.1$ & $97.4{\pm}0.1$ & $96.6{\pm}0.0$ \\
\bottomrule
\end{tabular}
\endgroup
\end{table}
\subsection{Frozen zero-shot transfer: the boundary of anchor equivalence}
\label{sec:frozen_transfer}

The final test asks what happens when an adapter cannot follow changed anchors. At $K=16$, we
trained each CLIP ViT-L/14 adapter on a medical source dataset, froze it, and evaluated a target
without target-side training. We scored the same target images with either the source pair that
trained the adapter or the target dataset's pair. Their AUROC difference (source minus target) is
the anchor-swap cost, separate from the domain shift common to both scores
(Appendix Table~\ref{tab:transfer_anchor_swap}).

Retaining a text-free random source pair preserved transfer, whereas redrawing a random target
pair without changing the frozen adapter was costly. Relative to correct
prompts, retained random pairs differed by only $0.1\%$ and $0.3\%$ image AUROC on CLIP ViT-L/14
and $0.2\%$ and $0.5\%$ on DINOv3-L/16 in the real-anomaly-support and normal-only regimes.
Random-pair swaps cost $9.5\%$ and $19.1\%$ on CLIP and $19.6\%$ and $13.2\%$ on DINOv3; the
corresponding correct-prompt swaps cost $0.6\%$ and $0.8\%$ on CLIP, and $-1.3\%$ and $-0.3\%$ on
DINOv3. Correct target
prompts shared a text-encoder frame, whereas independently drawn random target pairs did not
(Appendix Table~\ref{tab:transfer_anchor_swap_dinov3}).

Reusing the random source pair preserves the trained readout without linguistic coordinates. A
frozen adapter cannot make the compensating change used in Theorem~\ref{thm:gram}; swapping pairs
can change direction and scale together, so these results establish the importance of keeping the
adapter and its anchor pair coupled, not which geometric component causes the swap cost.

\section{Conclusion}
\label{sec:conclusion}

\ptitle{Summary.}
We asked which information from the text anchors of a trained fixed-anchor vision--language
anomaly detector affects its representational capacity and observed performance: semantic content,
effective scale, or joint anchor geometry. The answer is geometry, not semantic content. For a
single anchor pair, a rotation-closed adapter reparameterizes the difference-vector direction away,
so scale-matched random or orthogonal anchors permit exactly the maps of correct prompts and
matched them after training, while changing only the scale opened gaps that restoring it closed.
The same pattern held in MVFA and in MadCLIP with a text-free learned scale. When one adapter
serves many categories, the joint geometry of their difference vectors matters: spreading the
near-collinear prompt axes apart, without adding meaning, improved both industrial benchmarks.
Under frozen transfer, a retained random frame transferred as well as prompts, because an adapter
and its anchors form a coupled coordinate system.

\ptitle{Practical implications.}
Within this family, curating expert prompt banks or searching for better wording and text encoders
cannot improve accuracy through meaning: a text-free anchor pair with matched or learned effective
scale suffices, and a unified detector can improve by spreading its category axes apart. A frozen
adapter should be deployed with the anchors it was trained with, and claims that prompt semantics
drive performance should control for effective scale.

\ptitle{Limitations.}
Our guarantees require prescribed fixed anchors, rotation-closed adapters, and cosine-based binary
maps (Proposition~\ref{prop:breaking}). MVFA and MadCLIP lie outside this premise, yet text-free
anchors matched their prompts there too, suggesting the absence of a semantic benefit extends
beyond the theorem; this evidence is empirical, not guaranteed. All anchors come from one frozen
text encoder, so the near-collinearity of category axes may be specific to it. We compared only
the reference effective scale with the larger one of orthogonal anchors, without sweeping it, so
the best scale remains unknown; we also did not measure the proposed interference mechanism.
Closeness between conditions is descriptive over three seeds, not a formal equivalence test.

\clearpage
\subsection*{Reproducibility statement}
All datasets are publicly available (Section~\ref{sec:experimental_setup}).
Appendix~\ref{app:training_protocol} and Table~\ref{tab:training_protocol} give the complete
training, checkpoint, and evaluation protocol, including seeds, and
Appendix~\ref{app:synthetic_anomalies} specifies both synthetic-anomaly generators.
Appendices~\ref{app:mismatched_dataset_prompts} and~\ref{app:industrial_prompts} list every prompt
sentence, Appendix~\ref{app:anchor_geometry_values} reports the resulting anchor cosines and
effective scales, and Section~\ref{sec:instruments} and Appendix~\ref{app:external_models} define
the random and orthogonal anchor constructions. Appendix~\ref{app:external_models} also documents
every change to MVFA and MadCLIP. Complete proofs are in Appendix~\ref{app:proofs}, and complete
result tables in Appendix~\ref{app:full_results}.

\subsection*{AI use statement}
Generative AI tools were used to assist with mathematical derivations and proof
checking, experimental methodology, research-code implementation and
verification, interpretation of results, and the drafting and polishing of the
manuscript text. The authors independently reviewed and verified all AI-assisted
content, executed the reported experiments, and take full responsibility for the
final manuscript and its claims.

\clearpage
\bibliographystyle{iclr2027_conference}
\bibliography{iclr2027_conference}

\clearpage
\appendix
\renewcommand{\floatpagefraction}{0.8}

\section{Theoretical Details and Proofs}
\label{app:proofs}

This appendix first verifies that the implemented loss and metrics factor through abnormal
probability maps. It then gives the Gram-to-rotation lemma, the binary proof, and the shared
multi-category proof. Prompt banks, anchor geometry, rotation closure, finite-training residuals,
and the proof of Proposition~\ref{prop:breaking} follow.

\subsection{From the implemented detector to abnormal probability maps}
\label{app:objective}

Theorem~\ref{thm:gram} concerns the binary probability maps actually consumed by training and
evaluation. We therefore first verify that the complete implementation factors through the
normal--abnormal logit difference.

\ptitle{Layer-wise scorer.} The vision backbone retains \(L=4\) intermediate layers. Layer \(\ell\) has its
own adapter \(g^{(\ell)}\), but every layer uses the same fixed unit anchor pair. With
\(\hat\vu^{(\ell)}_p=\normalize(g^{(\ell)}(\vz^{(\ell)}_p))\),
\begin{equation}
s^{(\ell)}_{p,c}
=
\tau\inner{\hat\vu^{(\ell)}_p}{\va_c},
\qquad c\in\{N,A\},
\label{eq:layer_scores}
\end{equation}
where \(\tau=100\) in training and \(\tau=20\) at evaluation, each multiplied by the configured
similarity-scale multiplier. Define
\begin{equation}
\Delta s^{(\ell)}_p=s^{(\ell)}_{p,A}-s^{(\ell)}_{p,N}.
\label{eq:patch_logit_difference}
\end{equation}

\ptitle{Upsampling.} The two logit maps are bilinearly resampled from the patch grid to the image
grid before softmax. If \(w_{qp}\) denotes the class-independent interpolation weight from patch
\(p\) to pixel \(q\), linearity gives
\begin{equation}
\widetilde{\Delta s}^{(\ell)}_q
=
\tilde s^{(\ell)}_{q,A}-\tilde s^{(\ell)}_{q,N}
=
\sum_{p=1}^{P}w_{qp}\Delta s^{(\ell)}_p .
\label{eq:upsample_commutes}
\end{equation}

\ptitle{Binary probability.} A two-class softmax is a logistic function of this difference:
\begin{equation}
p^{(\ell)}_q
=
\frac{e^{\tilde s^{(\ell)}_{q,A}}}
{e^{\tilde s^{(\ell)}_{q,N}}+e^{\tilde s^{(\ell)}_{q,A}}}
=
\sigma\!\left(\widetilde{\Delta s}^{(\ell)}_q\right),
\qquad
p^{(\ell)}_{q,N}=1-p^{(\ell)}_q .
\label{eq:softmax_is_logistic}
\end{equation}
Thus the normal channel carries no information beyond the abnormal map.

\ptitle{Training loss.} For target mask \(y\), the focal term uses the probability assigned to the
target,
\begin{equation}
\pi^{(\ell)}_q
=
p^{(\ell)}_q y_q+\bigl(1-p^{(\ell)}_q\bigr)(1-y_q),
\qquad
\cL_{\mathrm{foc}}
=
\frac{1}{BR^2}\sum_q
-\bigl(1-\pi^{(\ell)}_q\bigr)^\gamma\log\pi^{(\ell)}_q ,
\label{eq:focal_term}
\end{equation}
with \(\gamma=2\) and the implemented numerical clamp. The Dice term reads the abnormal map
directly:
\begin{equation}
\cL_{\mathrm{dice}}
=
1-\frac{1}{B}\sum_{b=1}^{B}
\frac{2\sum_q p^{(\ell)}_{b,q}y_{b,q}+1}
{\sum_q p^{(\ell)}_{b,q}+\sum_q y_{b,q}+1}.
\label{eq:dice_term}
\end{equation}
Consequently,
\begin{equation}
\cL
=
\sum_{\ell=1}^{L}
\left(\cL_{\mathrm{foc}}+\cL_{\mathrm{dice}}\right)
=
\Phi\!\left(\Delta s^{(1)},\ldots,\Delta s^{(L)};y\right)
\label{eq:loss_factors}
\end{equation}
for a fixed functional \(\Phi\). Neither loss term observes the two logits separately.

\ptitle{Evaluation.} The layer maps are averaged and the image score is the top-fraction mean,
\begin{equation}
\mathrm{score}(x)
=
\frac{1}{k}\sum\TopK_q
\left(\frac{1}{L}\sum_{\ell=1}^{L}p^{(\ell)}_q\right),
\qquad
k=\max\!\left(1,\left\lfloor\alpha R^2\right\rfloor\right).
\label{eq:image_score}
\end{equation}
Pixel AUROC depends on the averaged maps and masks; image AUROC depends on these scores and image
labels. Both therefore factor through the same abnormal probability maps as the training objective.

\subsection{Gram-to-rotation lemma}
\label{app:gram_lemma}

\begin{lemma}[Equal Gram matrices give one orthogonal map]
\label{lem:gram_rotation}
Let \(\mX,\mX'\in\R^{m\times d}\) have corresponding rows and satisfy
\(\mX\mX^\top=\mX'{\mX'}^\top\). Then some \(\mQ\in O(d)\) maps every row vector of \(\mX\) to the
corresponding row vector of \(\mX'\).
\end{lemma}

\begin{proof}
The row correspondence defines a linear map from \(\operatorname{span}\{\vx_i\}\) to
\(\operatorname{span}\{\vx'_i\}\). Equality of the Gram matrices makes it well-defined and
inner-product preserving: every two linear combinations have the same inner product before and
after the map. It is therefore an isometry between the two row spans. Extend orthonormal bases of
those spans to bases of \(\R^d\); mapping the completed bases gives an orthogonal \(\mQ\).
\end{proof}

\subsection{Complete proof of Theorem~\ref{thm:gram}}
\label{app:gram}

Let \(\vdelta=\va_A-\va_N\), \(\vdelta'=\va'_A-\va'_N\), and assume first that both are nonzero. The score
definition in~\eqref{eq:layer_scores} gives
\begin{equation}
\Delta s^{(\ell)}_p
=
\tau\inner{\hat\vu^{(\ell)}_p}{\vdelta}
=
\tau\norm{\vdelta}\inner{\hat\vu^{(\ell)}_p}{\hat\vdelta}.
\label{eq:readout}
\end{equation}
Apply Lemma~\ref{lem:gram_rotation} to the one-row matrices \(\hat\vdelta^\top\) and
\(\widehat{\vdelta'}^{\top}\). It supplies \(\mQ\in O(d)\) with
\begin{equation}
\mQ\hat\vdelta=\hat\vdelta'.
\label{eq:anchor_rotation}
\end{equation}
For every layer set \(g'^{(\ell)}=\mQ\circ g^{(\ell)}\). Rotation closure makes all transformed
adapters admissible, and
\begin{equation}
\tau'\norm{\vdelta'}\inner{\mQ\hat\vu^{(\ell)}_p}{\hat\vdelta'}
=
\tau\norm{\vdelta}\inner{\hat\vu^{(\ell)}_p}{\hat\vdelta}
\end{equation}
by~\eqref{eq:matched_effective_logit_scale}. Thus corresponding adapters produce identical
logit-difference fields. Equations~\plaineqref{eq:upsample_commutes}--\plaineqref{eq:image_score} then give
identical probability maps, losses, and evaluation scores. Applying \(\mQ^\top\) proves the reverse
inclusion of attainable maps. If both difference vectors vanish, both binary probabilities are
identically \(1/2\); equality in~\eqref{eq:matched_effective_logit_scale} excludes the case in which
only one vanishes.

For unit anchors,
\begin{equation}
\norm{\vdelta}^2
=
\norm{\va_A}^2+\norm{\va_N}^2-2\inner{\va_A}{\va_N}
=
2-2\rho,
\end{equation}
which yields \(\lambda_{\mathrm{eff}}=\tau\sqrt{2-2\rho}\). An orthogonal pair scaled by a multiplier
\(c\) matches the correct pair when
\begin{equation}
c\tau\sqrt{2}
=
\tau\sqrt{2-2\rho_{\mathrm{correct}}}
\qquad\Longleftrightarrow\qquad
c=\sqrt{1-\rho_{\mathrm{correct}}}.
\label{eq:matched_multiplier}
\end{equation}
The implementation applies \(c\) in training and evaluation alike. Normalizing the adapter output
is essential: it prevents the adapter from absorbing \(\norm{\vdelta}\) by rescaling itself.

\subsection{Unified multi-category specialization}
\label{app:multicategory}

Category \(t\) has unit anchors \((\va_t^N,\va_t^A)\), difference vector
\(\vdelta_t=\va_t^A-\va_t^N\), and known category label. The shared detector selects only that pair:
\begin{equation}
p_{t,A}(\vz;g)
=
\sigma\!\left(
\tau\inner{\normalize(g(\vz))}{\vdelta_t}
\right).
\label{eq:category_conditioned_score}
\end{equation}
The category selects a readout pair, not a different adapter. Let \(\mD\) stack the
\(\vdelta_t^\top\) rows and let \(\mM=\mD\mD^\top\).

\ptitle{Proof of Corollary~\ref{cor:multicategory}.} If
\(\mD\mD^\top=\mD'{\mD'}^\top\), Lemma~\ref{lem:gram_rotation} supplies one \(\mQ\in O(d)\) with
\(\vdelta'_t=\mQ\vdelta_t\) for every category. Compose the same \(\mQ\) into every layer-wise adapter.
For a category-\(t\) sample,
\begin{equation}
\inner{\mQ\normalize(g(\vz))}{\vdelta'_t}
=
\inner{\mQ\normalize(g(\vz))}{\mQ\vdelta_t}
=
\inner{\normalize(g(\vz))}{\vdelta_t}.
\end{equation}
Thus every selected binary probability is preserved. The loss and evaluation factorization in
Appendix~\ref{app:objective} completes the proof. Equality of \(\mM\) is also necessary and
sufficient for the two finite difference-vector configurations to lie in the same global
orthogonal orbit. We make no converse claim about equality of hypothesis classes for an arbitrary
adapter family and input domain.

\ptitle{Relation to the full anchor Gram.} Stack the anchors as rows
\[
\mA=[\va_1^N;\va_1^A;\ldots;\va_T^N;\va_T^A],
\qquad
\mG=\mA\mA^\top,
\]
and let \(\mR=\mI_T\otimes[-1\;\;1]\). Then
\begin{equation}
\mD=\mR\mA,
\qquad
\mM=\mR\mG\mR^\top.
\label{eq:difference_from_full_gram}
\end{equation}
Hence an exact full-Gram match implies an exact difference-Gram match, but the converse fails.
Writing
\[
\va_t^N=\vm_t-\vdelta_t/2,
\qquad
\va_t^A=\vm_t+\vdelta_t/2,
\qquad
\inner{\vm_t}{\vdelta_t}=0,
\]
shows why: rotating a midpoint \(\vm_t\), while preserving its norm, inside the orthogonal
complement of \(\vdelta_t\) changes
entries of \(\mG\) while leaving every \(\vdelta_t\), probability, loss, and metric unchanged. The full
\(2T\times2T\) anchor Gram is therefore sufficient but not necessary. The exact statement used in
this paper is that \(\mM\) completely characterizes the $T$ difference vectors up to one global
orthogonal rotation.

\ptitle{Shared-rotation control.} The shared-rotation condition applies one seeded \(\mQ\) to every correct anchor. It consequently
rotates every \(\vdelta_t\) by the same \(\mQ\) and preserves all of \(\mM\). Correct prompts versus this
shared rotation measures only the finite-training residual allowed by Appendix~\ref{app:residual}.

\ptitle{Independent-rotation control.} For each category, the independent-rotation condition
constructs a seeded random pair with the same
\(\rho_t=\inner{\va_t^N}{\va_t^A}\). Equivalently, some category-specific \(\mQ_t\) maps the correct
pair to the replacement pair. It preserves
\[
M_{tt}=\norm{\vdelta_t}^2=2-2\rho_t
\]
but generally changes \(M_{ts}\) for \(t\ne s\). Shared rotation versus independent rotations
therefore changes the joint geometry while retaining every category's effective
scale. Neither randomized condition is independent of text-derived geometry: the shared condition
inherits the full reference joint geometry, and the independent condition inherits each
\(\rho_t\).

\ptitle{Seed scope of the current result.} Every industrial condition uses three training
seeds. In each randomized condition, the anchor seed equals the training seed, so the runs sample
three shared rotations or three independently rotated sets of difference vectors. Correct-prompt anchors are
deterministic and vary only through training.

\subsection{Clinical prompt banks and mismatched-dataset intervention}
\label{app:mismatched_dataset_prompts}

All seven medical datasets use the same fixed prompt protocol: 14 normal and 14
abnormal complete sentences, with the same ordered grammatical scaffolds for every dataset. Only
the anatomy, modality, and dataset-verified binary finding change. Each bank encodes the slice- or
image-level binary task of its own dataset, following the
benchmark constructions and their underlying datasets
\citep{bao2024bmad,zhang2024mediclip,baid2021brats,landman2015btcv,bilic2023lits,hu2019resc,
kermany2018oct,aldhabyani2020busi,irvin2019chexpert,kagglebrainmri}.
LiverCT follows the BMAD construction, which combines BTCV and LiTS slices
\citep{bao2024bmad,landman2015btcv,bilic2023lits}.
In particular, a normal slice is defined by the absence of the evaluated annotation on that image;
we do not infer that it came from a healthy patient. Abnormal subtypes receive equal prompt weight,
not dataset-frequency weight.

The medical semantic stress test replaces the complete prompt specification of target dataset $t$
by the prompt specification of a fixed donor dataset $w(t)$. Let
$\mathcal{P}_t^N$ and $\mathcal{P}_t^A$ denote all normal and abnormal descriptions used for the correct target
condition. The mismatched-dataset condition uses $\mathcal{P}_{w(t)}^N$ and $\mathcal{P}_{w(t)}^A$ without editing
them to mention the target. It therefore changes the organ, imaging modality, normality definition,
abnormality definition, state descriptions, and any dataset-specific templates together. It is not
implemented by replacing one anatomy noun inside the target templates.

Both donor banks are encoded by the same frozen text encoder used for the correct condition. The
normal embeddings are averaged and normalized to form $\va_N^{\mathrm{mis}}$, and the abnormal
embeddings are averaged and normalized to form $\va_A^{\mathrm{mis}}$. The anchors remain fixed
throughout the run, and the visual adapter is trained on the target dataset exactly as in the other
anchor conditions. Thus only the fixed prompt-derived anchor pair changes; images, supervision,
support owners, optimization settings, and evaluation remain target-specific.

Table~\ref{tab:mismatched_dataset_mapping} specifies the donor mapping a priori. It is a
derangement, uses each dataset once as a donor, and pairs every target with a different anatomy and
imaging modality. For example, Kaggle BrainMRI uses the complete LiverCT prompt banks: its anchors
describe normal and abnormal liver CT appearances rather than normal and abnormal brain MR
appearances.

\begin{table}[htbp]
\centering
\caption{Fixed target-to-donor mapping for the mismatched-dataset prompt-bank intervention. The
donor contributes its complete normal and abnormal description banks; no donor prompt is rewritten
to mention the target dataset.}
\label{tab:mismatched_dataset_mapping}
\begingroup
\footnotesize
\setlength{\tabcolsep}{5pt}
\renewcommand{\arraystretch}{1.15}
\begin{tabular}{lll}
\toprule
Target dataset & Target domain & Donor prompt bank \\
\midrule
BraTS2021 & brain MRI & BUSI (breast ultrasound) \\
LiverCT & liver CT & OCT2017 (retinal OCT) \\
RESC & retinal OCT & CheXpert (chest radiograph) \\
OCT2017 & retinal OCT & Kaggle BrainMRI (brain MRI) \\
BUSI & breast ultrasound & RESC (retinal OCT) \\
CheXpert & chest radiograph & BraTS2021 (brain MRI) \\
Kaggle BrainMRI & brain MRI & LiverCT (liver CT) \\
\bottomrule
\end{tabular}
\endgroup
\end{table}

This condition deliberately does not match
$\rho=\inner{\va_N^{\mathrm{mis}}}{\va_A^{\mathrm{mis}}}$ to the correct pair. It is therefore
not an exact probability-map control and cannot by itself separate semantics from effective
scale. Its role is narrower: it tests whether training with a grossly incorrect, internally coherent
textual specification produces a repeatable performance change. The exact semantic-necessity test
continues to be correct prompts versus the text-free controls matched in effective scale.

The following tables list the exact expanded banks used by both training and evaluation.

\begingroup
\footnotesize
\setlength{\tabcolsep}{5pt}
\renewcommand{\arraystretch}{1.15}
\begin{longtable}{@{}>{\raggedright\arraybackslash}p{0.485\linewidth}>{\raggedright\arraybackslash}p{0.485\linewidth}@{}}
\caption{Exact prompts for BraTS2021.}\label{tab:clinical-prompts-brain-ad}\\
\toprule
\textbf{Normal} & \textbf{Abnormal} \\
\midrule
\endfirsthead
\multicolumn{2}{c}{\tablename\ \thetable{} -- continued} \\
\toprule
\textbf{Normal} & \textbf{Abnormal} \\
\midrule
\endhead
\midrule
\multicolumn{2}{r}{Continued on next page} \\
\endfoot
\bottomrule
\endlastfoot
An axial T2-FLAIR brain MRI slice of the brain showing absence of any annotated glioma region, including edema, necrotic or non-enhancing tumor core, and enhancing tumor. & An axial T2-FLAIR brain MRI slice of the brain showing annotated brain tumor tissue. \\
A clinical axial T2-FLAIR brain MRI slice of the brain showing absence of any annotated glioma region, including edema, necrotic or non-enhancing tumor core, and enhancing tumor. & A clinical axial T2-FLAIR brain MRI slice of the brain showing a tumor-bearing axial brain slice. \\
A diagnostic axial T2-FLAIR brain MRI slice of the brain demonstrating absence of any annotated glioma region, including edema, necrotic or non-enhancing tumor core, and enhancing tumor. & A diagnostic axial T2-FLAIR brain MRI slice of the brain demonstrating glioma tissue. \\
This axial T2-FLAIR brain MRI slice of the brain shows absence of any annotated glioma region, including edema, necrotic or non-enhancing tumor core, and enhancing tumor. & This axial T2-FLAIR brain MRI slice of the brain shows a whole-tumor region. \\
This axial T2-FLAIR brain MRI slice of the brain demonstrates absence of any annotated glioma region, including edema, necrotic or non-enhancing tumor core, and enhancing tumor. & This axial T2-FLAIR brain MRI slice of the brain demonstrates peritumoral edema. \\
The brain on this axial T2-FLAIR brain MRI slice shows absence of any annotated glioma region, including edema, necrotic or non-enhancing tumor core, and enhancing tumor. & The brain on this axial T2-FLAIR brain MRI slice shows edema-associated T2-FLAIR hyperintensity. \\
The appearance of the brain on this axial T2-FLAIR brain MRI slice is consistent with absence of any annotated glioma region, including edema, necrotic or non-enhancing tumor core, and enhancing tumor. & The appearance of the brain on this axial T2-FLAIR brain MRI slice is consistent with a necrotic tumor core. \\
The primary finding on this axial T2-FLAIR brain MRI slice is absence of any annotated glioma region, including edema, necrotic or non-enhancing tumor core, and enhancing tumor. & The primary finding on this axial T2-FLAIR brain MRI slice is a non-enhancing tumor core. \\
The evaluated region of the brain on this axial T2-FLAIR brain MRI slice shows absence of any annotated glioma region, including edema, necrotic or non-enhancing tumor core, and enhancing tumor. & The evaluated region of the brain on this axial T2-FLAIR brain MRI slice shows an enhancing tumor region. \\
Image-level assessment of this axial T2-FLAIR brain MRI slice indicates absence of any annotated glioma region, including edema, necrotic or non-enhancing tumor core, and enhancing tumor. & Image-level assessment of this axial T2-FLAIR brain MRI slice indicates tumor core with surrounding edema. \\
The observable pattern on this axial T2-FLAIR brain MRI slice is consistent with absence of any annotated glioma region, including edema, necrotic or non-enhancing tumor core, and enhancing tumor. & The observable pattern on this axial T2-FLAIR brain MRI slice is consistent with enhancing and non-enhancing tumor tissue. \\
The scan appearance demonstrates absence of any annotated glioma region, including edema, necrotic or non-enhancing tumor core, and enhancing tumor. & The scan appearance demonstrates one or more BraTS tumor subregions. \\
An example of axial T2-FLAIR brain MRI slice showing absence of any annotated glioma region, including edema, necrotic or non-enhancing tumor core, and enhancing tumor. & An example of axial T2-FLAIR brain MRI slice showing a focal glioma-related abnormality. \\
This image is interpreted as showing absence of any annotated glioma region, including edema, necrotic or non-enhancing tumor core, and enhancing tumor on axial T2-FLAIR brain MRI slice. & This image is interpreted as showing a glioma lesion comprising edema, core, or enhancing tumor on axial T2-FLAIR brain MRI slice. \\
\end{longtable}
\endgroup

\begingroup
\footnotesize
\setlength{\tabcolsep}{5pt}
\renewcommand{\arraystretch}{1.15}
\begin{longtable}{@{}>{\raggedright\arraybackslash}p{0.485\linewidth}>{\raggedright\arraybackslash}p{0.485\linewidth}@{}}
\caption{Exact prompts for LiverCT.}\label{tab:clinical-prompts-liver-ad}\\
\toprule
\textbf{Normal} & \textbf{Abnormal} \\
\midrule
\endfirsthead
\multicolumn{2}{c}{\tablename\ \thetable{} -- continued} \\
\toprule
\textbf{Normal} & \textbf{Abnormal} \\
\midrule
\endhead
\midrule
\multicolumn{2}{r}{Continued on next page} \\
\endfoot
\bottomrule
\endlastfoot
An axial liver CT slice of the liver showing absence of any annotated primary or metastatic liver tumor on the slice. & An axial liver CT slice of the liver showing an annotated liver tumor. \\
A clinical axial liver CT slice of the liver showing absence of any annotated primary or metastatic liver tumor on the slice. & A clinical axial liver CT slice of the liver showing a neoplastic liver lesion. \\
A diagnostic axial liver CT slice of the liver demonstrating absence of any annotated primary or metastatic liver tumor on the slice. & A diagnostic axial liver CT slice of the liver demonstrating a primary liver tumor. \\
This axial liver CT slice of the liver shows absence of any annotated primary or metastatic liver tumor on the slice. & This axial liver CT slice of the liver shows hepatocellular carcinoma. \\
This axial liver CT slice of the liver demonstrates absence of any annotated primary or metastatic liver tumor on the slice. & This axial liver CT slice of the liver demonstrates cholangiocarcinoma. \\
The liver on this axial liver CT slice shows absence of any annotated primary or metastatic liver tumor on the slice. & The liver on this axial liver CT slice shows a metastatic liver lesion. \\
The appearance of the liver on this axial liver CT slice is consistent with absence of any annotated primary or metastatic liver tumor on the slice. & The appearance of the liver on this axial liver CT slice is consistent with a colorectal cancer metastasis in the liver. \\
The primary finding on this axial liver CT slice is absence of any annotated primary or metastatic liver tumor on the slice. & The primary finding on this axial liver CT slice is a breast cancer metastasis in the liver. \\
The evaluated region of the liver on this axial liver CT slice shows absence of any annotated primary or metastatic liver tumor on the slice. & The evaluated region of the liver on this axial liver CT slice shows a lung cancer metastasis in the liver. \\
Image-level assessment of this axial liver CT slice indicates absence of any annotated primary or metastatic liver tumor on the slice. & Image-level assessment of this axial liver CT slice indicates a hypodense liver tumor. \\
The observable pattern on this axial liver CT slice is consistent with absence of any annotated primary or metastatic liver tumor on the slice. & The observable pattern on this axial liver CT slice is consistent with a hyperdense liver tumor. \\
The scan appearance demonstrates absence of any annotated primary or metastatic liver tumor on the slice. & The scan appearance demonstrates a small focal liver tumor. \\
An example of axial liver CT slice showing absence of any annotated primary or metastatic liver tumor on the slice. & An example of axial liver CT slice showing multiple liver tumors. \\
This image is interpreted as showing absence of any annotated primary or metastatic liver tumor on the slice on axial liver CT slice. & This image is interpreted as showing one or more liver tumors with variable size and appearance on axial liver CT slice. \\
\end{longtable}
\endgroup

\begingroup
\footnotesize
\setlength{\tabcolsep}{5pt}
\renewcommand{\arraystretch}{1.15}
\begin{longtable}{@{}>{\raggedright\arraybackslash}p{0.485\linewidth}>{\raggedright\arraybackslash}p{0.485\linewidth}@{}}
\caption{Exact prompts for RESC.}\label{tab:clinical-prompts-retina-resc-ad}\\
\toprule
\textbf{Normal} & \textbf{Abnormal} \\
\midrule
\endfirsthead
\multicolumn{2}{c}{\tablename\ \thetable{} -- continued} \\
\toprule
\textbf{Normal} & \textbf{Abnormal} \\
\midrule
\endhead
\midrule
\multicolumn{2}{r}{Continued on next page} \\
\endfoot
\bottomrule
\endlastfoot
A retinal OCT B-scan of the retina showing absence of any annotated retinal-edema lesion, including subretinal fluid and pigment epithelial detachment. & A retinal OCT B-scan of the retina showing an annotated retinal-edema lesion. \\
A clinical retinal OCT B-scan of the retina showing absence of any annotated retinal-edema lesion, including subretinal fluid and pigment epithelial detachment. & A clinical retinal OCT B-scan of the retina showing retinal edema. \\
A diagnostic retinal OCT B-scan of the retina demonstrating absence of any annotated retinal-edema lesion, including subretinal fluid and pigment epithelial detachment. & A diagnostic retinal OCT B-scan of the retina demonstrating an edematous retinal region. \\
This retinal OCT B-scan of the retina shows absence of any annotated retinal-edema lesion, including subretinal fluid and pigment epithelial detachment. & This retinal OCT B-scan of the retina shows edema-associated retinal thickening. \\
This retinal OCT B-scan of the retina demonstrates absence of any annotated retinal-edema lesion, including subretinal fluid and pigment epithelial detachment. & This retinal OCT B-scan of the retina demonstrates subretinal fluid. \\
The retina on this retinal OCT B-scan shows absence of any annotated retinal-edema lesion, including subretinal fluid and pigment epithelial detachment. & The retina on this retinal OCT B-scan shows a localized pocket of subretinal fluid. \\
The appearance of the retina on this retinal OCT B-scan is consistent with absence of any annotated retinal-edema lesion, including subretinal fluid and pigment epithelial detachment. & The appearance of the retina on this retinal OCT B-scan is consistent with fluid accumulation beneath the retina. \\
The primary finding on this retinal OCT B-scan is absence of any annotated retinal-edema lesion, including subretinal fluid and pigment epithelial detachment. & The primary finding on this retinal OCT B-scan is pigment epithelial detachment. \\
The evaluated region of the retina on this retinal OCT B-scan shows absence of any annotated retinal-edema lesion, including subretinal fluid and pigment epithelial detachment. & The evaluated region of the retina on this retinal OCT B-scan shows a localized pigment epithelial detachment. \\
Image-level assessment of this retinal OCT B-scan indicates absence of any annotated retinal-edema lesion, including subretinal fluid and pigment epithelial detachment. & Image-level assessment of this retinal OCT B-scan indicates elevation of the retinal pigment epithelium. \\
The observable pattern on this retinal OCT B-scan is consistent with absence of any annotated retinal-edema lesion, including subretinal fluid and pigment epithelial detachment. & The observable pattern on this retinal OCT B-scan is consistent with subretinal fluid with pigment epithelial detachment. \\
The scan appearance demonstrates absence of any annotated retinal-edema lesion, including subretinal fluid and pigment epithelial detachment. & The scan appearance demonstrates one or more edema-related lesions. \\
An example of retinal OCT B-scan showing absence of any annotated retinal-edema lesion, including subretinal fluid and pigment epithelial detachment. & An example of retinal OCT B-scan showing retinal-layer distortion caused by edema. \\
This image is interpreted as showing absence of any annotated retinal-edema lesion, including subretinal fluid and pigment epithelial detachment on retinal OCT B-scan. & This image is interpreted as showing a fluid-related abnormality disrupting normal retinal structure on retinal OCT B-scan. \\
\end{longtable}
\endgroup

\begingroup
\footnotesize
\setlength{\tabcolsep}{5pt}
\renewcommand{\arraystretch}{1.15}
\begin{longtable}{@{}>{\raggedright\arraybackslash}p{0.485\linewidth}>{\raggedright\arraybackslash}p{0.485\linewidth}@{}}
\caption{Exact prompts for OCT2017.}\label{tab:clinical-prompts-retina-oct2017-ad}\\
\toprule
\textbf{Normal} & \textbf{Abnormal} \\
\midrule
\endfirsthead
\multicolumn{2}{c}{\tablename\ \thetable{} -- continued} \\
\toprule
\textbf{Normal} & \textbf{Abnormal} \\
\midrule
\endhead
\midrule
\multicolumn{2}{r}{Continued on next page} \\
\endfoot
\bottomrule
\endlastfoot
A retinal OCT B-scan of the retina showing a disease-free retinal appearance with no choroidal neovascularization, diabetic macular edema, or drusen. & A retinal OCT B-scan of the retina showing choroidal neovascularization. \\
A clinical retinal OCT B-scan of the retina showing a disease-free retinal appearance with no choroidal neovascularization, diabetic macular edema, or drusen. & A clinical retinal OCT B-scan of the retina showing CNV-related neovascular change. \\
A diagnostic retinal OCT B-scan of the retina demonstrating a disease-free retinal appearance with no choroidal neovascularization, diabetic macular edema, or drusen. & A diagnostic retinal OCT B-scan of the retina demonstrating a choroidal neovascular membrane. \\
This retinal OCT B-scan of the retina shows a disease-free retinal appearance with no choroidal neovascularization, diabetic macular edema, or drusen. & This retinal OCT B-scan of the retina shows diabetic macular edema. \\
This retinal OCT B-scan of the retina demonstrates a disease-free retinal appearance with no choroidal neovascularization, diabetic macular edema, or drusen. & This retinal OCT B-scan of the retina demonstrates DME-related retinal thickening. \\
The retina on this retinal OCT B-scan shows a disease-free retinal appearance with no choroidal neovascularization, diabetic macular edema, or drusen. & The retina on this retinal OCT B-scan shows diabetic macular edema with intraretinal fluid. \\
The appearance of the retina on this retinal OCT B-scan is consistent with a disease-free retinal appearance with no choroidal neovascularization, diabetic macular edema, or drusen. & The appearance of the retina on this retinal OCT B-scan is consistent with drusen deposits. \\
The primary finding on this retinal OCT B-scan is a disease-free retinal appearance with no choroidal neovascularization, diabetic macular edema, or drusen. & The primary finding on this retinal OCT B-scan is multiple drusen beneath the retinal pigment epithelium. \\
The evaluated region of the retina on this retinal OCT B-scan shows a disease-free retinal appearance with no choroidal neovascularization, diabetic macular edema, or drusen. & The evaluated region of the retina on this retinal OCT B-scan shows an early age-related macular degeneration appearance with drusen. \\
Image-level assessment of this retinal OCT B-scan indicates a disease-free retinal appearance with no choroidal neovascularization, diabetic macular edema, or drusen. & Image-level assessment of this retinal OCT B-scan indicates a retinal disease pattern from CNV, DME, or drusen. \\
The observable pattern on this retinal OCT B-scan is consistent with a disease-free retinal appearance with no choroidal neovascularization, diabetic macular edema, or drusen. & The observable pattern on this retinal OCT B-scan is consistent with a diseased retinal OCT appearance. \\
The scan appearance demonstrates a disease-free retinal appearance with no choroidal neovascularization, diabetic macular edema, or drusen. & The scan appearance demonstrates retinal fluid or edema associated with DME. \\
An example of retinal OCT B-scan showing a disease-free retinal appearance with no choroidal neovascularization, diabetic macular edema, or drusen. & An example of retinal OCT B-scan showing abnormal retinal structure associated with CNV, DME, or drusen. \\
This image is interpreted as showing a disease-free retinal appearance with no choroidal neovascularization, diabetic macular edema, or drusen on retinal OCT B-scan. & This image is interpreted as showing one of the OCT2017 CNV, DME, or drusen disease classes on retinal OCT B-scan. \\
\end{longtable}
\endgroup

\begingroup
\footnotesize
\setlength{\tabcolsep}{5pt}
\renewcommand{\arraystretch}{1.15}
\begin{longtable}{@{}>{\raggedright\arraybackslash}p{0.485\linewidth}>{\raggedright\arraybackslash}p{0.485\linewidth}@{}}
\caption{Exact prompts for BUSI.}\label{tab:clinical-prompts-busi}\\
\toprule
\textbf{Normal} & \textbf{Abnormal} \\
\midrule
\endfirsthead
\multicolumn{2}{c}{\tablename\ \thetable{} -- continued} \\
\toprule
\textbf{Normal} & \textbf{Abnormal} \\
\midrule
\endhead
\midrule
\multicolumn{2}{r}{Continued on next page} \\
\endfoot
\bottomrule
\endlastfoot
A grayscale breast ultrasound image of breast tissue showing absence of any focal breast mass labeled benign or malignant. & A grayscale breast ultrasound image of breast tissue showing a focal breast mass. \\
A clinical grayscale breast ultrasound image of breast tissue showing absence of any focal breast mass labeled benign or malignant. & A clinical grayscale breast ultrasound image of breast tissue showing an annotated breast lesion. \\
A diagnostic grayscale breast ultrasound image of breast tissue demonstrating absence of any focal breast mass labeled benign or malignant. & A diagnostic grayscale breast ultrasound image of breast tissue demonstrating a localized breast nodule. \\
This grayscale breast ultrasound image of breast tissue shows absence of any focal breast mass labeled benign or malignant. & This grayscale breast ultrasound image of breast tissue shows a benign breast mass. \\
This grayscale breast ultrasound image of breast tissue demonstrates absence of any focal breast mass labeled benign or malignant. & This grayscale breast ultrasound image of breast tissue demonstrates a benign breast lesion. \\
Breast tissue on this grayscale breast ultrasound image shows absence of any focal breast mass labeled benign or malignant. & Breast tissue on this grayscale breast ultrasound image shows a benign breast nodule. \\
The appearance of breast tissue on this grayscale breast ultrasound image is consistent with absence of any focal breast mass labeled benign or malignant. & The appearance of breast tissue on this grayscale breast ultrasound image is consistent with a localized benign tumor region. \\
The primary finding on this grayscale breast ultrasound image is absence of any focal breast mass labeled benign or malignant. & The primary finding on this grayscale breast ultrasound image is a malignant breast mass. \\
The evaluated region of breast tissue on this grayscale breast ultrasound image shows absence of any focal breast mass labeled benign or malignant. & The evaluated region of breast tissue on this grayscale breast ultrasound image shows a malignant breast lesion. \\
Image-level assessment of this grayscale breast ultrasound image indicates absence of any focal breast mass labeled benign or malignant. & Image-level assessment of this grayscale breast ultrasound image indicates a malignant breast tumor. \\
The observable pattern on this grayscale breast ultrasound image is consistent with absence of any focal breast mass labeled benign or malignant. & The observable pattern on this grayscale breast ultrasound image is consistent with a localized malignant tumor region. \\
The scan appearance demonstrates absence of any focal breast mass labeled benign or malignant. & The scan appearance demonstrates a breast mass labeled benign or malignant. \\
An example of grayscale breast ultrasound image showing absence of any focal breast mass labeled benign or malignant. & An example of grayscale breast ultrasound image showing a tumor-bearing breast ultrasound appearance. \\
This image is interpreted as showing absence of any focal breast mass labeled benign or malignant on grayscale breast ultrasound image. & This image is interpreted as showing one or more focal breast lesions on grayscale breast ultrasound image. \\
\end{longtable}
\endgroup

\begingroup
\footnotesize
\setlength{\tabcolsep}{5pt}
\renewcommand{\arraystretch}{1.15}
\begin{longtable}{@{}>{\raggedright\arraybackslash}p{0.485\linewidth}>{\raggedright\arraybackslash}p{0.485\linewidth}@{}}
\caption{Exact prompts for CheXpert.}\label{tab:clinical-prompts-chexpert}\\
\toprule
\textbf{Normal} & \textbf{Abnormal} \\
\midrule
\endfirsthead
\multicolumn{2}{c}{\tablename\ \thetable{} -- continued} \\
\toprule
\textbf{Normal} & \textbf{Abnormal} \\
\midrule
\endhead
\midrule
\multicolumn{2}{r}{Continued on next page} \\
\endfoot
\bottomrule
\endlastfoot
A chest radiograph of the chest showing the No Finding state with none of the twelve evaluated disease observations. & A chest radiograph of the chest showing enlarged cardiomediastinum. \\
A clinical chest radiograph of the chest showing the No Finding state with none of the twelve evaluated disease observations. & A clinical chest radiograph of the chest showing cardiomegaly. \\
A diagnostic chest radiograph of the chest demonstrating the No Finding state with none of the twelve evaluated disease observations. & A diagnostic chest radiograph of the chest demonstrating lung opacity. \\
This chest radiograph of the chest shows the No Finding state with none of the twelve evaluated disease observations. & This chest radiograph of the chest shows a lung lesion. \\
This chest radiograph of the chest demonstrates the No Finding state with none of the twelve evaluated disease observations. & This chest radiograph of the chest demonstrates pulmonary edema. \\
The chest on this chest radiograph shows the No Finding state with none of the twelve evaluated disease observations. & The chest on this chest radiograph shows pulmonary consolidation. \\
The appearance of the chest on this chest radiograph is consistent with the No Finding state with none of the twelve evaluated disease observations. & The appearance of the chest on this chest radiograph is consistent with pneumonia. \\
The primary finding on this chest radiograph is the No Finding state with none of the twelve evaluated disease observations. & The primary finding on this chest radiograph is atelectasis. \\
The evaluated region of the chest on this chest radiograph shows the No Finding state with none of the twelve evaluated disease observations. & The evaluated region of the chest on this chest radiograph shows pneumothorax. \\
Image-level assessment of this chest radiograph indicates the No Finding state with none of the twelve evaluated disease observations. & Image-level assessment of this chest radiograph indicates pleural effusion. \\
The observable pattern on this chest radiograph is consistent with the No Finding state with none of the twelve evaluated disease observations. & The observable pattern on this chest radiograph is consistent with another pleural abnormality. \\
The scan appearance demonstrates the No Finding state with none of the twelve evaluated disease observations. & The scan appearance demonstrates a fracture. \\
An example of chest radiograph showing the No Finding state with none of the twelve evaluated disease observations. & An example of chest radiograph showing one or more of the twelve evaluated disease observations. \\
This image is interpreted as showing the No Finding state with none of the twelve evaluated disease observations on chest radiograph. & This image is interpreted as showing an abnormal chest radiographic finding from the evaluated disease set on chest radiograph. \\
\end{longtable}
\endgroup

\begingroup
\footnotesize
\setlength{\tabcolsep}{5pt}
\renewcommand{\arraystretch}{1.15}
\begin{longtable}{@{}>{\raggedright\arraybackslash}p{0.485\linewidth}>{\raggedright\arraybackslash}p{0.485\linewidth}@{}}
\caption{Exact prompts for Kaggle BrainMRI.}\label{tab:clinical-prompts-brainmri}\\
\toprule
\textbf{Normal} & \textbf{Abnormal} \\
\midrule
\endfirsthead
\multicolumn{2}{c}{\tablename\ \thetable{} -- continued} \\
\toprule
\textbf{Normal} & \textbf{Abnormal} \\
\midrule
\endhead
\midrule
\multicolumn{2}{r}{Continued on next page} \\
\endfoot
\bottomrule
\endlastfoot
A 2D brain MRI image of the brain showing absence of any visible brain tumor or intracranial mass. & A 2D brain MRI image of the brain showing a visible brain tumor. \\
A clinical 2D brain MRI image of the brain showing absence of any visible brain tumor or intracranial mass. & A clinical 2D brain MRI image of the brain showing an intracranial tumor. \\
A diagnostic 2D brain MRI image of the brain demonstrating absence of any visible brain tumor or intracranial mass. & A diagnostic 2D brain MRI image of the brain demonstrating a tumor-bearing brain MRI appearance. \\
This 2D brain MRI image of the brain shows absence of any visible brain tumor or intracranial mass. & This 2D brain MRI image of the brain shows a focal intracranial mass. \\
This 2D brain MRI image of the brain demonstrates absence of any visible brain tumor or intracranial mass. & This 2D brain MRI image of the brain demonstrates a neoplastic brain lesion. \\
The brain on this 2D brain MRI image shows absence of any visible brain tumor or intracranial mass. & The brain on this 2D brain MRI image shows a space-occupying brain lesion. \\
The appearance of the brain on this 2D brain MRI image is consistent with absence of any visible brain tumor or intracranial mass. & The appearance of the brain on this 2D brain MRI image is consistent with a localized brain tumor region. \\
The primary finding on this 2D brain MRI image is absence of any visible brain tumor or intracranial mass. & The primary finding on this 2D brain MRI image is abnormal brain tissue consistent with a tumor. \\
The evaluated region of the brain on this 2D brain MRI image shows absence of any visible brain tumor or intracranial mass. & The evaluated region of the brain on this 2D brain MRI image shows a tumor-positive brain MRI. \\
Image-level assessment of this 2D brain MRI image indicates absence of any visible brain tumor or intracranial mass. & Image-level assessment of this 2D brain MRI image indicates an intracranial mass lesion. \\
The observable pattern on this 2D brain MRI image is consistent with absence of any visible brain tumor or intracranial mass. & The observable pattern on this 2D brain MRI image is consistent with one or more visible tumor regions. \\
The scan appearance demonstrates absence of any visible brain tumor or intracranial mass. & The scan appearance demonstrates MRI evidence of a brain tumor. \\
An example of 2D brain MRI image showing absence of any visible brain tumor or intracranial mass. & An example of 2D brain MRI image showing a focal tumor within the brain. \\
This image is interpreted as showing absence of any visible brain tumor or intracranial mass on 2D brain MRI image. & This image is interpreted as showing a brain MRI appearance containing a tumor on 2D brain MRI image. \\
\end{longtable}
\endgroup

\noindent The donor condition uses the donor's complete normal and abnormal banks unchanged; no
sentence is rewritten for the target.

\subsection{Industrial prompt banks}
\label{app:industrial_prompts}

MVTec-AD and VisA use the fixed two-level WinCLIP-style prompt ensemble
\citep{jeong2023winclip}. For category noun $o_t$, each of seven normal and seven abnormal state
phrases is inserted into each of 22 photo templates, giving $7\times22=154$ sentences per state and
category. The frozen OpenAI CLIP ViT-L/14 text encoder \citep{radford2021clip} encodes the sentences
individually. We average the 154 normal embeddings and the 154 abnormal embeddings separately, then
unit-normalize the two means to obtain $(\va_t^N,\va_t^A)$. No prompt token or anchor is trained.
Tables~\ref{tab:industrial-state-phrases}--\ref{tab:industrial-category-nouns} give the complete
construction, including literal punctuation and object nouns; together they
determine every sentence in every industrial bank exactly. Expansion is state-major and
template-minor in the displayed orders.

\begin{table}[!htbp]
\centering
\caption{Exact state phrases used to construct every industrial prompt bank. The object placeholder
\texttt{\{\}} is replaced by the noun in Table~\ref{tab:industrial-category-nouns}. The leading space
in the sixth abnormal phrase is retained from the reference implementation.}
\label{tab:industrial-state-phrases}
\begingroup
\footnotesize
\setlength{\tabcolsep}{5pt}
\renewcommand{\arraystretch}{1.15}
\begin{tabular}{@{}ll@{}}
\toprule
Normal state phrase & Abnormal state phrase \\
\midrule
\texttt{\{\}} & \texttt{damaged \{\}} \\
\texttt{flawless \{\}} & \texttt{broken \{\}} \\
\texttt{perfect \{\}} & \texttt{\{\} with flaw} \\
\texttt{unblemished \{\}} & \texttt{\{\} with defect} \\
\texttt{\{\} without flaw} & \texttt{\{\} with damage} \\
\texttt{\{\} without defect} & \texttt{\char32 disease \{\}} \\
\texttt{\{\} without damage} & \texttt{abnormal \{\}} \\
\bottomrule
\end{tabular}
\endgroup
\end{table}

\begin{table}[!htbp]
\centering
\caption{Exact photo templates used for the industrial prompt ensemble. The state placeholder
\texttt{\{\}} is replaced by one of the state phrases in
Table~\ref{tab:industrial-state-phrases}.}
\label{tab:industrial-photo-templates}
\begingroup
\footnotesize
\setlength{\tabcolsep}{5pt}
\renewcommand{\arraystretch}{1.15}
\begin{tabular}{@{}rl@{}}
\toprule
& Template \\
\midrule
1 & \texttt{a cropped photo of the \{\}.} \\
2 & \texttt{a cropped photo of a \{\}.} \\
3 & \texttt{a close-up photo of a \{\}.} \\
4 & \texttt{a close-up photo of the \{\}.} \\
5 & \texttt{a bright photo of a \{\}.} \\
6 & \texttt{a bright photo of the \{\}.} \\
7 & \texttt{a dark photo of the \{\}.} \\
8 & \texttt{a dark photo of a \{\}.} \\
9 & \texttt{a jpeg corrupted photo of a \{\}.} \\
10 & \texttt{a jpeg corrupted photo of the \{\}.} \\
11 & \texttt{a blurry photo of the \{\}.} \\
12 & \texttt{a blurry photo of a \{\}.} \\
13 & \texttt{a photo of the \{\}} \\
14 & \texttt{a photo of a \{\}} \\
15 & \texttt{a photo of a small \{\}} \\
16 & \texttt{a photo of the small \{\}} \\
17 & \texttt{a photo of a large \{\}} \\
18 & \texttt{a photo of the large \{\}} \\
19 & \texttt{a photo of the \{\} for visual inspection.} \\
20 & \texttt{a photo of a \{\} for visual inspection.} \\
21 & \texttt{a photo of the \{\} for anomaly detection.} \\
22 & \texttt{a photo of a \{\} for anomaly detection.} \\
\bottomrule
\end{tabular}
\endgroup
\end{table}

\begin{table}[!htbp]
\centering
\caption{Exact object nouns substituted into the industrial prompt templates. Numbered VisA
categories deliberately retain their index.}
\label{tab:industrial-category-nouns}
\begingroup
\footnotesize
\setlength{\tabcolsep}{5pt}
\renewcommand{\arraystretch}{1.15}
\begin{tabular}{@{}>{\bfseries}l>{\raggedright\arraybackslash}p{0.80\textwidth}@{}}
\toprule
Dataset & Object nouns \\
\midrule
MVTec-AD & bottle; cable; capsule; carpet; grid; hazelnut; leather; metal nut; pill; screw; tile;
toothbrush; transistor; wood; zipper \\
\midrule
VisA & candle; capsules; cashew; chewing gum; fryum; macaroni 1; macaroni 2; printed circuit board 1;
printed circuit board 2; printed circuit board 3; printed circuit board 4; pipe fryum \\
\bottomrule
\end{tabular}
\endgroup
\end{table}

\subsection{Anchor cosines and effective scales}
\label{app:anchor_geometry_values}

For every correct anchor pair, Tables~\ref{tab:medical_anchor_scales}
and~\ref{tab:industrial_anchor_scales} report
\begin{equation}
    \rho=\inner{\va_N}{\va_A},
    \qquad
    c=\sqrt{1-\rho},
    \qquad
    \lambda_{\mathrm{eff}}^{\mathrm{train}}=100\sqrt{2-2\rho},
    \qquad
    \lambda_{\mathrm{eff}}^{\mathrm{eval}}=20\sqrt{2-2\rho}.
    \label{eq:reported_anchor_scales}
\end{equation}
Matched orthogonal anchors multiply the training and evaluation similarity scales by $c$;
because their anchor difference has norm $\sqrt{2}$, this gives exactly the two reported
$\lambda_{\mathrm{eff}}$ values. Unmatched orthogonal anchors instead have
$\lambda_{\mathrm{eff}}^{\mathrm{train}}=100\sqrt{2}=141.42$ and
$\lambda_{\mathrm{eff}}^{\mathrm{eval}}=20\sqrt{2}=28.28$ for every dataset or category.
They therefore inflate the correct-prompt effective scale by the directly readable factor $1/c$.
All entries are computed from the unrounded frozen OpenAI CLIP ViT-L/14 anchors; only their display
is rounded.

\begin{table}[!htbp]
\centering
\caption{Correct-prompt geometry for the seven medical datasets. The reported values determine both
the Gram-matched random construction and the rescaling of matched orthogonal anchors.}
\label{tab:medical_anchor_scales}
\begingroup
\footnotesize
\setlength{\tabcolsep}{5pt}
\renewcommand{\arraystretch}{1.15}
\begin{tabular}{@{}l*{7}{r}@{}}
\toprule
Quantity & \makecell{BraTS\\2021} & \makecell{Liver\\CT} & RESC & \makecell{OCT\\2017} & BUSI & CheXpert & \makecell{Kaggle\\BrainMRI} \\
\midrule
$\rho_{\mathrm{correct}}$ & $0.9255$ & $0.9711$ & $0.9524$ & $0.9408$ & $0.9771$ & $0.9241$ & $0.9742$ \\
$c$ & $0.2730$ & $0.1700$ & $0.2183$ & $0.2432$ & $0.1513$ & $0.2755$ & $0.1608$ \\
$\lambda_{\mathrm{eff}}^{\mathrm{train}}$ & $38.61$ & $24.04$ & $30.87$ & $34.40$ & $21.40$ & $38.97$ & $22.74$ \\
$\lambda_{\mathrm{eff}}^{\mathrm{eval}}$ & $7.72$ & $4.81$ & $6.17$ & $6.88$ & $4.28$ & $7.79$ & $4.55$ \\
\bottomrule
\end{tabular}
\endgroup
\end{table}

\begin{table}[!htbp]
\centering
\caption{Per-category correct-prompt geometry for the unified industrial detector. Each category's
shared- and independent-rotation controls preserve its displayed $\rho_t$, $c_t$, and
$\lambda_{\mathrm{eff},t}$; independent rotations change only cross-category off-diagonal geometry.
The participation-ratio effective rank of the correct-prompt difference Gram matrix $\mM$
(Section~\ref{sec:industrial_extension}) is $r_{\mathrm{eff}}=1.79$ for MVTec-AD and $1.78$ for VisA. Shared
rotations preserve $\mM$ and hence $r_{\mathrm{eff}}$; independent rotations raise it to $14.28$--$14.29$ of 15
(MVTec-AD) and about $11.10$ of 12 (VisA) across the three geometry seeds.}
\label{tab:industrial_anchor_scales}
\begingroup
\footnotesize
\setlength{\tabcolsep}{5pt}
\renewcommand{\arraystretch}{1.15}
\begin{tabular}{@{}lrrrr@{\hspace{1.6em}}lrrrr@{}}
\toprule
\multicolumn{5}{c}{\textbf{MVTec-AD}} & \multicolumn{5}{c}{\textbf{VisA}} \\
\cmidrule(lr){1-5}\cmidrule(lr){6-10}
Category & $\rho_t$ & $c_t$ & $\lambda_{\mathrm{eff},t}^{\mathrm{train}}$ & $\lambda_{\mathrm{eff},t}^{\mathrm{eval}}$ & Category & $\rho_t$ & $c_t$ & $\lambda_{\mathrm{eff},t}^{\mathrm{train}}$ & $\lambda_{\mathrm{eff},t}^{\mathrm{eval}}$ \\
\midrule
bottle & $0.9754$ & $0.1567$ & $22.17$ & $4.43$ & candle & $0.9743$ & $0.1604$ & $22.68$ & $4.54$ \\
cable & $0.9771$ & $0.1514$ & $21.41$ & $4.28$ & capsules & $0.9742$ & $0.1605$ & $22.70$ & $4.54$ \\
capsule & $0.9715$ & $0.1688$ & $23.87$ & $4.77$ & cashew & $0.9771$ & $0.1514$ & $21.41$ & $4.28$ \\
carpet & $0.9767$ & $0.1526$ & $21.58$ & $4.32$ & chewing gum & $0.9801$ & $0.1410$ & $19.95$ & $3.99$ \\
grid & $0.9735$ & $0.1629$ & $23.03$ & $4.61$ & fryum & $0.9701$ & $0.1730$ & $24.47$ & $4.89$ \\
hazelnut & $0.9696$ & $0.1744$ & $24.67$ & $4.93$ & macaroni 1 & $0.9738$ & $0.1617$ & $22.87$ & $4.57$ \\
leather & $0.9711$ & $0.1700$ & $24.05$ & $4.81$ & macaroni 2 & $0.9732$ & $0.1638$ & $23.16$ & $4.63$ \\
metal nut & $0.9830$ & $0.1305$ & $18.45$ & $3.69$ & printed circuit board 1 & $0.9854$ & $0.1209$ & $17.09$ & $3.42$ \\
pill & $0.9779$ & $0.1485$ & $21.00$ & $4.20$ & printed circuit board 2 & $0.9857$ & $0.1194$ & $16.89$ & $3.38$ \\
screw & $0.9834$ & $0.1289$ & $18.22$ & $3.64$ & printed circuit board 3 & $0.9850$ & $0.1224$ & $17.31$ & $3.46$ \\
tile & $0.9699$ & $0.1735$ & $24.54$ & $4.91$ & printed circuit board 4 & $0.9868$ & $0.1150$ & $16.27$ & $3.25$ \\
toothbrush & $0.9817$ & $0.1354$ & $19.15$ & $3.83$ & pipe fryum & $0.9744$ & $0.1600$ & $22.63$ & $4.53$ \\
transistor & $0.9759$ & $0.1551$ & $21.94$ & $4.39$ & & & & & \\
wood & $0.9730$ & $0.1645$ & $23.26$ & $4.65$ & & & & & \\
zipper & $0.9802$ & $0.1407$ & $19.90$ & $3.98$ & & & & & \\
\bottomrule
\end{tabular}
\endgroup
\end{table}

\subsection{Adapter coverage and rotation closure}
\label{app:closure}

Table~\ref{tab:family} checks the exact scope of Theorem~\ref{thm:gram} against representative
systems. A ``no'' does not say that the method uses no related symmetry; it says that its published
parameterization does not satisfy all assumptions of our exact result. Coverage can fail on either
of two sides, and the table names which one applies in each case. On the \emph{anchor} side the pair
must be prescribed and held fixed within the run: if the anchors are produced from learned tokens by
a frozen text encoder, the reachable set of pairs is the image of the token space under that
encoder, so an arbitrarily rotated pair need not be reachable, and $\rho$---hence
$\lambda_{\mathrm{eff}}$---is fitted rather than prescribed. On the \emph{visual} side the admissible
set $\cS$ of maps into text-embedding space must be closed under left multiplication by $O(d)$
(Lemma~\ref{lem:closure}); failure is caused by a specific constraint that prevents
$\mQ\mL$ from remaining admissible, not by a family label alone.

\begin{table}[htbp]
\centering
\caption{Scope check for representative VLM anomaly detectors. Exact coverage requires prescribed
text embeddings held fixed within a run and a visual scoring path whose admissible set $\cS$ is
closed under left multiplication by $O(d)$. The last column names the specific constraint that
blocks---or admits---$\mQ\mL\in\cS$. A ``no'' is not a judgment about the method, and where a
published parameterization admits more than one reading we state the condition under which the
verdict holds.}
\label{tab:family}
\begingroup
\footnotesize
\setlength{\tabcolsep}{5pt}
\renewcommand{\arraystretch}{1.15}
\begin{adjustbox}{max width=\textwidth,center}
\begin{tabular}{@{}>{\raggedright\arraybackslash}p{0.12\textwidth}
                  >{\raggedright\arraybackslash}p{0.13\textwidth}
                  >{\raggedright\arraybackslash}p{0.15\textwidth}
                  c
                  >{\raggedright\arraybackslash}p{0.34\textwidth}@{}}
\toprule
Method & Text embeddings & Visual path used for scoring & Cov. &
Constraint that blocks or admits $\mQ\mL\in\cS$ \\
\midrule
APRIL-GAN \citep{chen2023aprilgan} & prescribed prompts & free linear projections & yes$^{\dagger}$ &
None on either side: the per-layer map into text-embedding space is a free linear map, which absorbs
$\mQ$. $^{\dagger}$The verdict covers the trained linear-projection branch; the few-shot mode adds a
text-free reference-memory branch that our statement does not describe. \\
\addlinespace[1pt]
MVFA \citep{huang2024mvfa} & prescribed prompt ensemble & residual multi-level adapters &
no$^{\ddagger}$ &
Visual side: the scored feature retains a fixed, untrained component, so absorbing $\mQ$ requires
the trainable branch to cancel that component's rotation in addition to supplying its own output.
$^{\ddagger}$This holds when the trainable branch is capacity-constrained; a freely parameterized
full-rank residual branch would instead be closed, since $\{\mI+\mA:\mA\ \text{free}\}$ is all of
$\R^{d\times d}$. \\
\addlinespace[1pt]
MediCLIP \citep{zhang2024mediclip} & learned prompts & visual adapters & no &
Anchor side: prompt-token bottleneck. The pair is not prescribed, so $\rho$ and $\lambda_{\mathrm{eff}}$
are fitted, and the rotated pair need not lie in the image of the token space under the frozen text
encoder. \\
\addlinespace[1pt]
MadCLIP \citep{shiri2025madclip} & learned prompts & dual residual adapters & no &
Both sides: the same prompt-token bottleneck as above, together with a fixed untrained component in
the residual visual path. \\
\addlinespace[1pt]
\textbf{Ours (few-shot medical)} & prescribed & free output map & yes &
None: free final linear map, so Lemma~\ref{lem:closure} applies directly. Table~\ref{tab:closure}
records the verdict for each parameterization we train. \\
\addlinespace[1pt]
\textbf{Ours (full-shot unified)} & prescribed & one free shared adapter & yes &
None on either side; the multi-category statement additionally requires the replacement bank to
match $\mM=\mD\mD^\top$ (Corollary~\ref{cor:multicategory}). \\
\midrule
WinCLIP \citep{jeong2023winclip} & prescribed prompts & no trained projection & no &
Neither side is at issue: there is no $\mL$ to rotate, so the premise of an adapter trained against
the anchors fails before closure arises. \\
\bottomrule
\end{tabular}
\end{adjustbox}
\endgroup
\end{table}

\begin{definition}[Rotation-closed adapter family]
\label{def:closed}
A family $\cG$ of maps into $\R^d$ is rotation-closed if $g\in\cG$ implies
$\mQ\circ g\in\cG$ for every $\mQ\in O(d)$.
\end{definition}

\begin{lemma}[A free final map implies rotation closure]
\label{lem:closure}
Suppose every $g\in\cG$ factors as $g=\mL\circ h$, where $\mL$ ranges over a set
$\cS\subseteq\R^{d\times m}$ closed under left multiplication by $O(d)$, independently of the
admissible $h$. Then $\cG$ is rotation-closed.
\end{lemma}

\begin{proof}
For any $\mQ\in O(d)$, $\mQ\circ g=(\mQ\mL)\circ h$, and $\mQ\mL\in\cS$ by assumption.
\end{proof}

\begin{table}[htbp]
\centering
\caption{Rotation closure of the three adapter families.}
\label{tab:closure}
\begingroup
\footnotesize
\setlength{\tabcolsep}{5pt}
\renewcommand{\arraystretch}{1.15}
\begin{tabular}{@{}llcl@{}}
\toprule
Family & Parameterization & Closed & Reason \\
\midrule
linear & $\vy=\widetilde\mW^{\top}\tilde\vz$, $\widetilde\mW$ free & yes & absorb $\mQ$ into $\widetilde\mW$ \\
MLP & $\vy=\mW_2\,\mathrm{GELU}(\mW_1\vz)$ & yes & absorb $\mQ$ into $\mW_2$ \\
Transformer & $\vy=\mU_{\mathrm{out}}T(\mU_{\mathrm{in}}Z+E_{\mathrm{pos}})$ & yes & absorb $\mQ$ into $\mU_{\mathrm{out}}$ \\
\bottomrule
\end{tabular}
\endgroup
\end{table}

\ptitle{Linear.} Let $\tilde\vz=[\vz;1]$ and
$g(\vz)=\widetilde\mW^\top\tilde\vz$, so $\widetilde\mW$ contains both weights and bias.
Then $\mQ g(\vz)=(\widetilde\mW\mQ^\top)^\top\tilde\vz$, and
$\widetilde\mW\mQ^\top$ is another unconstrained affine map.

\ptitle{MLP.} For
$g(\vz)=\mW_2\mathrm{GELU}(\mW_1\vz)$, replace the free final matrix $\mW_2$ by
$\mQ\mW_2$.

\ptitle{Transformer.} For
$g(Z)=\mU_{\mathrm{out}}T(\mU_{\mathrm{in}}Z+E_{\mathrm{pos}})$, replace
$\mU_{\mathrm{out}}$ by $\mQ\mU_{\mathrm{out}}$. Patch interaction inside $T$ does not affect
closure.

\subsection{Finite-training residual}
\label{app:residual}

Theorem~\ref{thm:gram} and Corollary~\ref{cor:multicategory} equate attainable probability maps
and objective-value sets, not arbitrary optimization paths. Under the reparameterization
$\mW\mapsto\mW\mQ^\top$, an empirical difference between exactly matched conditions can arise
from components of the finite training or evaluation procedure that are not equivariant to this
rotation, or from stochastic elements that are not fully coupled across the compared runs.

\ptitle{Weight decay.} Frobenius weight decay is invariant:
\begin{equation}
\normF{\mW\mQ^\top}^2
=\operatorname{tr}(\mQ\mW^\top\mW\mQ^\top)
=\normF{\mW}^2.
\end{equation}

\ptitle{Initialization.} Equivariance requires the distribution of $\mW^{(0)}$ to be invariant
under right multiplication by $O(d)$. A coordinate-aligned uniform initialization is not exactly
invariant.

\ptitle{Optimizer.} Gradient descent is equivariant because its gradient rotates with $\mQ$.
Adam is not: its diagonal second-moment preconditioner is tied to the coordinate frame. Initialization
and Adam are therefore the principal algorithmic sources of a finite-training discrepancy under
exact arithmetic and fully coupled randomness. The reported experiments also use bfloat16 arithmetic,
and seed identifiers are not common to every reported condition
(Appendix~\ref{app:training_protocol}); finite-precision effects and uncoupled stochastic variation
can therefore contribute as well. Observed residuals are not attributed uniquely to initialization
or Adam.

\subsection{Proof of Proposition~\ref{prop:breaking}}
\label{app:breaking}

Theorem~\ref{thm:gram} constructs its correspondence by rotating every layer-wise adapter. If
the family is not rotation-closed, the transformed adapters need not be admissible. If the
adapter is frozen rather than trained against the replacement anchors, the compensating
parameterization is unavailable. In the category-conditioned case, Corollary~\ref{cor:multicategory} requires one
rotation for all $T$ difference vectors; changing $\mD\mD^\top$ removes that construction.
These are failures of a sufficient equivalence argument, not sufficient conditions for a
performance difference. In particular, none establishes that linguistic meaning is beneficial.

\section{Internal Training, Checkpoint, and Evaluation Protocol}
\label{app:training_protocol}

This section consolidates the settings shared by our internal experiments;
Table~\ref{tab:training_protocol} lists them. The DINOv3-L/16,
DINOv2-L/14, and CLIP ViT-L/14 vision backbones
\citep{simeoni2025dinov3,oquab2024dinov2,radford2021clip}, as well as the OpenAI CLIP ViT-L/14 text
encoder, are frozen. Inputs are resized to \(224\times224\), and visual
features are retained from transformer blocks \(5,11,17,23\). Each retained layer has its own
adapter into the \(768\)-dimensional anchor space. The prompt context length is zero: prompt
ensembles are encoded once into prescribed anchors, and neither prompt tokens nor anchors are
optimized. Unless a table explicitly names another family, the adapter is affine linear. The MLP
and Transformer parameterizations are specified in
Appendix~\ref{app:closure}.

\begin{table}[htbp]
\centering
\caption{Internal training and evaluation settings. Settings not separated by column are shared.}
\label{tab:training_protocol}
\begingroup
\footnotesize
\setlength{\tabcolsep}{5pt}
\renewcommand{\arraystretch}{1.15}
\begin{adjustbox}{max width=\textwidth,center}
\begin{tabular}{@{}llll@{}}
\toprule
\textbf{Setting} & \textbf{Medical} & \textbf{MVTec-AD} & \textbf{VisA} \\
\midrule
Training images & regime-specific \(K\) support & all normal train images & all normal train images \\
Epochs & \(200\) & \(50\) & \(30\) \\
Training seeds & \multicolumn{3}{c}{three per condition} \\
Input size & \multicolumn{3}{c}{\(224\times224\)} \\
Retained blocks & \multicolumn{3}{c}{\(5,11,17,23\)} \\
Optimizer & \multicolumn{3}{c}{Adam, learning rate \(10^{-3}\), weight decay \(10^{-3}\)} \\
Adam parameters & \multicolumn{3}{c}{\(\beta=(0.9,0.95)\), \(\epsilon=10^{-8}\)} \\
Schedule & \multicolumn{3}{c}{linear decay after \(10\%\) linear warm-up} \\
Batch / accumulation & \multicolumn{3}{c}{\(4/2\)} \\
Gradient clipping & \multicolumn{3}{c}{global norm \(0.5\)} \\
Numerical precision & \multicolumn{3}{c}{bfloat16 automatic mixed precision} \\
Support augmentation & \(16\) deterministic views & none & none \\
Training similarity scale & \multicolumn{3}{c}{\(100c\)} \\
Evaluation similarity scale & \multicolumn{3}{c}{\(20c\)} \\
Reported checkpoint & \multicolumn{3}{c}{completed final epoch; no checkpoint selection} \\
Image score & \multicolumn{3}{c}{mean of the highest-scoring \(1\%\) of pixels} \\
AUROC aggregation & pooled within dataset & category-macro & category-macro \\
\bottomrule
\end{tabular}
\end{adjustbox}
\endgroup
\end{table}

Here \(c=1\) except for matched orthogonal anchors, where
\(c=\sqrt{1-\rho_{\mathrm{correct}}}\). Training minimizes the sum, over the four layer maps, of
focal loss with \(\gamma=2\) and Dice loss. Logits are bilinearly upsampled before the binary
softmax. Evaluation averages the four abnormal-probability maps before computing pixel AUROC and
the top-\(1\%\) image score. Industrial AUROCs are first computed separately for each object
category and then macro-averaged; medical AUROC is computed over the applicable dataset test split.

\ptitle{Medical supervision and evaluation masks.} In the real-anomaly-support regime, each support
set contains \(K\) normal and \(K\) real abnormal images. BraTS2021, LiverCT, RESC, and OCT2017 use
annotated dense ground-truth training masks, which enter the same pixel-level focal and Dice losses.
OCT2017 is evaluated at image level only, so the paper reports only its image AUROC. In the normal-only
regime, anomalies and their masks are generated online
from normal support images. BUSI does have evaluation masks, so its pixel AUROC could be
reported. We nevertheless omit BUSI pixel AUROC from the normal-only anchor tables
to retain a common image-only endpoint across BUSI, CheXpert, and Kaggle BrainMRI; the latter two
do not have evaluation masks. Thus this table-level omission must not be read as absence of BUSI
masks.

\subsection{Synthetic-anomaly generation}
\label{app:synthetic_anomalies}

The medical normal-only and industrial full-shot experiments use different synthesis procedures.
Both operate online: every access pairs the unchanged normal image with a newly synthesized image
and its dense target, so the focal and Dice losses see clean and synthetic examples together.

\ptitle{Medical normal-only synthesis.} BUSI, CheXpert, and Kaggle BrainMRI share one task mixture
adapted from MediCLIP \citep{zhang2024mediclip}. Each selected normal support image is first
expanded into a fixed set of 16 deterministic views spanning small rotations, translations, and
zooms, brightness and contrast jitter, and the two flips; synthesis is then resampled on the drawn
view at every access. One of three tasks is then drawn: CutPaste \citep{li2021cutpaste}, which
copies a rotated and rescaled crop of a normal support image into the view; a Gaussian intensity
change, which adds a smooth signed perturbation to it; or the identity, which returns the view
unchanged. Both non-identity tasks confine their edit to one or two Perlin-noise regions
\citep{perlin1985image} of randomly drawn size, orientation, and position, and neither consults a
foreground mask. CutPaste derives a soft target from the local source-to-destination intensity
difference, the Gaussian task uses the binary region mask, and the identity uses an all-zero target.

\ptitle{Industrial full-shot synthesis.} MVTec-AD and VisA use foreground-guided natural synthetic
anomalies \citep{schluter2022natural}. For each normal destination image, a source image is drawn
from the same object category and a rectangular crop of it is Poisson-blended into the destination.
Crop height and width are drawn independently from a heavy-tailed distribution with
category-specific bounds, capped against the destination's foreground bounding box, and then
rescaled; one to a few patches are pasted per image. Placement is restricted to the object: a fixed
per-image foreground mask, thresholded offline from the normal training images, supplies the paste
center, a placement is accepted only when the patch lies mostly on foreground, and the final patch
mask is intersected with that foreground. The dense target is a per-category calibrated sigmoid of
the resulting mean absolute RGB intensity change, median-filtered to suppress isolated pixels, and a
synthesis that leaves the destination effectively unchanged is dropped from the synthetic half of
the batch. The crop bounds, patch counts, and target-curve parameters are fixed per object category
and held identical across every anchor condition.

\ptitle{Internal checkpoint provenance.} Every result for our medical, industrial, adapter-family,
and transfer experiments comes from the checkpoint saved after the completed final training epoch.
None of these internal values selects an epoch using validation or test performance. The external
MVFA and MadCLIP replications follow their released selection rules, documented separately in
Appendix~\ref{app:external_models}.

\ptitle{Randomness and pairing.} Medical anchor-substitution tables average three training runs per
condition. The random-anchor construction couples each anchor draw to its run seed. Because seed
identifiers are not common to every reported condition, the main anchor comparison is made between
condition means and is not presented as a seed-paired test. The architecture ablation pairs its
three seeds across conditions. In the industrial randomized conditions, anchor seed equals training seed,
so the three seeds represent three geometry draws as well as three training runs. Correct-prompt
anchors are deterministic.

\section{External Validation Protocols}
\label{app:external_models}

These experiments retain the released visual, optimization, scoring, and evaluation pipelines of
MVFA~\citep{huang2024mvfa} and MadCLIP~\citep{shiri2025madclip}; only the anchor path is changed.
Both methods use one CLIP image encoder with multiple adapter branches. They are empirical extensions:
MVFA's residual parameterization and MadCLIP's jointly learned prompts are outside the exact
fixed-anchor, rotation-closed premise of Theorem~\ref{thm:gram}.

\ptitle{MVFA architecture and controls.} MVFA inserts residual adapters at multiple levels of one
CLIP image encoder, then uses its published, separate detection and segmentation readouts. Its
released prompt ensemble is encoded once into unit anchors $(\va_N,\va_A)$ with cosine $\rho$.
For every random arm, Gram--Schmidt applied to two seeded Gaussian draws produces an orthonormal
frame $(\ve_1,\ve_2)$; the frame seed equals the training seed.
Correct prompts use the released anchors. Gram-matched random anchors use
$\va_N=\ve_1$ and $\va_A=\rho\ve_1+\sqrt{1-\rho^2}\ve_2$. Unmatched orthogonal anchors use
$(\ve_1,\ve_2)$. Matched orthogonal anchors reuse this pair but multiply the similarity scale by
$\sqrt{1-\rho}$, restoring the correct effective scale. Cosines and matched difference norms are
checked to $10^{-6}$ before training, and the published visual, detection, and segmentation paths are
otherwise unchanged.

\ptitle{MadCLIP architecture and controls.} MadCLIP uses separate normal and abnormal residual
adaptation branches within one CLIP image encoder and jointly learns its prompt tokens. The released
reference learns two seven-template prompt banks with eight class-specific context tokens at the end,
front, and middle token positions; the text encoder is frozen. Because its fitted cosine $\rho$ is an
outcome rather than a prescribed input, MVFA's fixed Gram-matched control is unavailable. In the
text-free learned-scale arm, the text encoder is absent, a seeded orthonormal frame is frozen, and
$\va_N=\ve_1$, $\va_A=\rho\ve_1+\sqrt{1-\rho^2}\ve_2$ with $\rho=\tanh\theta$.
The scalar $\theta$ is the only trainable text-path quantity; it starts at zero and uses Adam with
$\beta=(0.5,0.999)$ and learning rate $0.02$. The prompt-scale arm uses the same text-free frame but
fixes $\rho=+0.986$, leaving no trainable text-path parameter. All visual branches, losses, data,
schedules, and released image and pixel scoring rules remain fixed across the three arms.

\ptitle{Realized geometry.} Across all 48 dataset--support--seed cells per MadCLIP arm, the learned
prompt cosine averaged $-0.984$ (range $-0.994$ to $-0.949$), and the text-free learned cosine
averaged $-0.994$ (range $-1.000$ to $-0.975$). The fixed prompt-scale arm remained at $+0.986$.
Thus both fitted paths approached the maximum attainable effective scale, with or without language;
the prompt-scale arm deliberately supplied a much smaller scale.

\ptitle{Checkpoint provenance and aggregation.} Unlike our internal experiments in
Appendix~\ref{app:training_protocol}, which report the completed final epoch, these external results
reproduce the released repositories' test-selected checkpoint protocol. Segmentation runs select the
checkpoint maximizing test pixel+image AUROC, and detection-only runs maximize test image AUROC.
Each method uses three training seeds. Tables~\ref{tab:mvfa_full} and~\ref{tab:madclip_full} report sample mean$\pm$sd.
AC is image AUROC and AS is pixel AUROC. OCT2017 pixel AUROC is not reported, so Avg AS covers BraTS2021,
LiverCT, and RESC; Avg AC covers all four datasets.

\begin{table}[!tbp]
\centering
\caption{Complete MVFA anchor-control results across support sizes (AUROC \%).}
\label{tab:mvfa_full}
\begingroup
\scriptsize
\setlength{\tabcolsep}{2.5pt}
\renewcommand{\arraystretch}{0.88}
\begin{adjustbox}{max width=\textwidth,center}
\begin{tabular}{l c c cc cc cc cc}
\toprule
& & \textbf{OCT2017} & \multicolumn{2}{c}{\textbf{BraTS2021}} & \multicolumn{2}{c}{\textbf{LiverCT}} & \multicolumn{2}{c}{\textbf{RESC}} & \multicolumn{2}{c}{\textbf{Avg}} \\
\cmidrule(lr){3-3} \cmidrule(lr){4-5} \cmidrule(lr){6-7} \cmidrule(lr){8-9} \cmidrule(lr){10-11}
\textbf{Anchor condition} & \textbf{$K$} & AC & AC & AS & AC & AS & AC & AS & AC & AS \\
\midrule
Correct prompts  & \multirow{4}{*}{2}  & $92.0{\pm}8.2$  & $92.3{\pm}1.0$  & $96.8{\pm}0.6$  & $80.5{\pm}7.9$  & $98.9{\pm}0.9$  & $90.9{\pm}1.3$  & $97.5{\pm}0.8$  & $88.9$  & $97.7$ \\
Gram-matched random  &   & $93.0{\pm}6.2$  & $90.8{\pm}1.2$  & $96.6{\pm}0.8$  & $79.0{\pm}7.8$  & $98.6{\pm}1.0$  & $90.4{\pm}2.0$  & $97.5{\pm}0.7$  & $88.3$  & $97.6$ \\
Unmatched orthogonal  &   & $90.8{\pm}6.5$  & $60.4{\pm}4.8$  & $93.3{\pm}1.3$  & $58.6{\pm}4.0$  & $95.3{\pm}0.7$  & $68.9{\pm}16.6$  & $91.4{\pm}2.3$  & $69.7$  & $93.3$ \\
Matched orthogonal  &   & $92.2{\pm}6.7$  & $90.6{\pm}1.2$  & $96.7{\pm}0.6$  & $77.7{\pm}7.9$  & $98.7{\pm}1.0$  & $91.5{\pm}0.2$  & $97.8{\pm}0.6$  & $88.0$  & $97.7$ \\
\midrule
Correct prompts  & \multirow{4}{*}{4}  & $98.7{\pm}0.4$  & $90.6{\pm}1.4$  & $96.6{\pm}0.4$  & $80.1{\pm}1.3$  & $99.1{\pm}0.6$  & $93.2{\pm}4.0$  & $98.2{\pm}0.9$  & $90.7$  & $97.9$ \\
Gram-matched random  &   & $98.4{\pm}0.5$  & $91.4{\pm}1.6$  & $95.2{\pm}1.6$  & $77.8{\pm}1.1$  & $98.9{\pm}0.7$  & $93.4{\pm}3.7$  & $98.2{\pm}1.0$  & $90.2$  & $97.5$ \\
Unmatched orthogonal  &   & $94.2{\pm}2.6$  & $61.2{\pm}3.4$  & $93.3{\pm}0.9$  & $59.8{\pm}8.0$  & $96.0{\pm}0.1$  & $73.8{\pm}16.3$  & $92.4{\pm}4.8$  & $72.2$  & $93.9$ \\
Matched orthogonal  &   & $98.6{\pm}0.2$  & $91.9{\pm}1.2$  & $96.2{\pm}1.1$  & $80.6{\pm}2.2$  & $99.1{\pm}0.5$  & $92.6{\pm}4.5$  & $98.1{\pm}0.8$  & $90.9$  & $97.8$ \\
\midrule
Correct prompts  & \multirow{4}{*}{8}  & $99.2{\pm}0.8$  & $92.3{\pm}2.7$  & $97.2{\pm}0.3$  & $80.8{\pm}3.8$  & $98.7{\pm}1.1$  & $97.6{\pm}0.3$  & $99.1{\pm}0.2$  & $92.5$  & $98.3$ \\
Gram-matched random  &   & $99.5{\pm}0.3$  & $91.7{\pm}3.1$  & $97.2{\pm}0.3$  & $77.9{\pm}9.9$  & $98.3{\pm}2.1$  & $97.7{\pm}0.1$  & $99.0{\pm}0.0$  & $91.7$  & $98.2$ \\
Unmatched orthogonal  &   & $89.8{\pm}5.4$  & $62.9{\pm}9.9$  & $93.7{\pm}0.8$  & $60.4{\pm}3.3$  & $96.1{\pm}0.4$  & $65.8{\pm}1.6$  & $90.4{\pm}0.3$  & $69.7$  & $93.4$ \\
Matched orthogonal  &   & $99.5{\pm}0.2$  & $92.0{\pm}2.7$  & $97.2{\pm}0.2$  & $83.0{\pm}4.7$  & $98.7{\pm}1.0$  & $97.6{\pm}0.2$  & $99.2{\pm}0.1$  & $93.0$  & $98.4$ \\
\midrule
Correct prompts  & \multirow{4}{*}{16}  & $99.9{\pm}0.0$  & $95.0{\pm}0.6$  & $97.5{\pm}0.9$  & $86.1{\pm}1.5$  & $98.8{\pm}1.4$  & $98.2{\pm}0.4$  & $99.4{\pm}0.0$  & $94.8$  & $98.6$ \\
Gram-matched random  &   & $99.9{\pm}0.0$  & $94.1{\pm}1.5$  & $97.6{\pm}0.1$  & $83.7{\pm}2.5$  & $99.1{\pm}0.9$  & $98.0{\pm}0.2$  & $99.4{\pm}0.1$  & $93.9$  & $98.7$ \\
Unmatched orthogonal  &   & $92.2{\pm}2.1$  & $66.8{\pm}4.0$  & $93.1{\pm}1.1$  & $63.0{\pm}5.2$  & $96.5{\pm}0.3$  & $72.9{\pm}9.1$  & $89.9{\pm}2.4$  & $73.7$  & $93.2$ \\
Matched orthogonal  &   & $99.9{\pm}0.1$  & $94.5{\pm}1.1$  & $97.5{\pm}0.6$  & $85.5{\pm}1.8$  & $99.0{\pm}1.0$  & $98.0{\pm}0.3$  & $99.3{\pm}0.1$  & $94.5$  & $98.6$ \\
\bottomrule
\end{tabular}
\end{adjustbox}
\endgroup
\end{table}

\begin{table}[!tbp]
\centering
\caption{Complete MadCLIP learned-prompt and text-free controls by support size (AUROC \%).}
\label{tab:madclip_full}
\begingroup
\scriptsize
\setlength{\tabcolsep}{2.5pt}
\renewcommand{\arraystretch}{0.88}
\begin{adjustbox}{max width=\textwidth,center}
\begin{tabular}{l c c cc cc cc cc}
\toprule
& & \textbf{OCT2017} & \multicolumn{2}{c}{\textbf{BraTS2021}} & \multicolumn{2}{c}{\textbf{LiverCT}} & \multicolumn{2}{c}{\textbf{RESC}} & \multicolumn{2}{c}{\textbf{Avg}} \\
\cmidrule(lr){3-3} \cmidrule(lr){4-5} \cmidrule(lr){6-7} \cmidrule(lr){8-9} \cmidrule(lr){10-11}
\textbf{Anchor condition} & \textbf{$K$} & AC & AC & AS & AC & AS & AC & AS & AC & AS \\
\midrule
Learned prompts  & \multirow{3}{*}{2}  & $96.5{\pm}1.5$  & $92.9{\pm}1.0$  & $97.6{\pm}0.3$  & $79.0{\pm}1.2$  & $97.3{\pm}0.3$  & $91.0{\pm}2.8$  & $96.8{\pm}1.1$  & $89.8$  & $97.2$ \\
Learned scale (no text)  &   & $96.3{\pm}1.6$  & $91.8{\pm}1.4$  & $97.6{\pm}0.2$  & $78.6{\pm}2.2$  & $97.5{\pm}0.4$  & $90.9{\pm}2.7$  & $97.1{\pm}0.8$  & $89.4$  & $97.4$ \\
Prompt-scale (no text)  &   & $97.1{\pm}1.4$  & $91.1{\pm}2.6$  & $97.1{\pm}0.8$  & $64.3{\pm}16.8$  & $71.1{\pm}22.7$  & $87.9{\pm}5.4$  & $96.4{\pm}0.7$  & $85.1$  & $88.2$ \\
\midrule
Learned prompts  & \multirow{3}{*}{4}  & $98.3{\pm}0.8$  & $92.4{\pm}1.8$  & $97.4{\pm}0.3$  & $84.7{\pm}3.2$  & $98.2{\pm}1.1$  & $95.0{\pm}1.3$  & $98.1{\pm}0.5$  & $92.6$  & $97.9$ \\
Learned scale (no text)  &   & $98.5{\pm}0.9$  & $91.9{\pm}2.4$  & $97.5{\pm}0.3$  & $85.1{\pm}1.9$  & $98.3{\pm}1.1$  & $95.0{\pm}0.9$  & $98.2{\pm}0.5$  & $92.6$  & $98.0$ \\
Prompt-scale (no text)  &   & $98.2{\pm}0.8$  & $90.8{\pm}5.1$  & $96.5{\pm}0.8$  & $73.7{\pm}4.9$  & $87.4{\pm}18.7$  & $94.4{\pm}0.1$  & $98.2{\pm}0.4$  & $89.3$  & $94.0$ \\
\midrule
Learned prompts  & \multirow{3}{*}{8}  & $98.9{\pm}0.5$  & $94.8{\pm}1.0$  & $97.8{\pm}0.3$  & $88.2{\pm}2.2$  & $99.2{\pm}0.5$  & $96.9{\pm}1.1$  & $98.9{\pm}0.1$  & $94.7$  & $98.6$ \\
Learned scale (no text)  &   & $98.8{\pm}0.5$  & $94.7{\pm}1.1$  & $97.8{\pm}0.4$  & $87.6{\pm}1.5$  & $99.2{\pm}0.6$  & $96.9{\pm}1.2$  & $98.9{\pm}0.1$  & $94.5$  & $98.6$ \\
Prompt-scale (no text)  &   & $98.6{\pm}0.3$  & $95.1{\pm}1.1$  & $97.7{\pm}0.3$  & $83.6{\pm}3.8$  & $98.7{\pm}1.0$  & $96.3{\pm}1.1$  & $98.8{\pm}0.2$  & $93.4$  & $98.4$ \\
\midrule
Learned prompts  & \multirow{3}{*}{16}  & $99.6{\pm}0.2$  & $96.0{\pm}0.5$  & $97.6{\pm}0.4$  & $90.5{\pm}0.9$  & $99.5{\pm}0.5$  & $98.6{\pm}0.2$  & $99.4{\pm}0.1$  & $96.2$  & $98.8$ \\
Learned scale (no text)  &   & $99.7{\pm}0.1$  & $95.8{\pm}0.6$  & $97.7{\pm}0.6$  & $91.1{\pm}0.8$  & $99.1{\pm}0.7$  & $98.6{\pm}0.2$  & $99.4{\pm}0.1$  & $96.3$  & $98.7$ \\
Prompt-scale (no text)  &   & $99.6{\pm}0.1$  & $96.1{\pm}0.5$  & $98.1{\pm}0.1$  & $87.7{\pm}3.4$  & $99.1{\pm}1.0$  & $97.7{\pm}0.2$  & $99.2{\pm}0.1$  & $95.3$  & $98.8$ \\
\bottomrule
\end{tabular}
\end{adjustbox}
\endgroup
\end{table}

\FloatBarrier
\section{Full Result Tables}
\label{app:full_results}

The main paper reports the CLIP ViT-L/14 medical anchor results, the combined external-detector
summary, aggregate transfer contrasts, and the compact industrial boundary experiment. The tables
below give, in order, the industrial replication across all three vision backbones, the
individual source--target transfer rows and their DINOv3-L/16 replication, the per-dataset
adapter-family results, and the complete three-vision-backbone medical anchor factorials.

\begin{table}[!h]
\centering
\caption{Full-shot industrial experiment across three vision backbones (\%). Each
vision backbone uses a linear adapter and full-shot synthetic-anomaly training. Cells are
mean$\pm$sd over three training seeds; for randomized rows, each training seed also selects the
corresponding geometry seed. Pixel columns are AUROC, per-region overlap (PRO) and maximum F1 (F1max); image columns are AUROC, average precision (AUPR) and F1max. The two AUROC columns are category-macro means over the $T$ categories; PRO, F1max and AUPR are pooled over all categories.}
\label{tab:industrial_anchor_conditions_backbones}
\begingroup
\scriptsize
\setlength{\tabcolsep}{5pt}
\renewcommand{\arraystretch}{1.0}
\begin{adjustbox}{max width=\textwidth,center}
\begin{tabular}{llcccccc}
\toprule
\multicolumn{8}{c}{\textbf{VisA}} \\
\midrule
\multirow{2}{*}{\textbf{Vision backbone}} & \multirow{2}{*}{\textbf{Anchor condition}} & \multicolumn{3}{c}{\textbf{Pixel}} & \multicolumn{3}{c}{\textbf{Image}} \\
\cmidrule(lr){3-5}\cmidrule(lr){6-8}
& & AUROC & PRO & F1max & AUROC & AUPR & F1max \\
\midrule
\multirow{3}{*}{\textbf{DINOv3-L/16}} & Correct prompts & $88.5{\pm}0.4$ & $65.7{\pm}0.3$ & $32.2{\pm}0.9$ & $83.1{\pm}0.3$ & $87.7{\pm}0.3$ & $76.3{\pm}0.4$ \\
 & Shared rotation & $88.4{\pm}0.3$ & $65.5{\pm}0.4$ & $32.0{\pm}0.9$ & $83.2{\pm}0.2$ & $87.8{\pm}0.3$ & $76.4{\pm}0.3$ \\
 & Independent rotations & $96.8{\pm}0.1$ & $76.2{\pm}0.5$ & $42.6{\pm}0.5$ & $91.9{\pm}0.4$ & $92.7{\pm}0.5$ & $82.7{\pm}0.9$ \\
\midrule
\multirow{3}{*}{\textbf{DINOv2-L/14}} & Correct prompts & $92.2{\pm}0.1$ & $68.4{\pm}1.1$ & $32.8{\pm}0.5$ & $89.0{\pm}0.3$ & $90.5{\pm}0.2$ & $80.0{\pm}0.5$ \\
 & Shared rotation & $92.2{\pm}0.2$ & $68.5{\pm}0.9$ & $32.8{\pm}0.3$ & $89.2{\pm}0.2$ & $90.7{\pm}0.2$ & $80.3{\pm}0.4$ \\
 & Independent rotations & $97.1{\pm}0.1$ & $76.8{\pm}1.5$ & $41.3{\pm}0.4$ & $92.5{\pm}0.3$ & $93.1{\pm}0.5$ & $83.4{\pm}0.9$ \\
\midrule
\multirow{3}{*}{\textbf{CLIP ViT-L/14}} & Correct prompts & $91.7{\pm}0.1$ & $61.0{\pm}1.1$ & $32.3{\pm}0.3$ & $86.2{\pm}0.1$ & $88.7{\pm}0.0$ & $78.3{\pm}0.3$ \\
 & Shared rotation & $91.6{\pm}0.0$ & $60.9{\pm}0.5$ & $32.1{\pm}0.4$ & $86.3{\pm}0.1$ & $88.7{\pm}0.1$ & $78.3{\pm}0.2$ \\
 & Independent rotations & $96.1{\pm}0.1$ & $73.2{\pm}0.8$ & $39.9{\pm}0.4$ & $91.8{\pm}0.1$ & $92.6{\pm}0.2$ & $83.0{\pm}0.6$ \\
\bottomrule
\end{tabular}
\end{adjustbox}
\\[4pt]
\begin{adjustbox}{max width=\textwidth,center}
\begin{tabular}{llcccccc}
\toprule
\multicolumn{8}{c}{\textbf{MVTec-AD}} \\
\midrule
\multirow{2}{*}{\textbf{Vision backbone}} & \multirow{2}{*}{\textbf{Anchor condition}} & \multicolumn{3}{c}{\textbf{Pixel}} & \multicolumn{3}{c}{\textbf{Image}} \\
\cmidrule(lr){3-5}\cmidrule(lr){6-8}
& & AUROC & PRO & F1max & AUROC & AUPR & F1max \\
\midrule
\multirow{3}{*}{\textbf{DINOv3-L/16}} & Correct prompts & $92.8{\pm}0.1$ & $82.2{\pm}0.6$ & $56.0{\pm}0.1$ & $95.0{\pm}0.1$ & $97.9{\pm}0.1$ & $92.6{\pm}0.4$ \\
 & Shared rotation & $92.8{\pm}0.0$ & $82.3{\pm}0.1$ & $55.9{\pm}0.2$ & $95.0{\pm}0.0$ & $97.9{\pm}0.0$ & $92.6{\pm}0.0$ \\
 & Independent rotations & $96.3{\pm}0.0$ & $85.6{\pm}0.2$ & $54.5{\pm}0.6$ & $97.3{\pm}0.1$ & $98.6{\pm}0.1$ & $94.7{\pm}0.0$ \\
\midrule
\multirow{3}{*}{\textbf{DINOv2-L/14}} & Correct prompts & $94.0{\pm}0.1$ & $85.4{\pm}0.5$ & $59.2{\pm}0.1$ & $96.1{\pm}0.2$ & $98.5{\pm}0.1$ & $94.0{\pm}0.3$ \\
 & Shared rotation & $94.0{\pm}0.1$ & $85.3{\pm}0.5$ & $59.2{\pm}0.1$ & $96.0{\pm}0.2$ & $98.5{\pm}0.1$ & $93.9{\pm}0.3$ \\
 & Independent rotations & $96.9{\pm}0.1$ & $85.8{\pm}0.4$ & $55.8{\pm}0.9$ & $97.5{\pm}0.1$ & $98.9{\pm}0.0$ & $95.3{\pm}0.3$ \\
\midrule
\multirow{3}{*}{\textbf{CLIP ViT-L/14}} & Correct prompts & $93.6{\pm}0.1$ & $85.2{\pm}0.5$ & $53.6{\pm}0.5$ & $96.2{\pm}0.1$ & $98.1{\pm}0.1$ & $92.9{\pm}0.2$ \\
 & Shared rotation & $93.6{\pm}0.1$ & $85.5{\pm}0.8$ & $53.7{\pm}0.2$ & $96.2{\pm}0.1$ & $98.1{\pm}0.1$ & $92.9{\pm}0.2$ \\
 & Independent rotations & $96.6{\pm}0.0$ & $88.4{\pm}0.4$ & $59.5{\pm}0.0$ & $97.4{\pm}0.1$ & $98.8{\pm}0.0$ & $94.6{\pm}0.2$ \\
\bottomrule
\end{tabular}
\end{adjustbox}
\endgroup
\end{table}

\begin{table}[!tbp]
\centering
\caption{Anchor swap cost under a frozen transferred adapter. CLIP ViT-L/14,
$K{=}16$, image AUROC (\%), mean over three seeds. A detector trained on $S$ is
evaluated on $T$ twice, differing only in which anchors score the images:
\textbf{src} uses $S$'s, \textbf{tgt} uses $T$'s; both share the same domain
shift, so \textbf{swap} $=$ src $-$ tgt isolates the substitution.
$\Delta_{\mathrm{swap}}$ is swap(random) $-$ swap(correct).
Dataset keys are B=BraTS2021, L=LiverCT, O=OCT2017, R=RESC, M=Kaggle BrainMRI, U=BUSI, and
C=CheXpert.}
\label{tab:transfer_anchor_swap}
\begingroup
\scriptsize
\setlength{\tabcolsep}{3pt}
\renewcommand{\arraystretch}{1.0}
\begin{adjustbox}{max totalsize={\textwidth}{0.55\textheight},center}
\begin{tabular}{lcccccc@{\hspace{1.5em}}lcccccc}
\toprule
\multicolumn{7}{c}{\textbf{Real-anomaly support}} & \multicolumn{7}{c}{\textbf{Normal-only}} \\
\cmidrule(lr){1-7}\cmidrule(lr){8-14}
& \multicolumn{3}{c}{\textbf{Correct prompts}} & \multicolumn{3}{c}{\textbf{Gram-matched random}}
& & \multicolumn{3}{c}{\textbf{Correct prompts}} & \multicolumn{3}{c}{\textbf{Gram-matched random}} \\
\cmidrule(lr){2-4}\cmidrule(lr){5-7}\cmidrule(lr){9-11}\cmidrule(lr){12-14}
$S{\to}T$ & \textbf{src} & \textbf{tgt} & \textbf{swap} & \textbf{src} & \textbf{tgt} & \textbf{swap}
& $S{\to}T$ & \textbf{src} & \textbf{tgt} & \textbf{swap} & \textbf{src} & \textbf{tgt} & \textbf{swap} \\
\midrule
B$\to$L & $63.4$ & $63.9$ & $-0.5$ & $64.1$ & $50.9$ & $13.1$ & M$\to$U & $73.1$ & $73.7$ & $-0.6$ & $73.9$ & $55.2$ & $18.7$ \\
B$\to$O & $59.7$ & $61.1$ & $-1.5$ & $58.7$ & $64.7$ & $-6.0$ & M$\to$C & $57.6$ & $56.9$ & $0.7$ & $57.6$ & $54.2$ & $3.4$ \\
B$\to$R & $45.2$ & $44.5$ & $0.6$ & $44.9$ & $48.8$ & $-3.9$ & U$\to$M & $78.9$ & $79.6$ & $-0.7$ & $79.7$ & $36.8$ & $42.9$ \\
L$\to$B & $68.9$ & $69.2$ & $-0.3$ & $68.7$ & $54.4$ & $14.3$ & U$\to$C & $55.8$ & $57.2$ & $-1.3$ & $58.7$ & $51.3$ & $7.4$ \\
L$\to$O & $69.7$ & $62.0$ & $7.8$ & $69.7$ & $55.8$ & $13.9$ & C$\to$M & $75.8$ & $68.0$ & $7.9$ & $74.7$ & $32.1$ & $42.5$ \\
L$\to$R & $66.1$ & $66.8$ & $-0.8$ & $67.1$ & $58.6$ & $8.5$ & C$\to$U & $68.3$ & $69.6$ & $-1.3$ & $67.0$ & $67.1$ & $-0.1$ \\
O$\to$B & $55.4$ & $54.2$ & $1.2$ & $54.2$ & $52.5$ & $1.8$ & & & & & & & \\
O$\to$L & $62.3$ & $60.2$ & $2.1$ & $63.4$ & $56.7$ & $6.7$ & & & & & & & \\
O$\to$R & $84.3$ & $86.6$ & $-2.4$ & $84.0$ & $61.9$ & $22.1$ & & & & & & & \\
R$\to$B & $45.6$ & $45.9$ & $-0.3$ & $45.7$ & $54.0$ & $-8.2$ & & & & & & & \\
R$\to$L & $51.6$ & $51.4$ & $0.2$ & $52.0$ & $49.2$ & $2.8$ & & & & & & & \\
R$\to$O & $90.3$ & $89.3$ & $1.0$ & $90.5$ & $42.0$ & $48.5$ & & & & & & & \\
\midrule
\textit{mean} & $63.5$ & $62.9$ & $0.6$ & $63.6$ & $54.1$ & $9.5$ & \textit{mean} & $68.3$ & $67.5$ & $0.8$ & $68.6$ & $49.4$ & $19.1$ \\
\textbf{$\Delta_{\mathrm{swap}}$} & \multicolumn{6}{c}{$+8.9$} & \textbf{$\Delta_{\mathrm{swap}}$} & \multicolumn{6}{c}{$+18.4$} \\
\bottomrule
\end{tabular}
\end{adjustbox}
\endgroup
\end{table}

\begin{table}[!tbp]
\centering
\caption{Anchor swap cost under a frozen transferred adapter. DINOv3-L/16,
$K{=}16$, image AUROC (\%), mean over three seeds. A detector trained on $S$ is
evaluated on $T$ twice, differing only in which anchors score the images:
\textbf{src} uses $S$'s, \textbf{tgt} uses $T$'s; both share the same domain
shift, so \textbf{swap} $=$ src $-$ tgt isolates the substitution.
$\Delta_{\mathrm{swap}}$ is swap(random) $-$ swap(correct). The source-anchor score reuses the pair
that trained the adapter and preserves its anchor frame. In the random condition, the target pair
uses a different seed and changes the readout direction without co-rotating the adapter; its swap
measures coordinate-frame disruption, not a benefit from linguistic meaning.
Dataset keys are B=BraTS2021, L=LiverCT, O=OCT2017, R=RESC, M=Kaggle BrainMRI, U=BUSI, and
C=CheXpert.}
\label{tab:transfer_anchor_swap_dinov3}
\begingroup
\scriptsize
\setlength{\tabcolsep}{3pt}
\renewcommand{\arraystretch}{1.0}
\begin{adjustbox}{max totalsize={\textwidth}{0.55\textheight},center}
\begin{tabular}{lcccccc@{\hspace{1.5em}}lcccccc}
\toprule
\multicolumn{7}{c}{\textbf{Real-anomaly support}} & \multicolumn{7}{c}{\textbf{Normal-only}} \\
\cmidrule(lr){1-7}\cmidrule(lr){8-14}
& \multicolumn{3}{c}{\textbf{Correct prompts}} & \multicolumn{3}{c}{\textbf{Gram-matched random}}
& & \multicolumn{3}{c}{\textbf{Correct prompts}} & \multicolumn{3}{c}{\textbf{Gram-matched random}} \\
\cmidrule(lr){2-4}\cmidrule(lr){5-7}\cmidrule(lr){9-11}\cmidrule(lr){12-14}
$S{\to}T$ & \textbf{src} & \textbf{tgt} & \textbf{swap} & \textbf{src} & \textbf{tgt} & \textbf{swap}
& $S{\to}T$ & \textbf{src} & \textbf{tgt} & \textbf{swap} & \textbf{src} & \textbf{tgt} & \textbf{swap} \\
\midrule
B$\to$L & $71.2$ & $73.9$ & $-2.6$ & $71.6$ & $47.6$ & $24.0$ & M$\to$U & $71.6$ & $70.0$ & $1.6$ & $70.5$ & $51.3$ & $19.1$ \\
B$\to$O & $80.2$ & $87.6$ & $-7.3$ & $80.3$ & $64.1$ & $16.2$ & M$\to$C & $49.3$ & $51.1$ & $-1.8$ & $48.6$ & $46.3$ & $2.3$ \\
B$\to$R & $74.7$ & $79.1$ & $-4.4$ & $75.5$ & $50.1$ & $25.4$ & U$\to$M & $93.0$ & $92.9$ & $0.2$ & $92.1$ & $51.2$ & $40.9$ \\
L$\to$B & $67.6$ & $67.7$ & $-0.1$ & $63.6$ & $45.9$ & $17.7$ & U$\to$C & $55.7$ & $56.4$ & $-0.8$ & $56.5$ & $48.1$ & $8.4$ \\
L$\to$O & $72.4$ & $75.6$ & $-3.1$ & $72.1$ & $48.4$ & $23.7$ & C$\to$M & $81.2$ & $80.0$ & $1.2$ & $80.6$ & $51.0$ & $29.6$ \\
L$\to$R & $75.0$ & $78.0$ & $-3.0$ & $75.3$ & $50.2$ & $25.0$ & C$\to$U & $45.2$ & $47.3$ & $-2.1$ & $44.8$ & $65.7$ & $-20.9$ \\
O$\to$B & $57.5$ & $58.9$ & $-1.4$ & $61.4$ & $44.5$ & $17.0$ & & & & & & & \\
O$\to$L & $67.4$ & $64.2$ & $3.2$ & $65.7$ & $57.8$ & $7.9$ & & & & & & & \\
O$\to$R & $86.5$ & $84.8$ & $1.7$ & $86.6$ & $48.2$ & $38.3$ & & & & & & & \\
R$\to$B & $52.9$ & $53.0$ & $-0.1$ & $52.2$ & $58.3$ & $-6.1$ & & & & & & & \\
R$\to$L & $53.8$ & $53.4$ & $0.5$ & $54.0$ & $57.5$ & $-3.5$ & & & & & & & \\
R$\to$O & $94.5$ & $93.5$ & $1.0$ & $94.3$ & $45.3$ & $49.0$ & & & & & & & \\
\midrule
\textit{mean} & $71.2$ & $72.5$ & $-1.3$ & $71.0$ & $51.5$ & $19.6$ & \textit{mean} & $66.0$ & $66.3$ & $-0.3$ & $65.5$ & $52.3$ & $13.2$ \\
\textbf{$\Delta_{\mathrm{swap}}$} & \multicolumn{6}{c}{$+20.9$} & \textbf{$\Delta_{\mathrm{swap}}$} & \multicolumn{6}{c}{$+13.5$} \\
\bottomrule
\end{tabular}
\end{adjustbox}
\endgroup
\end{table}

\begin{table}[p]
\centering
\caption{Per-dataset results across adapter architectures, CLIP ViT-L/14, $K{=}16$, under real-anomaly support (top) and normal-only synthetic-anomaly training (bottom). Cells are mean$\pm$sd over three training seeds; AC is image-level and AS pixel-level AUROC (\%). Each panel fixes the anchor condition. Avg is the unweighted mean over datasets, AS over BraTS2021, LiverCT, and RESC, for which pixel AUROC is reported. As in Table~\ref{tab:semantic_necessity_audit_syn}, the normal-only block reports image AUROC only: BUSI has evaluation masks, but its pixel AUROC is omitted to keep a common image-level endpoint, and CheXpert and Kaggle BrainMRI have no evaluation masks. The families are approximately parameter-matched: summed over the four retained layers, the linear ($1024\!\to\!768$), MLP ($1024\!\to\!448\!\to\!768$), and Transformer ($1024\!\to\!192\!\to\!768$) adapters have 3,148,800, 3,216,128, and 3,159,552 trainable parameters.}
\label{tab:adapter_family_dataset_ac_clip}
\begingroup
\scriptsize
\setlength{\tabcolsep}{3pt}
\renewcommand{\arraystretch}{1.05}
\textbf{Real-anomaly support}\\[2pt]
\begin{adjustbox}{max width=\textwidth,center}
\begin{tabular}{lccccccccc}
\toprule
\multirow{2}{*}{\textbf{Adapter}} &
\textbf{OCT2017} &
\multicolumn{2}{c}{\textbf{BraTS2021}} &
\multicolumn{2}{c}{\makecell{\textbf{LiverCT}\\\textbf{(BTCV+LiTS)}}} &
\multicolumn{2}{c}{\textbf{RESC}} &
\multicolumn{2}{c}{\textbf{Avg}} \\
\cmidrule(lr){2-2}
\cmidrule(lr){3-4}
\cmidrule(lr){5-6}
\cmidrule(lr){7-8}
\cmidrule(lr){9-10}
& \textbf{AC} &
\textbf{AC} & \textbf{AS} &
\textbf{AC} & \textbf{AS} &
\textbf{AC} & \textbf{AS} &
\textbf{AC} & \textbf{AS} \\
\midrule
\multicolumn{10}{l}{\textbf{Correct prompts}} \\
\quad Linear & $99.6{\pm}0.1$ & $96.9{\pm}0.2$ & $99.2{\pm}0.0$ & $91.6{\pm}1.1$ & $99.9{\pm}0.0$ & $97.2{\pm}0.9$ & $99.0{\pm}0.1$ & $96.3$ & $99.4$ \\
\quad MLP & $99.5{\pm}0.2$ & $97.0{\pm}0.0$ & $99.1{\pm}0.1$ & $92.7{\pm}1.4$ & $99.8{\pm}0.1$ & $97.7{\pm}1.2$ & $99.2{\pm}0.1$ & $96.7$ & $99.4$ \\
\quad Transformer & $99.4{\pm}0.1$ & $96.8{\pm}0.3$ & $99.0{\pm}0.0$ & $90.9{\pm}1.2$ & $99.8{\pm}0.1$ & $97.3{\pm}1.0$ & $99.1{\pm}0.0$ & $96.1$ & $99.3$ \\
\midrule
\multicolumn{10}{l}{\textbf{Mismatched-dataset bank}} \\
\quad Linear & $99.7{\pm}0.1$ & $96.9{\pm}0.2$ & $99.3{\pm}0.0$ & $90.8{\pm}1.3$ & $99.9{\pm}0.0$ & $97.2{\pm}0.9$ & $99.0{\pm}0.1$ & $96.1$ & $99.4$ \\
\quad MLP & $99.5{\pm}0.4$ & $97.1{\pm}0.1$ & $99.1{\pm}0.1$ & $91.7{\pm}1.3$ & $99.8{\pm}0.1$ & $97.8{\pm}1.1$ & $99.2{\pm}0.1$ & $96.5$ & $99.4$ \\
\quad Transformer & $99.5{\pm}0.0$ & $96.8{\pm}0.2$ & $99.0{\pm}0.0$ & $90.4{\pm}1.6$ & $99.8{\pm}0.1$ & $97.1{\pm}1.2$ & $99.0{\pm}0.2$ & $96.0$ & $99.3$ \\
\midrule
\multicolumn{10}{l}{\textbf{Gram-matched random}} \\
\quad Linear & $99.6{\pm}0.1$ & $96.9{\pm}0.2$ & $99.1{\pm}0.0$ & $91.5{\pm}1.2$ & $99.9{\pm}0.0$ & $97.2{\pm}0.9$ & $99.0{\pm}0.1$ & $96.3$ & $99.3$ \\
\quad MLP & $99.4{\pm}0.3$ & $97.1{\pm}0.2$ & $99.0{\pm}0.1$ & $93.0{\pm}0.7$ & $99.8{\pm}0.1$ & $97.7{\pm}1.0$ & $99.2{\pm}0.1$ & $96.8$ & $99.4$ \\
\quad Transformer & $99.4{\pm}0.2$ & $96.9{\pm}0.2$ & $99.0{\pm}0.0$ & $91.4{\pm}1.0$ & $99.8{\pm}0.0$ & $97.2{\pm}1.0$ & $99.0{\pm}0.1$ & $96.2$ & $99.3$ \\
\midrule
\multicolumn{10}{l}{\textbf{Unmatched orthogonal}} \\
\quad Linear & $99.2{\pm}0.0$ & $97.0{\pm}0.1$ & $99.2{\pm}0.0$ & $89.0{\pm}0.3$ & $99.8{\pm}0.0$ & $97.1{\pm}1.0$ & $99.1{\pm}0.1$ & $95.6$ & $99.4$ \\
\quad MLP & $98.6{\pm}0.9$ & $97.1{\pm}0.2$ & $99.1{\pm}0.0$ & $89.4{\pm}1.6$ & $99.6{\pm}0.4$ & $97.9{\pm}1.2$ & $99.1{\pm}0.2$ & $95.8$ & $99.3$ \\
\quad Transformer & $81.7{\pm}27.4$ & $96.9{\pm}0.2$ & $99.0{\pm}0.0$ & $87.3{\pm}0.9$ & $98.9{\pm}0.8$ & $97.1{\pm}1.0$ & $99.1{\pm}0.1$ & $90.7$ & $99.0$ \\
\midrule
\multicolumn{10}{l}{\textbf{Matched orthogonal}} \\
\quad Linear & $99.6{\pm}0.1$ & $96.9{\pm}0.2$ & $99.1{\pm}0.0$ & $91.6{\pm}1.4$ & $99.9{\pm}0.0$ & $97.2{\pm}0.9$ & $99.0{\pm}0.1$ & $96.3$ & $99.3$ \\
\quad MLP & $99.3{\pm}0.2$ & $97.1{\pm}0.1$ & $99.1{\pm}0.1$ & $92.9{\pm}1.4$ & $99.8{\pm}0.1$ & $97.8{\pm}0.9$ & $99.2{\pm}0.2$ & $96.8$ & $99.4$ \\
\quad Transformer & $99.3{\pm}0.1$ & $96.8{\pm}0.2$ & $99.0{\pm}0.0$ & $91.1{\pm}1.3$ & $99.8{\pm}0.0$ & $97.4{\pm}1.0$ & $99.0{\pm}0.1$ & $96.1$ & $99.3$ \\
\bottomrule
\end{tabular}
\end{adjustbox}
\\[6pt]
\textbf{Normal-only}\\[2pt]
\begin{adjustbox}{max width=\textwidth,center}
\begin{tabular}{lcccc}
\toprule
\multirow{2}{*}{\textbf{Adapter}} &
\textbf{\makecell{Kaggle\\BrainMRI}} & \textbf{CheXpert} &
\textbf{BUSI} & \textbf{Avg} \\
\cmidrule(lr){2-2}
\cmidrule(lr){3-3}
\cmidrule(lr){4-4}
\cmidrule(lr){5-5}
& \textbf{AC} & \textbf{AC} &
\textbf{AC} & \textbf{AC} \\
\midrule
\multicolumn{5}{l}{\textbf{Correct prompts}} \\
\quad Linear & $92.9{\pm}0.8$ & $69.3{\pm}1.6$ & $88.2{\pm}0.6$ & $83.4$ \\
\quad MLP & $94.8{\pm}0.9$ & $71.6{\pm}1.0$ & $89.2{\pm}0.4$ & $85.2$ \\
\quad Transformer & $86.7{\pm}2.0$ & $66.7{\pm}1.5$ & $86.2{\pm}0.6$ & $79.9$ \\
\midrule
\multicolumn{5}{l}{\textbf{Mismatched-dataset bank}} \\
\quad Linear & $92.9{\pm}0.7$ & $69.0{\pm}2.3$ & $87.7{\pm}0.9$ & $83.2$ \\
\quad MLP & $94.9{\pm}0.8$ & $72.3{\pm}1.5$ & $89.3{\pm}0.9$ & $85.5$ \\
\quad Transformer & $89.8{\pm}2.6$ & $65.2{\pm}1.5$ & $86.0{\pm}0.7$ & $80.3$ \\
\midrule
\multicolumn{5}{l}{\textbf{Gram-matched random}} \\
\quad Linear & $92.9{\pm}0.2$ & $69.5{\pm}1.3$ & $88.4{\pm}0.6$ & $83.6$ \\
\quad MLP & $94.7{\pm}0.7$ & $72.3{\pm}0.4$ & $89.3{\pm}1.0$ & $85.4$ \\
\quad Transformer & $88.2{\pm}2.4$ & $65.1{\pm}1.8$ & $84.7{\pm}1.8$ & $79.4$ \\
\midrule
\multicolumn{5}{l}{\textbf{Unmatched orthogonal}} \\
\quad Linear & $92.5{\pm}0.4$ & $69.8{\pm}1.3$ & $87.8{\pm}0.6$ & $83.4$ \\
\quad MLP & $94.0{\pm}1.3$ & $71.5{\pm}1.0$ & $89.0{\pm}0.5$ & $84.8$ \\
\quad Transformer & $92.7{\pm}0.2$ & $67.3{\pm}2.1$ & $87.4{\pm}1.2$ & $82.5$ \\
\midrule
\multicolumn{5}{l}{\textbf{Matched orthogonal}} \\
\quad Linear & $92.9{\pm}0.7$ & $69.4{\pm}1.7$ & $87.8{\pm}0.8$ & $83.4$ \\
\quad MLP & $95.1{\pm}0.3$ & $72.0{\pm}0.4$ & $89.4{\pm}1.0$ & $85.5$ \\
\quad Transformer & $89.2{\pm}3.7$ & $64.2{\pm}1.9$ & $85.2{\pm}1.4$ & $79.5$ \\
\bottomrule
\end{tabular}
\end{adjustbox}
\endgroup
\end{table}

\begin{table}[p]
\centering
\caption{Per-dataset results across adapter architectures, DINOv3-L/16, $K{=}16$, under real-anomaly support (top) and normal-only synthetic-anomaly training (bottom). Cells are mean$\pm$sd over three training seeds; AC is image-level and AS pixel-level AUROC (\%). Each panel fixes the anchor condition. Avg is the unweighted mean over datasets, AS over BraTS2021, LiverCT, and RESC, for which pixel AUROC is reported. As in Table~\ref{tab:semantic_necessity_audit_syn}, the normal-only block reports image AUROC only: BUSI has evaluation masks, but its pixel AUROC is omitted to keep a common image-level endpoint, and CheXpert and Kaggle BrainMRI have no evaluation masks. This table replicates Table~\ref{tab:adapter_family_dataset_ac_clip} on a second vision backbone.}
\label{tab:adapter_family_dataset_ac}
\begingroup
\scriptsize
\setlength{\tabcolsep}{3pt}
\renewcommand{\arraystretch}{1.05}
\textbf{Real-anomaly support}\\[2pt]
\begin{adjustbox}{max width=\textwidth,center}
\begin{tabular}{lccccccccc}
\toprule
\multirow{2}{*}{\textbf{Adapter}} &
\textbf{OCT2017} &
\multicolumn{2}{c}{\textbf{BraTS2021}} &
\multicolumn{2}{c}{\makecell{\textbf{LiverCT}\\\textbf{(BTCV+LiTS)}}} &
\multicolumn{2}{c}{\textbf{RESC}} &
\multicolumn{2}{c}{\textbf{Avg}} \\
\cmidrule(lr){2-2}
\cmidrule(lr){3-4}
\cmidrule(lr){5-6}
\cmidrule(lr){7-8}
\cmidrule(lr){9-10}
& \textbf{AC} &
\textbf{AC} & \textbf{AS} &
\textbf{AC} & \textbf{AS} &
\textbf{AC} & \textbf{AS} &
\textbf{AC} & \textbf{AS} \\
\midrule
\multicolumn{10}{l}{\textbf{Correct prompts}} \\
\quad Linear & $99.3{\pm}0.1$ & $97.2{\pm}0.2$ & $99.4{\pm}0.1$ & $91.4{\pm}0.9$ & $99.9{\pm}0.0$ & $97.9{\pm}0.7$ & $99.3{\pm}0.1$ & $96.4$ & $99.5$ \\
\quad MLP & $98.9{\pm}0.1$ & $97.2{\pm}0.3$ & $99.3{\pm}0.1$ & $91.1{\pm}2.3$ & $99.8{\pm}0.1$ & $98.1{\pm}0.6$ & $99.4{\pm}0.0$ & $96.3$ & $99.5$ \\
\quad Transformer & $99.1{\pm}0.5$ & $97.0{\pm}0.2$ & $99.3{\pm}0.1$ & $91.1{\pm}0.9$ & $99.8{\pm}0.1$ & $97.8{\pm}0.9$ & $99.2{\pm}0.0$ & $96.2$ & $99.4$ \\
\midrule
\multicolumn{10}{l}{\textbf{Mismatched-dataset bank}} \\
\quad Linear & $99.3{\pm}0.2$ & $97.1{\pm}0.2$ & $99.4{\pm}0.1$ & $90.9{\pm}1.0$ & $99.9{\pm}0.0$ & $97.8{\pm}0.8$ & $99.2{\pm}0.1$ & $96.3$ & $99.5$ \\
\quad MLP & $99.1{\pm}0.2$ & $97.1{\pm}0.4$ & $99.3{\pm}0.1$ & $91.1{\pm}2.4$ & $99.8{\pm}0.2$ & $98.0{\pm}0.6$ & $99.3{\pm}0.1$ & $96.3$ & $99.4$ \\
\quad Transformer & $99.3{\pm}0.3$ & $96.9{\pm}0.4$ & $99.3{\pm}0.1$ & $90.7{\pm}1.7$ & $99.6{\pm}0.4$ & $97.9{\pm}0.8$ & $99.2{\pm}0.0$ & $96.2$ & $99.4$ \\
\midrule
\multicolumn{10}{l}{\textbf{Gram-matched random}} \\
\quad Linear & $99.2{\pm}0.3$ & $97.2{\pm}0.2$ & $99.4{\pm}0.1$ & $91.4{\pm}0.9$ & $99.9{\pm}0.1$ & $97.9{\pm}0.7$ & $99.2{\pm}0.1$ & $96.4$ & $99.5$ \\
\quad MLP & $99.2{\pm}0.2$ & $97.2{\pm}0.1$ & $99.2{\pm}0.1$ & $91.1{\pm}1.8$ & $99.8{\pm}0.1$ & $98.0{\pm}0.6$ & $99.3{\pm}0.1$ & $96.4$ & $99.5$ \\
\quad Transformer & $99.1{\pm}0.4$ & $97.1{\pm}0.2$ & $99.3{\pm}0.1$ & $90.9{\pm}1.4$ & $99.8{\pm}0.1$ & $97.7{\pm}1.0$ & $99.1{\pm}0.1$ & $96.2$ & $99.4$ \\
\midrule
\multicolumn{10}{l}{\textbf{Unmatched orthogonal}} \\
\quad Linear & $98.7{\pm}0.1$ & $97.1{\pm}0.2$ & $99.4{\pm}0.1$ & $88.9{\pm}1.6$ & $99.8{\pm}0.1$ & $97.7{\pm}0.8$ & $99.2{\pm}0.1$ & $95.6$ & $99.5$ \\
\quad MLP & $98.0{\pm}0.8$ & $97.2{\pm}0.1$ & $99.4{\pm}0.1$ & $87.9{\pm}1.8$ & $99.7{\pm}0.2$ & $98.1{\pm}0.9$ & $99.4{\pm}0.1$ & $95.3$ & $99.5$ \\
\quad Transformer & $50.0{\pm}0.0$ & $97.1{\pm}0.1$ & $99.4{\pm}0.1$ & $65.1{\pm}19.9$ & $55.2{\pm}38.5$ & $97.9{\pm}0.6$ & $99.2{\pm}0.0$ & $77.5$ & $84.6$ \\
\midrule
\multicolumn{10}{l}{\textbf{Matched orthogonal}} \\
\quad Linear & $99.2{\pm}0.3$ & $97.1{\pm}0.2$ & $99.4{\pm}0.1$ & $91.3{\pm}0.8$ & $99.9{\pm}0.0$ & $97.9{\pm}0.7$ & $99.2{\pm}0.1$ & $96.4$ & $99.5$ \\
\quad MLP & $99.0{\pm}0.3$ & $97.1{\pm}0.3$ & $99.3{\pm}0.1$ & $90.8{\pm}1.3$ & $99.8{\pm}0.1$ & $98.1{\pm}0.6$ & $99.3{\pm}0.1$ & $96.3$ & $99.5$ \\
\quad Transformer & $99.2{\pm}0.2$ & $97.0{\pm}0.1$ & $99.3{\pm}0.1$ & $91.3{\pm}1.5$ & $99.7{\pm}0.2$ & $98.0{\pm}0.7$ & $99.2{\pm}0.1$ & $96.4$ & $99.4$ \\
\bottomrule
\end{tabular}
\end{adjustbox}
\\[6pt]
\textbf{Normal-only}\\[2pt]
\begin{adjustbox}{max width=\textwidth,center}
\begin{tabular}{lcccc}
\toprule
\multirow{2}{*}{\textbf{Adapter}} &
\textbf{\makecell{Kaggle\\BrainMRI}} & \textbf{CheXpert} &
\textbf{BUSI} & \textbf{Avg} \\
\cmidrule(lr){2-2}
\cmidrule(lr){3-3}
\cmidrule(lr){4-4}
\cmidrule(lr){5-5}
& \textbf{AC} & \textbf{AC} &
\textbf{AC} & \textbf{AC} \\
\midrule
\multicolumn{5}{l}{\textbf{Correct prompts}} \\
\quad Linear & $95.2{\pm}0.4$ & $70.8{\pm}0.9$ & $89.2{\pm}0.8$ & $85.0$ \\
\quad MLP & $94.5{\pm}0.3$ & $69.4{\pm}0.7$ & $88.8{\pm}0.9$ & $84.2$ \\
\quad Transformer & $94.5{\pm}0.6$ & $65.5{\pm}1.3$ & $89.0{\pm}0.5$ & $83.0$ \\
\midrule
\multicolumn{5}{l}{\textbf{Mismatched-dataset bank}} \\
\quad Linear & $95.0{\pm}0.3$ & $70.8{\pm}1.1$ & $89.9{\pm}1.0$ & $85.2$ \\
\quad MLP & $94.4{\pm}0.6$ & $69.4{\pm}0.5$ & $89.0{\pm}0.5$ & $84.3$ \\
\quad Transformer & $94.2{\pm}0.5$ & $64.9{\pm}0.8$ & $89.1{\pm}1.2$ & $82.7$ \\
\midrule
\multicolumn{5}{l}{\textbf{Gram-matched random}} \\
\quad Linear & $95.2{\pm}0.4$ & $70.8{\pm}1.0$ & $89.4{\pm}0.6$ & $85.1$ \\
\quad MLP & $94.5{\pm}0.3$ & $69.1{\pm}0.5$ & $87.9{\pm}1.2$ & $83.8$ \\
\quad Transformer & $94.3{\pm}0.2$ & $64.7{\pm}2.4$ & $88.9{\pm}0.6$ & $82.6$ \\
\midrule
\multicolumn{5}{l}{\textbf{Unmatched orthogonal}} \\
\quad Linear & $93.5{\pm}1.0$ & $70.3{\pm}1.7$ & $88.1{\pm}2.2$ & $84.0$ \\
\quad MLP & $94.9{\pm}0.6$ & $69.5{\pm}0.3$ & $89.8{\pm}3.2$ & $84.8$ \\
\quad Transformer & $94.0{\pm}0.8$ & $66.1{\pm}1.8$ & $89.0{\pm}1.0$ & $83.0$ \\
\midrule
\multicolumn{5}{l}{\textbf{Matched orthogonal}} \\
\quad Linear & $95.1{\pm}0.5$ & $70.6{\pm}0.6$ & $89.6{\pm}0.7$ & $85.1$ \\
\quad MLP & $94.6{\pm}0.5$ & $69.7{\pm}1.3$ & $89.4{\pm}0.6$ & $84.5$ \\
\quad Transformer & $94.4{\pm}0.5$ & $64.7{\pm}1.1$ & $88.6{\pm}0.9$ & $82.6$ \\
\bottomrule
\end{tabular}
\end{adjustbox}
\endgroup
\end{table}

\begin{table}[p]
\centering
\caption{Anchor substitution with real-anomaly support ($K$ normal $+$ $K$ observed abnormal samples). Cells are mean$\pm$sd over three training seeds; AC is image-level and AS pixel-level AUROC (\%). Avg is the unweighted mean over datasets, AS over BraTS2021, LiverCT, and RESC, for which pixel AUROC is reported. Rows within a block share support samples, training and scoring; only the anchor pair differs.}
\label{tab:semantic_necessity_audit_real}
\begingroup
\scriptsize
\setlength{\tabcolsep}{3pt}
\renewcommand{\arraystretch}{1.15}
\begin{adjustbox}{max totalsize={\textwidth}{0.80\textheight},center}
\begin{tabular}{llccccccccc}
\toprule
\multirow{2}{*}{\textbf{$K$}} & \multirow{2}{*}{\textbf{Anchor condition}} &
\textbf{OCT2017} &
\multicolumn{2}{c}{\textbf{BraTS2021}} &
\multicolumn{2}{c}{\makecell{\textbf{LiverCT}\\\textbf{(BTCV+LiTS)}}} &
\multicolumn{2}{c}{\textbf{RESC}} &
\multicolumn{2}{c}{\textbf{Avg}} \\
\cmidrule(lr){3-3}
\cmidrule(lr){4-5}
\cmidrule(lr){6-7}
\cmidrule(lr){8-9}
\cmidrule(lr){10-11}
& & \textbf{AC} &
\textbf{AC} & \textbf{AS} &
\textbf{AC} & \textbf{AS} &
\textbf{AC} & \textbf{AS} &
\textbf{AC} & \textbf{AS} \\
\midrule
\multirow{18}{*}{2} & \multicolumn{9}{l}{\textbf{DINOv3-L/16}} \\
 & \quad Correct prompts & $93.6{\pm}6.7$ & $94.9{\pm}1.0$ & $98.8{\pm}0.3$ & $81.8{\pm}7.1$ & $99.4{\pm}0.4$ & $93.2{\pm}1.9$ & $97.0{\pm}1.2$ & $90.9$ & $98.4$ \\
 & \quad Mismatched-dataset bank & $94.9{\pm}4.8$ & $95.0{\pm}0.8$ & $98.8{\pm}0.3$ & $80.5{\pm}6.8$ & $99.4{\pm}0.4$ & $93.3{\pm}1.7$ & $97.2{\pm}1.1$ & $90.9$ & $98.5$ \\
 & \quad Gram-matched random & $93.7{\pm}5.3$ & $94.8{\pm}1.1$ & $98.8{\pm}0.3$ & $81.5{\pm}7.3$ & $99.4{\pm}0.4$ & $93.2{\pm}2.1$ & $97.0{\pm}1.2$ & $90.8$ & $98.4$ \\
 & \quad Unmatched orthogonal & $81.8{\pm}15.9$ & $94.4{\pm}1.1$ & $98.5{\pm}0.6$ & $78.2{\pm}7.7$ & $97.1{\pm}4.1$ & $93.2{\pm}2.1$ & $97.1{\pm}1.1$ & $86.9$ & $97.6$ \\
 & \quad Matched orthogonal & $93.9{\pm}6.4$ & $94.7{\pm}1.2$ & $98.8{\pm}0.3$ & $81.0{\pm}7.5$ & $99.4{\pm}0.3$ & $93.3{\pm}2.0$ & $97.1{\pm}1.2$ & $90.7$ & $98.4$ \\
 & \multicolumn{9}{l}{\textbf{DINOv2-L/14}} \\
 & \quad Correct prompts & $94.5{\pm}6.2$ & $94.8{\pm}1.2$ & $98.4{\pm}0.6$ & $85.9{\pm}2.8$ & $99.3{\pm}0.6$ & $90.2{\pm}3.2$ & $96.2{\pm}2.0$ & $91.4$ & $98.0$ \\
 & \quad Mismatched-dataset bank & $95.5{\pm}4.6$ & $95.0{\pm}0.9$ & $98.7{\pm}0.4$ & $84.6{\pm}3.4$ & $99.1{\pm}0.8$ & $89.8{\pm}3.0$ & $96.2{\pm}2.0$ & $91.2$ & $98.0$ \\
 & \quad Gram-matched random & $94.6{\pm}5.9$ & $94.7{\pm}1.3$ & $98.4{\pm}0.5$ & $86.7{\pm}2.7$ & $99.4{\pm}0.6$ & $89.8{\pm}3.0$ & $96.2{\pm}1.9$ & $91.5$ & $98.0$ \\
 & \quad Unmatched orthogonal & $93.0{\pm}7.1$ & $94.4{\pm}1.4$ & $98.2{\pm}0.9$ & $80.5{\pm}4.9$ & $97.1{\pm}4.1$ & $89.9{\pm}3.8$ & $96.1{\pm}2.4$ & $89.5$ & $97.2$ \\
 & \quad Matched orthogonal & $94.2{\pm}6.5$ & $94.7{\pm}1.4$ & $98.4{\pm}0.5$ & $86.2{\pm}3.0$ & $99.3{\pm}0.6$ & $89.8{\pm}3.3$ & $96.2{\pm}1.9$ & $91.2$ & $98.0$ \\
 & \multicolumn{9}{l}{\textbf{CLIP ViT-L/14}} \\
 & \quad Correct prompts & $94.0{\pm}3.3$ & $94.7{\pm}0.6$ & $98.0{\pm}0.7$ & $80.7{\pm}8.3$ & $99.1{\pm}0.7$ & $91.9{\pm}2.5$ & $96.8{\pm}1.4$ & $90.3$ & $97.9$ \\
 & \quad Mismatched-dataset bank & $94.2{\pm}2.2$ & $94.8{\pm}0.9$ & $98.4{\pm}0.4$ & $79.7{\pm}8.5$ & $98.1{\pm}2.3$ & $91.6{\pm}2.6$ & $96.7{\pm}1.5$ & $90.1$ & $97.7$ \\
 & \quad Gram-matched random & $93.8{\pm}3.5$ & $94.6{\pm}0.5$ & $98.1{\pm}0.6$ & $80.3{\pm}8.1$ & $99.0{\pm}0.7$ & $91.8{\pm}2.6$ & $96.8{\pm}1.4$ & $90.1$ & $98.0$ \\
 & \quad Unmatched orthogonal & $92.5{\pm}4.4$ & $93.6{\pm}2.4$ & $97.7{\pm}1.4$ & $72.8{\pm}5.2$ & $98.8{\pm}1.3$ & $91.4{\pm}2.7$ & $96.8{\pm}1.3$ & $87.6$ & $97.8$ \\
 & \quad Matched orthogonal & $94.5{\pm}4.3$ & $94.6{\pm}0.6$ & $98.0{\pm}0.6$ & $80.4{\pm}7.9$ & $99.1{\pm}0.7$ & $91.6{\pm}2.7$ & $96.7{\pm}1.4$ & $90.3$ & $97.9$ \\
\midrule
\multirow{18}{*}{4} & \multicolumn{9}{l}{\textbf{DINOv3-L/16}} \\
 & \quad Correct prompts & $97.2{\pm}1.4$ & $96.0{\pm}1.1$ & $99.2{\pm}0.0$ & $87.0{\pm}3.8$ & $99.7{\pm}0.2$ & $93.9{\pm}2.9$ & $98.0{\pm}0.5$ & $93.5$ & $98.9$ \\
 & \quad Mismatched-dataset bank & $98.0{\pm}0.8$ & $96.0{\pm}0.9$ & $99.2{\pm}0.0$ & $85.3{\pm}2.6$ & $99.6{\pm}0.2$ & $93.8{\pm}3.0$ & $97.9{\pm}0.6$ & $93.3$ & $98.9$ \\
 & \quad Gram-matched random & $96.9{\pm}1.5$ & $95.9{\pm}1.0$ & $99.2{\pm}0.0$ & $85.8{\pm}2.0$ & $99.7{\pm}0.1$ & $93.9{\pm}2.9$ & $98.0{\pm}0.6$ & $93.1$ & $98.9$ \\
 & \quad Unmatched orthogonal & $95.5{\pm}1.8$ & $95.8{\pm}1.2$ & $99.1{\pm}0.1$ & $82.6{\pm}2.6$ & $99.4{\pm}0.5$ & $94.2{\pm}2.2$ & $98.3{\pm}0.3$ & $92.0$ & $98.9$ \\
 & \quad Matched orthogonal & $97.6{\pm}0.9$ & $95.9{\pm}1.0$ & $99.2{\pm}0.0$ & $86.7{\pm}3.8$ & $99.7{\pm}0.1$ & $93.9{\pm}2.9$ & $97.9{\pm}0.6$ & $93.5$ & $98.9$ \\
 & \multicolumn{9}{l}{\textbf{DINOv2-L/14}} \\
 & \quad Correct prompts & $96.5{\pm}3.2$ & $95.9{\pm}0.6$ & $99.0{\pm}0.2$ & $88.8{\pm}2.5$ & $99.8{\pm}0.1$ & $93.7{\pm}2.4$ & $97.8{\pm}0.5$ & $93.7$ & $98.9$ \\
 & \quad Mismatched-dataset bank & $97.6{\pm}1.7$ & $96.0{\pm}0.7$ & $99.1{\pm}0.1$ & $88.2{\pm}2.5$ & $99.8{\pm}0.1$ & $93.3{\pm}2.2$ & $97.8{\pm}0.4$ & $93.8$ & $98.9$ \\
 & \quad Gram-matched random & $97.4{\pm}1.5$ & $95.9{\pm}0.6$ & $99.0{\pm}0.2$ & $89.0{\pm}2.6$ & $99.8{\pm}0.1$ & $93.6{\pm}2.3$ & $97.8{\pm}0.5$ & $94.0$ & $98.9$ \\
 & \quad Unmatched orthogonal & $96.4{\pm}1.7$ & $95.9{\pm}0.7$ & $99.1{\pm}0.1$ & $86.2{\pm}2.1$ & $99.5{\pm}0.5$ & $93.3{\pm}2.4$ & $97.9{\pm}0.4$ & $92.9$ & $98.8$ \\
 & \quad Matched orthogonal & $97.4{\pm}1.7$ & $95.9{\pm}0.6$ & $99.0{\pm}0.2$ & $88.9{\pm}2.8$ & $99.8{\pm}0.1$ & $93.6{\pm}2.2$ & $97.8{\pm}0.5$ & $93.9$ & $98.9$ \\
 & \multicolumn{9}{l}{\textbf{CLIP ViT-L/14}} \\
 & \quad Correct prompts & $97.5{\pm}1.7$ & $95.9{\pm}0.2$ & $98.9{\pm}0.0$ & $84.4{\pm}5.0$ & $99.7{\pm}0.2$ & $92.8{\pm}2.4$ & $97.9{\pm}0.4$ & $92.7$ & $98.8$ \\
 & \quad Mismatched-dataset bank & $97.6{\pm}1.7$ & $95.8{\pm}0.3$ & $99.0{\pm}0.1$ & $83.9{\pm}4.4$ & $99.7{\pm}0.2$ & $92.8{\pm}2.3$ & $97.8{\pm}0.4$ & $92.5$ & $98.8$ \\
 & \quad Gram-matched random & $97.6{\pm}1.5$ & $95.8{\pm}0.2$ & $98.9{\pm}0.1$ & $84.6{\pm}4.9$ & $99.7{\pm}0.2$ & $92.8{\pm}2.5$ & $97.8{\pm}0.5$ & $92.7$ & $98.8$ \\
 & \quad Unmatched orthogonal & $96.2{\pm}1.3$ & $96.1{\pm}0.3$ & $98.9{\pm}0.1$ & $82.8{\pm}4.9$ & $99.5{\pm}0.3$ & $92.7{\pm}2.4$ & $98.0{\pm}0.5$ & $91.9$ & $98.8$ \\
 & \quad Matched orthogonal & $97.4{\pm}1.8$ & $95.8{\pm}0.2$ & $98.9{\pm}0.0$ & $84.4{\pm}4.7$ & $99.7{\pm}0.2$ & $92.9{\pm}2.4$ & $97.8{\pm}0.4$ & $92.6$ & $98.8$ \\
\midrule
\multirow{18}{*}{8} & \multicolumn{9}{l}{\textbf{DINOv3-L/16}} \\
 & \quad Correct prompts & $97.8{\pm}0.8$ & $96.1{\pm}0.3$ & $99.2{\pm}0.1$ & $89.4{\pm}3.0$ & $98.6{\pm}1.7$ & $97.3{\pm}0.9$ & $98.7{\pm}0.4$ & $95.2$ & $98.8$ \\
 & \quad Mismatched-dataset bank & $98.0{\pm}0.8$ & $96.0{\pm}0.2$ & $99.1{\pm}0.1$ & $88.8{\pm}3.3$ & $97.5{\pm}3.6$ & $97.2{\pm}1.0$ & $98.7{\pm}0.5$ & $95.0$ & $98.4$ \\
 & \quad Gram-matched random & $97.2{\pm}1.6$ & $96.1{\pm}0.3$ & $99.1{\pm}0.2$ & $89.4{\pm}2.8$ & $98.6{\pm}1.7$ & $97.2{\pm}1.0$ & $98.6{\pm}0.4$ & $95.0$ & $98.8$ \\
 & \quad Unmatched orthogonal & $96.1{\pm}2.0$ & $96.1{\pm}0.2$ & $99.2{\pm}0.2$ & $86.1{\pm}2.7$ & $95.3{\pm}7.4$ & $97.4{\pm}0.5$ & $98.7{\pm}0.4$ & $93.9$ & $97.8$ \\
 & \quad Matched orthogonal & $96.8{\pm}1.8$ & $96.2{\pm}0.4$ & $99.1{\pm}0.2$ & $89.3{\pm}2.8$ & $98.8{\pm}1.4$ & $97.2{\pm}1.0$ & $98.6{\pm}0.4$ & $94.9$ & $98.9$ \\
 & \multicolumn{9}{l}{\textbf{DINOv2-L/14}} \\
 & \quad Correct prompts & $98.8{\pm}0.8$ & $96.4{\pm}0.1$ & $99.0{\pm}0.2$ & $89.6{\pm}4.2$ & $99.2{\pm}0.9$ & $97.3{\pm}0.6$ & $98.5{\pm}0.4$ & $95.5$ & $98.9$ \\
 & \quad Mismatched-dataset bank & $99.1{\pm}0.5$ & $96.5{\pm}0.1$ & $99.2{\pm}0.1$ & $89.2{\pm}3.9$ & $98.8{\pm}1.6$ & $97.3{\pm}0.6$ & $98.5{\pm}0.4$ & $95.5$ & $98.8$ \\
 & \quad Gram-matched random & $99.0{\pm}0.5$ & $96.4{\pm}0.2$ & $99.0{\pm}0.2$ & $89.6{\pm}3.9$ & $99.3{\pm}0.8$ & $97.3{\pm}0.5$ & $98.5{\pm}0.4$ & $95.6$ & $98.9$ \\
 & \quad Unmatched orthogonal & $98.4{\pm}0.3$ & $96.2{\pm}0.1$ & $99.1{\pm}0.1$ & $86.5{\pm}3.3$ & $96.4{\pm}4.7$ & $97.5{\pm}0.8$ & $98.8{\pm}0.2$ & $94.7$ & $98.1$ \\
 & \quad Matched orthogonal & $98.9{\pm}0.4$ & $96.4{\pm}0.1$ & $99.0{\pm}0.2$ & $89.3{\pm}3.9$ & $99.3{\pm}0.8$ & $97.3{\pm}0.5$ & $98.5{\pm}0.4$ & $95.5$ & $98.9$ \\
 & \multicolumn{9}{l}{\textbf{CLIP ViT-L/14}} \\
 & \quad Correct prompts & $98.5{\pm}0.1$ & $96.0{\pm}0.3$ & $98.7{\pm}0.5$ & $87.5{\pm}3.0$ & $98.8{\pm}1.5$ & $95.9{\pm}1.2$ & $98.4{\pm}0.3$ & $94.5$ & $98.6$ \\
 & \quad Mismatched-dataset bank & $98.6{\pm}0.1$ & $96.0{\pm}0.4$ & $99.0{\pm}0.2$ & $87.0{\pm}3.2$ & $98.2{\pm}2.3$ & $95.9{\pm}1.1$ & $98.2{\pm}0.5$ & $94.4$ & $98.5$ \\
 & \quad Gram-matched random & $98.2{\pm}0.5$ & $95.9{\pm}0.3$ & $98.7{\pm}0.5$ & $87.8{\pm}2.6$ & $98.7{\pm}1.6$ & $96.0{\pm}1.1$ & $98.3{\pm}0.3$ & $94.5$ & $98.6$ \\
 & \quad Unmatched orthogonal & $97.8{\pm}0.3$ & $95.9{\pm}0.2$ & $98.9{\pm}0.3$ & $84.3{\pm}3.2$ & $99.0{\pm}0.9$ & $95.8{\pm}1.0$ & $98.5{\pm}0.3$ & $93.5$ & $98.8$ \\
 & \quad Matched orthogonal & $98.3{\pm}0.3$ & $95.9{\pm}0.4$ & $98.7{\pm}0.5$ & $87.7{\pm}2.9$ & $98.7{\pm}1.5$ & $95.9{\pm}1.2$ & $98.3{\pm}0.3$ & $94.5$ & $98.6$ \\
\midrule
\multirow{18}{*}{16} & \multicolumn{9}{l}{\textbf{DINOv3-L/16}} \\
 & \quad Correct prompts & $99.3{\pm}0.1$ & $97.2{\pm}0.2$ & $99.4{\pm}0.1$ & $91.4{\pm}0.9$ & $99.9{\pm}0.0$ & $97.9{\pm}0.7$ & $99.3{\pm}0.1$ & $96.4$ & $99.5$ \\
 & \quad Mismatched-dataset bank & $99.3{\pm}0.2$ & $97.1{\pm}0.2$ & $99.4{\pm}0.1$ & $90.9{\pm}1.0$ & $99.9{\pm}0.0$ & $97.8{\pm}0.8$ & $99.2{\pm}0.1$ & $96.3$ & $99.5$ \\
 & \quad Gram-matched random & $99.2{\pm}0.3$ & $97.2{\pm}0.2$ & $99.4{\pm}0.1$ & $91.4{\pm}0.9$ & $99.9{\pm}0.1$ & $97.9{\pm}0.7$ & $99.2{\pm}0.1$ & $96.4$ & $99.5$ \\
 & \quad Unmatched orthogonal & $98.7{\pm}0.1$ & $97.1{\pm}0.2$ & $99.4{\pm}0.1$ & $88.9{\pm}1.6$ & $99.8{\pm}0.1$ & $97.7{\pm}0.8$ & $99.2{\pm}0.1$ & $95.6$ & $99.5$ \\
 & \quad Matched orthogonal & $99.2{\pm}0.3$ & $97.1{\pm}0.2$ & $99.4{\pm}0.1$ & $91.3{\pm}0.8$ & $99.9{\pm}0.0$ & $97.9{\pm}0.7$ & $99.2{\pm}0.1$ & $96.4$ & $99.5$ \\
 & \multicolumn{9}{l}{\textbf{DINOv2-L/14}} \\
 & \quad Correct prompts & $99.6{\pm}0.0$ & $97.3{\pm}0.2$ & $99.4{\pm}0.1$ & $92.1{\pm}1.4$ & $99.9{\pm}0.1$ & $98.2{\pm}0.6$ & $99.3{\pm}0.1$ & $96.8$ & $99.5$ \\
 & \quad Mismatched-dataset bank & $99.7{\pm}0.0$ & $97.3{\pm}0.2$ & $99.4{\pm}0.0$ & $91.4{\pm}1.4$ & $99.8{\pm}0.1$ & $98.0{\pm}0.7$ & $99.3{\pm}0.1$ & $96.6$ & $99.5$ \\
 & \quad Gram-matched random & $99.6{\pm}0.0$ & $97.3{\pm}0.2$ & $99.4{\pm}0.1$ & $92.2{\pm}1.2$ & $99.9{\pm}0.1$ & $98.2{\pm}0.6$ & $99.3{\pm}0.1$ & $96.8$ & $99.5$ \\
 & \quad Unmatched orthogonal & $99.3{\pm}0.1$ & $97.2{\pm}0.3$ & $99.4{\pm}0.1$ & $89.5{\pm}1.1$ & $99.5{\pm}0.7$ & $98.0{\pm}0.9$ & $99.3{\pm}0.1$ & $96.0$ & $99.4$ \\
 & \quad Matched orthogonal & $99.6{\pm}0.0$ & $97.2{\pm}0.2$ & $99.4{\pm}0.1$ & $92.3{\pm}1.1$ & $99.9{\pm}0.1$ & $98.1{\pm}0.6$ & $99.3{\pm}0.1$ & $96.8$ & $99.5$ \\
 & \multicolumn{9}{l}{\textbf{CLIP ViT-L/14}} \\
 & \quad Correct prompts & $99.6{\pm}0.1$ & $96.9{\pm}0.2$ & $99.2{\pm}0.0$ & $91.6{\pm}1.1$ & $99.9{\pm}0.0$ & $97.2{\pm}0.9$ & $99.0{\pm}0.1$ & $96.3$ & $99.4$ \\
 & \quad Mismatched-dataset bank & $99.7{\pm}0.1$ & $96.9{\pm}0.2$ & $99.3{\pm}0.0$ & $90.8{\pm}1.3$ & $99.9{\pm}0.0$ & $97.2{\pm}0.9$ & $99.0{\pm}0.1$ & $96.1$ & $99.4$ \\
 & \quad Gram-matched random & $99.6{\pm}0.1$ & $96.9{\pm}0.2$ & $99.1{\pm}0.0$ & $91.5{\pm}1.2$ & $99.9{\pm}0.0$ & $97.2{\pm}0.9$ & $99.0{\pm}0.1$ & $96.3$ & $99.3$ \\
 & \quad Unmatched orthogonal & $99.2{\pm}0.0$ & $97.0{\pm}0.1$ & $99.2{\pm}0.0$ & $89.0{\pm}0.3$ & $99.8{\pm}0.0$ & $97.1{\pm}1.0$ & $99.1{\pm}0.1$ & $95.6$ & $99.4$ \\
 & \quad Matched orthogonal & $99.6{\pm}0.1$ & $96.9{\pm}0.2$ & $99.1{\pm}0.0$ & $91.6{\pm}1.4$ & $99.9{\pm}0.0$ & $97.2{\pm}0.9$ & $99.0{\pm}0.1$ & $96.3$ & $99.3$ \\
\bottomrule
\end{tabular}
\end{adjustbox}
\endgroup
\end{table}

\begin{table}[p]
\centering
\caption{Anchor substitution in the normal-only regime ($K$ normal samples; synthetic anomalies generated online). Cells are mean$\pm$sd over three training seeds of image-level AUROC (\%). BUSI has evaluation masks, but its pixel AUROC is omitted to keep a common image-level endpoint across all three datasets; CheXpert and Kaggle BrainMRI have no evaluation masks. Conditions are as in the real-anomaly-support table; the two supervision regimes are never pooled.}
\label{tab:semantic_necessity_audit_syn}
\begingroup
\scriptsize
\setlength{\tabcolsep}{3pt}
\renewcommand{\arraystretch}{1.15}
\begin{adjustbox}{max totalsize={\textwidth}{0.80\textheight},center}
\begin{tabular}{llcccc}
\toprule
\multirow{2}{*}{\textbf{$K$}} & \multirow{2}{*}{\textbf{Anchor condition}} &
\textbf{\makecell{Kaggle\\BrainMRI}} & \textbf{CheXpert} &
\textbf{BUSI} & \textbf{Avg} \\
\cmidrule(lr){3-3}
\cmidrule(lr){4-4}
\cmidrule(lr){5-5}
\cmidrule(lr){6-6}
& & \textbf{AC} & \textbf{AC} &
\textbf{AC} & \textbf{AC} \\
\midrule
\multirow{18}{*}{4} & \multicolumn{4}{l}{\textbf{DINOv3-L/16}} \\
 & \quad Correct prompts & $93.6{\pm}1.7$ & $71.1{\pm}1.3$ & $87.0{\pm}2.7$ & $83.9$ \\
 & \quad Mismatched-dataset bank & $94.1{\pm}1.6$ & $72.2{\pm}1.2$ & $87.3{\pm}1.4$ & $84.5$ \\
 & \quad Gram-matched random & $93.5{\pm}2.4$ & $71.5{\pm}1.3$ & $87.0{\pm}2.1$ & $84.0$ \\
 & \quad Unmatched orthogonal & $91.0{\pm}2.7$ & $69.9{\pm}4.2$ & $84.9{\pm}1.8$ & $82.0$ \\
 & \quad Matched orthogonal & $93.7{\pm}2.3$ & $71.7{\pm}0.9$ & $86.9{\pm}1.2$ & $84.1$ \\
 & \multicolumn{4}{l}{\textbf{DINOv2-L/14}} \\
 & \quad Correct prompts & $92.6{\pm}1.4$ & $66.8{\pm}4.7$ & $85.4{\pm}2.4$ & $81.6$ \\
 & \quad Mismatched-dataset bank & $91.9{\pm}1.6$ & $66.7{\pm}3.4$ & $85.7{\pm}1.4$ & $81.4$ \\
 & \quad Gram-matched random & $92.3{\pm}1.7$ & $67.1{\pm}3.6$ & $86.2{\pm}1.1$ & $81.9$ \\
 & \quad Unmatched orthogonal & $92.1{\pm}2.1$ & $67.2{\pm}3.6$ & $85.3{\pm}0.7$ & $81.5$ \\
 & \quad Matched orthogonal & $92.5{\pm}1.3$ & $66.6{\pm}4.1$ & $85.9{\pm}1.5$ & $81.7$ \\
 & \multicolumn{4}{l}{\textbf{CLIP ViT-L/14}} \\
 & \quad Correct prompts & $91.6{\pm}2.0$ & $69.6{\pm}3.1$ & $86.2{\pm}1.9$ & $82.5$ \\
 & \quad Mismatched-dataset bank & $92.4{\pm}1.4$ & $70.0{\pm}3.2$ & $86.0{\pm}0.4$ & $82.8$ \\
 & \quad Gram-matched random & $92.0{\pm}2.6$ & $70.0{\pm}3.5$ & $85.9{\pm}1.4$ & $82.6$ \\
 & \quad Unmatched orthogonal & $91.9{\pm}2.4$ & $70.1{\pm}3.6$ & $84.4{\pm}0.9$ & $82.1$ \\
 & \quad Matched orthogonal & $91.7{\pm}3.0$ & $70.1{\pm}2.8$ & $86.2{\pm}2.0$ & $82.7$ \\
\midrule
\multirow{18}{*}{8} & \multicolumn{4}{l}{\textbf{DINOv3-L/16}} \\
 & \quad Correct prompts & $95.1{\pm}1.0$ & $70.2{\pm}1.1$ & $89.5{\pm}1.0$ & $84.9$ \\
 & \quad Mismatched-dataset bank & $95.2{\pm}0.8$ & $70.0{\pm}1.1$ & $89.6{\pm}1.1$ & $84.9$ \\
 & \quad Gram-matched random & $95.1{\pm}1.0$ & $70.3{\pm}0.6$ & $89.5{\pm}0.9$ & $85.0$ \\
 & \quad Unmatched orthogonal & $93.1{\pm}1.0$ & $67.4{\pm}2.4$ & $86.1{\pm}1.9$ & $82.2$ \\
 & \quad Matched orthogonal & $95.3{\pm}1.0$ & $70.3{\pm}0.8$ & $89.2{\pm}1.1$ & $84.9$ \\
 & \multicolumn{4}{l}{\textbf{DINOv2-L/14}} \\
 & \quad Correct prompts & $93.6{\pm}0.2$ & $67.2{\pm}3.1$ & $88.2{\pm}1.7$ & $83.0$ \\
 & \quad Mismatched-dataset bank & $93.8{\pm}0.1$ & $66.5{\pm}2.9$ & $88.5{\pm}1.5$ & $82.9$ \\
 & \quad Gram-matched random & $93.4{\pm}0.4$ & $67.3{\pm}3.3$ & $88.2{\pm}2.4$ & $82.9$ \\
 & \quad Unmatched orthogonal & $93.7{\pm}0.3$ & $68.0{\pm}3.4$ & $88.2{\pm}1.6$ & $83.3$ \\
 & \quad Matched orthogonal & $93.6{\pm}0.3$ & $66.6{\pm}2.8$ & $88.6{\pm}1.8$ & $82.9$ \\
 & \multicolumn{4}{l}{\textbf{CLIP ViT-L/14}} \\
 & \quad Correct prompts & $91.2{\pm}0.9$ & $68.7{\pm}2.5$ & $87.5{\pm}0.9$ & $82.4$ \\
 & \quad Mismatched-dataset bank & $92.5{\pm}0.7$ & $69.0{\pm}3.0$ & $87.3{\pm}0.9$ & $83.0$ \\
 & \quad Gram-matched random & $91.9{\pm}0.4$ & $68.8{\pm}2.7$ & $87.3{\pm}0.1$ & $82.7$ \\
 & \quad Unmatched orthogonal & $91.8{\pm}0.8$ & $69.3{\pm}2.9$ & $87.2{\pm}0.8$ & $82.8$ \\
 & \quad Matched orthogonal & $91.9{\pm}0.5$ & $69.1{\pm}2.3$ & $87.1{\pm}1.3$ & $82.7$ \\
\midrule
\multirow{18}{*}{16} & \multicolumn{4}{l}{\textbf{DINOv3-L/16}} \\
 & \quad Correct prompts & $95.2{\pm}0.4$ & $70.8{\pm}0.9$ & $89.2{\pm}0.8$ & $85.0$ \\
 & \quad Mismatched-dataset bank & $95.0{\pm}0.3$ & $70.8{\pm}1.1$ & $89.9{\pm}1.0$ & $85.2$ \\
 & \quad Gram-matched random & $95.2{\pm}0.4$ & $70.8{\pm}1.0$ & $89.4{\pm}0.6$ & $85.1$ \\
 & \quad Unmatched orthogonal & $93.5{\pm}1.0$ & $70.3{\pm}1.7$ & $88.1{\pm}2.2$ & $84.0$ \\
 & \quad Matched orthogonal & $95.1{\pm}0.5$ & $70.6{\pm}0.6$ & $89.6{\pm}0.7$ & $85.1$ \\
 & \multicolumn{4}{l}{\textbf{DINOv2-L/14}} \\
 & \quad Correct prompts & $93.7{\pm}0.7$ & $67.0{\pm}3.0$ & $90.0{\pm}0.7$ & $83.6$ \\
 & \quad Mismatched-dataset bank & $93.6{\pm}0.5$ & $66.7{\pm}3.0$ & $89.8{\pm}0.6$ & $83.4$ \\
 & \quad Gram-matched random & $93.7{\pm}0.6$ & $66.8{\pm}3.0$ & $89.3{\pm}0.6$ & $83.3$ \\
 & \quad Unmatched orthogonal & $94.2{\pm}0.5$ & $67.7{\pm}2.0$ & $89.7{\pm}0.3$ & $83.9$ \\
 & \quad Matched orthogonal & $93.8{\pm}0.7$ & $66.5{\pm}2.7$ & $89.9{\pm}0.8$ & $83.4$ \\
 & \multicolumn{4}{l}{\textbf{CLIP ViT-L/14}} \\
 & \quad Correct prompts & $92.9{\pm}0.8$ & $69.3{\pm}1.6$ & $88.2{\pm}0.6$ & $83.4$ \\
 & \quad Mismatched-dataset bank & $92.9{\pm}0.7$ & $69.0{\pm}2.3$ & $87.7{\pm}0.9$ & $83.2$ \\
 & \quad Gram-matched random & $92.9{\pm}0.2$ & $69.5{\pm}1.3$ & $88.4{\pm}0.6$ & $83.6$ \\
 & \quad Unmatched orthogonal & $92.5{\pm}0.4$ & $69.8{\pm}1.3$ & $87.8{\pm}0.6$ & $83.4$ \\
 & \quad Matched orthogonal & $92.9{\pm}0.7$ & $69.4{\pm}1.7$ & $87.8{\pm}0.8$ & $83.4$ \\
\midrule
\multirow{18}{*}{32} & \multicolumn{4}{l}{\textbf{DINOv3-L/16}} \\
 & \quad Correct prompts & $94.7{\pm}0.2$ & $69.9{\pm}1.7$ & $87.4{\pm}0.4$ & $84.0$ \\
 & \quad Mismatched-dataset bank & $94.6{\pm}0.2$ & $69.6{\pm}1.9$ & $88.0{\pm}0.2$ & $84.1$ \\
 & \quad Gram-matched random & $94.2{\pm}0.3$ & $69.9{\pm}1.9$ & $87.2{\pm}0.5$ & $83.8$ \\
 & \quad Unmatched orthogonal & $92.8{\pm}1.4$ & $67.6{\pm}1.5$ & $87.0{\pm}1.5$ & $82.5$ \\
 & \quad Matched orthogonal & $94.2{\pm}0.0$ & $69.9{\pm}1.3$ & $87.3{\pm}0.2$ & $83.8$ \\
 & \multicolumn{4}{l}{\textbf{DINOv2-L/14}} \\
 & \quad Correct prompts & $94.2{\pm}0.1$ & $65.8{\pm}1.9$ & $89.5{\pm}0.7$ & $83.2$ \\
 & \quad Mismatched-dataset bank & $94.1{\pm}0.1$ & $65.9{\pm}1.8$ & $89.6{\pm}0.3$ & $83.2$ \\
 & \quad Gram-matched random & $94.1{\pm}0.1$ & $65.8{\pm}2.0$ & $89.4{\pm}0.5$ & $83.1$ \\
 & \quad Unmatched orthogonal & $94.1{\pm}0.2$ & $66.7{\pm}2.3$ & $88.9{\pm}0.1$ & $83.2$ \\
 & \quad Matched orthogonal & $94.0{\pm}0.4$ & $65.8{\pm}1.9$ & $89.8{\pm}0.4$ & $83.2$ \\
 & \multicolumn{4}{l}{\textbf{CLIP ViT-L/14}} \\
 & \quad Correct prompts & $93.0{\pm}0.3$ & $69.7{\pm}1.0$ & $88.1{\pm}0.4$ & $83.6$ \\
 & \quad Mismatched-dataset bank & $92.9{\pm}0.2$ & $69.3{\pm}1.4$ & $87.4{\pm}0.2$ & $83.2$ \\
 & \quad Gram-matched random & $92.9{\pm}0.3$ & $68.8{\pm}1.1$ & $87.8{\pm}0.5$ & $83.2$ \\
 & \quad Unmatched orthogonal & $92.8{\pm}0.4$ & $70.1{\pm}1.1$ & $87.6{\pm}0.3$ & $83.5$ \\
 & \quad Matched orthogonal & $92.9{\pm}0.1$ & $69.1{\pm}1.4$ & $87.8{\pm}0.2$ & $83.3$ \\
\bottomrule
\end{tabular}
\end{adjustbox}
\endgroup
\end{table}

\end{document}